\documentclass[preprint]{jmlr}

\makeatletter
\def\set@curr@file#1{\def\@curr@file{#1}} 
\makeatother
\usepackage[load-configurations=version-1]{siunitx} 

\usepackage{lipsum}
\usepackage{mwe}
\usepackage{verbatim}

\usepackage{tikz}
\usepackage{filecontents}
\usepackage{pgfplots}
\usepackage{scrextend}
\usepackage{cleveref}[2012/02/15]

\newtheorem{rules}{Do-Calculus Rule}

\usepackage{algorithm}
\usepackage{algorithmic}

\usepackage{enumitem}
\usepackage{multirow}
\usepackage{footmisc}
\usepackage{longtable}
\usepackage{array,booktabs,arydshln,xcolor}
\usepackage{booktabs}
\usepackage[load-configurations=version-1]{siunitx} 
\usepackage{pdflscape}
\usepackage{adjustbox}
\usepackage{setspace}

\usepackage[title]{appendix}

\begin{document}

\title[SCM to Estimate Effect of Oxygen Therapy]{Structural Causal Model with Expert Augmented Knowledge to Estimate the Effect of Oxygen Therapy on Mortality in the ICU}

\author{\Name {Md Osman Gani} \Email {mogani@umbc.edu} \\
      \addr Information Systems\\
      University of Maryland, Baltimore County\\
      Baltimore, Maryland, USA
      \AND
      \Name {Shravan Kethireddy} \Email {shravan.kethireddy@nghs.com}\\
      \addr Critical Care, Northeast Georgia Health System\\
      Georgia, USA
      \AND
      \Name {Marvi Bikak} \Email {marvi.bikak@gmail.com}\\
      \addr Critical Care, Palos Health\\
      Illinois, USA
      \AND
      \Name {Paul Griffin} \Email {pmg14@psu.edu} \\
      \addr Department of Industrial and Manufacturing Engineering\\
      Penn State University\\
      University Park, Pennsylvania, USA
      \AND
      \Name {Mohammad Adibuzzaman} \Email {madibuzz@purdue.edu} \\
      \addr Regenstrief Center for Healthcare Engineering\\
      Purdue University\\
      West Lafayette, Indiana, USA
      }
\maketitle


\newcommand\independent{\protect\mathpalette{\protect\independenT}{\perp}}
\def\independenT#1#2{\mathrel{\rlap{$#1#2$}\mkern2mu{#1#2}}}

\begin{abstract}
    Recent advances in causal inference techniques, more specifically, in the theory of structural causal models, provide the framework for identification of causal effects from observational data in the cases where the causal graph is identifiable, i.e., the data generating mechanism can be recovered from the joint distribution. However, no such studies have been done to demonstrate this concept with a clinical example. We present a complete framework to estimate the causal effect from observational data by augmenting expert knowledge in the model development phase and with a practical clinical application. Our clinical application entails a timely and important research question, i.e., the effect of oxygen therapy intervention in the intensive care unit (ICU); the result of this project is useful in a variety of disease conditions, including severe acute respiratory syndrome coronavirus-2 (SARS-CoV-2) patients in the ICU. We used data from the MIMIC III database, a standard database in the machine learning community that contains 58,976 admissions from an ICU in Boston, MA, for estimating the oxygen therapy effect on morality. We also identified the covariate-specific effect to oxygen therapy from the model for more personalized intervention. 
\end{abstract}


\section{Introduction}
    Since the 1960s, randomized controlled trials (RCTs) are considered the gold standard to identify causation by regulatory bodies such as the US Food and Drug Administration (US FDA) and by the clinical communities \citep{greene2012reform}. The key idea behind RCT is that, by random assignment of treatment or interventions, the confounding bias, i.e., the bias due to the assignment of treatment or presence of other variables, including the unobserved confounders, can be accounted for in the estimand. Despite the strength of RCTs to identify causation, RCTs are increasingly considered time-consuming, costly, and often infeasible for safety and efficacy reasons \citep{frieden2017evidence}. Furthermore, RCTs can be plagued with selection bias due to the strict eligibility criteria, or transportability bias, due to the study population being different from the target population. At the other end of the spectrum, advances in the technology and the adoption of computerized systems in routine healthcare has enabled the collection and curation of large volumes of data, such as electronic health records (EHR), during routine care, albeit with the presence of confounding biases. Therefore, researchers are increasingly trying to address the challenge of identifying causal relationships from large but biased observational datasets. 

Recent advances in the theory of causal inference, more specifically, structural causal models (SCM) provide the framework for adjusting for different kinds of biases such as confounding, transportability, or selection bias under the general theory of data fusion with identifiability formulas, i.e., identification of causal effects,  such as backdoor and front-door adjustment, in many cases even if the confounders are unobserved \citep{bareinboim2016causal,pearl2016causal, pearl2009causality}. However, this approach requires developing a graphical representation of the causal relationship between variables of interest with meticulous scrutiny, using structure learning algorithms as well as with expert knowledge from clinicians or existing literature \citep{CausalStrucureLearning,lederer2019control}. While there are other clinical use cases with SCM \citep{schluter2019impact, adegunsoye2019computed, vitolo2018modeling, arendt2016placental, hernandez2008causal, schisterman2005lipid}, most of the literature developed the causal graph with domain expertise only. On the other hand, there are many algorithms that can learn the graphical structure from data with different assumptions. Consequently, the resultant causal graph can vary significantly based on the underlying assumptions ~\citep{CausalStrucureLearning}. There have been some recent works that tried to bridge between these two approaches: expert knowledge and automated data-driven graph development \citep{nordon2019building, rotmensch2017learning, goodwin2013automatic}. However, none of these existing approaches provides a complete framework for expert augmented causal graph generation with specific clinical application. We provide an approach for the development of the SCM with structure learning algorithms and expert augmentation of clinical knowledge in a principled way. We also present a clinical application to demonstrate the feasibility of conducting virtual trials -- from observational data -- under the proposed framework.

Our approach starts with processing observational data collected during routine health care such as EHR data from hospitals. These EHR systems or data repositories are queried based on the clinical question of interest, with explicit inclusion-exclusion criteria, to identify a cohort of eligible patients representative of the clinical setting. All the variables within the scope of the clinical question are extracted from the data repository. A graphical representation of cause and effect relationship between the variables can be learned from the extracted data, with structure learning algorithms augmented with domain knowledge through expert reviews. The resultant graph can be validated from the data through testable implications (e.g., conditional independence, Verma constraints). The key advantage of using SCM over statistical methods such as g-methods under the potential outcome framework \citep{hernan2008observational,hernan2020causal} is that all the assumptions such as the assumptions on observed and unobserved confounders are explicitly encoded in the SCM that can be analyzed in a principled way with \textit{do}-calculus \citep{pearl2009causality,pearl2018book}. 

For the clinical application, we study the effect of liberal versus conservative oxygen therapy in the intensive care unit with respect to ICU mortality as the primary outcome. Although oxygenation has been used to treat critically ill patients for hundreds of years \citep{Heffner18}, there has been a renewed interest \citep{Panwar2016, Girardis2016, NEJMoa1903297} in this topic as new evidence suggests not all patients may benefit from the same oxygenation strategy. To define the study question with a causal query that encodes eligibility or inclusion/exclusion criteria, we aimed at implementing the study protocol as described by Panwar et al. \citep{Panwar2016} in a multi-center pilot RCT. We will refer to this study as the oxygen therapy RCT (OT-RCT) in this paper. We used MIMIC-III database to emulate the OT-RCT protocol using observational data and SCM ~\cite{johnson2016mimic}. For the study, we used a conservative oxygenation strategy, target $SpO_{2}$ is between $88-95\%$, compared to a liberal oxygenation strategy, target $SpO_{2}$ is greater than or equal to $96\%$, for adult ICU patients receiving invasive mechanical ventilation (IMV). In the observational data,  we implemented a clear separation in the mean $SpO_{2}$ $94.25 (3.14)$ vs. $97.88(3.42)$, $SaO_{2}$ $91.39 (9.20)$ vs. $91.78 (11.84)$, $PaO_{2}$ $89.80 (69-98)$ vs. $121.32 (86-140)$, and $FiO_{2}$ $0.59 (0.18)$ vs. $0.47 (0.14)$ values between the conservative and the liberal groups respectively. Our analysis was performed in two different ways using the MIMIC database: i) correlation-based observational study (CB-OBS) and ii) structural causal model based  observational study as a virtual RCT (SCM-VRCT). 

The point estimate for 90-day mortality with SCM-VRCT was lower with the liberal oxygenation strategy compared to the conservative strategy. This is not consistent with the findings of the OT-RCT. The expectation of 90-day mortality for patients in the conservative oxygenation arm in the SCM-VRCT protocol (54\%) was equal to the CB-OBS approach (54\%). Also, given the severity of illness based on the physiological scores and age of the patients, the liberal oxygenation strategy in SCM VRCT has lower mortality $29.8\% (95\%$ CI $28.1 - 31.1\%)$ compared to the conservative approach in SCM-VRCT $54.1\%$. Furthermore, within the liberal arm, the mortality in SCM-VRCT is lower compared to that of in the CB-OBS $34.1\% (95\%$ CI $33 - 35.3\%)$. Another important observation is that the expectation of mortality is higher for the patients with a sequential organ failure assessment (SOFA) score greater than ten compared to those with scores less than or equal to ten. Conditioned on the severity of illness based on the SOFA score, there was no statistically significant difference in mortality for different oxygenation targets. Therefore, conditioned on the severity of illness, as measured by the SOFA score (SOFA score is greater than or equal to 10), SCM-VRCT supports the feasibility of a larger study to investigate the effect of conservative oxygenation to treat patients requiring invasive MV in ICUs.

\subsection{Technical Significance}
We provide a framework with a step by step approach from data curation, model development, model validation, and causal estimation and demonstrate the approach using a clinical application. Previous research has shown model development using SCMs with only a few variables that was mostly hypothesis-driven \citep{schluter2019impact, adegunsoye2019computed, vitolo2018modeling, arendt2016placental, hernandez2008causal, schisterman2005lipid}. Our approach provides a framework to encode  expert knowledge using human-in-the-loop model development with the SCM in a principled way. Furthermore, we provide the derivation for the identifiability equation using backdoor adjustment and do-calculus as well as the estimation process with inverse probability weighting (IPW). 
\subsection{Clinical Relevance}
There has been a number of trials in recent years to investigate the effect of conservative versus liberal (conventional) oxygenation strategies on mortality and number of ventilator-free days among patients in ICUs receiving invasive mechanical ventilation \citep{Panwar2016,Girardis2016,NEJMoa1903297}.  Results have been mixed and the most recent study \citep{NEJMoa1903297} showed no significant effect on the number of ventilator-free days between the treatment and the control group. Consequently, the study has a significant impact on clinical practice. The approach can be generalized for other clinical questions and can be adopted by the broader clinical community if relevant data and domain expertise exists.

\section{State of the Art}
Ronald Fischer in his seminal work, ``Design of Experiments'', first introduced randomized controlled trials (RCTs) as an experimental approach to identify causation by accounting for confounding biases \citep{fisher1960design}. The enactment of the 1962 Amendments to the Food, Drug and Cosmetics Act required that new treatments need to be proven efficacious in ``adequate and well-controlled investigations'' \citep{greene2012reform}. In the 1970s, the FDA translated this into the requirement that an RCT is  needed to validate the causal link between a new treatment (causal intervention) and a putative clinical outcome. The rationale for adopting RCTs as the means of gathering scientific evidence is that spurious associations due to factors extraneous to the relationship between the treatment and the outcome can be controlled for by randomizing the treatment assignment. Despite a host of  benefits, it is widely acknowledged that RCTs are far from perfect. In practice, RCTs are usually time-consuming, overly expensive, not always entirely ethical, contain biases even with randomization, and are often only applicable to a narrow stratum of the population.  

In the 1970s, the potential outcome framework for experimental studies was extended to non-experimental or observational studies through the introduction of the ignorability assumption and propensity score matching for adjustment of confounding biases \citep{rubin1974estimating,rosenbaum1983central}.  In the 1980s, sequential backdoor and g-methods were introduced with the potential outcome framework for time-varying exposures \citep{robins1986new, robins2009estimation}. In the 1990s the potential outcome framework was first conceptualized through graphical models, namely, structural causal models, on the foundations of Bayesian networks \citep{pearl1995causal}. Subsequently, \textit{do-calculus} was introduced with SCMs to mathematically transform observational and experimental studies. The strength of the SCM is that all the assumptions are explicitly specified in the model to provide a unified framework for an objective reproducible estimation approach. In recent years, a unified framework, called \textit{data fusion},  for adjusting other types of bias such as selection bias and transportability bias have been proposed \citep{pearl2014external,bareinboim2016causal,pearl2019seven}. 

Most of the applications of SCM based on clinical literature has been purely hypothesis-driven. In other words, the graphical model that describes the assumptions was generated with expert knowledge with a small set of variables \citep{schluter2019impact, adegunsoye2019computed, vitolo2018modeling, arendt2016placental, hernandez2008causal, schisterman2005lipid}. Several recent works attempted to generate the causal graph from medical literature automatically \citep{nordon2019building, rotmensch2017learning, goodwin2013automatic}. While both of these approaches are useful, to the best of our knowledge, there are no hybrid approaches that develop the causal graph iteratively from both the clinical data and expert knowledge. A recent study provided a framework for expert augmented machine learning with boosting trees to extract expert knowledge \citep{gennatas2020expert}, but not with SCMs. 

We propose an approach for generating evidence from clinical data captured during routine healthcare such as electronic health records with expert augmented causal graphs. We used the Medical Information Mart for the Intensive Care (MIMIC) III \citep{johnson2016mimic, adibuzzaman2016closing, adibuzzaman2017big} to answer the clinical question, whether liberal or conservative oxygen therapy help to improve survival rate in the intensive care unit. Numerous machine learning and clinical application studies have been published using this standard database \citep{ghassemi2015state, zhu2018bayesian} and the database is widely considered as a benchmark due to the high-quality granular data.

The clinical question of interest, i.e., the effect of liberal versus conservative oxygen therapy in the ICU has drawn widespread attention in recent years with a number of randomized trials. These trials investigated the effect of oxygenation strategies on mortality and the number of ventilator-free days among patients in ICUs receiving intermittent mandatory ventilation (IMV). In 2016, a pilot RCT \citep{Panwar2016} to investigate the effect of different oxygenation targets during mechanical ventilation on mortality supported the feasibility of a conservative oxygenation strategy in patients but concluded with a need for a larger RCT to be performed to evaluate its efficacy. In another study \citep{Girardis2016}, a conservative oxygenation strategy was associated with lower ICU mortality compared to the conventional strategy. One important note about this study is that the findings were reported after an unplanned early termination of the trial due to the potential adverse effect in one of the study arms. They also emphasized the need for a multi-center trial that evaluates this intervention. The most recent study \citep{NEJMoa1903297} with a 1000 patients' cohort found no significant effect on the number of ventilator-free days from the use of conservative oxygen therapy as compared to the conventional liberal oxygen therapy. \autoref{tab:study_demographics} summarizes the study demographics and \autoref{tab:study_outcomes} summarizes the results of these studies, among others. The lack of a consensus from these studies shows the importance of the continued study of clinical practice and resulting outcomes on mortality \citep{hirase2019impact}.

\section{Background}
    The causal analysis goes beyond association to infer probabilities under experimental conditions to estimate the effect of interventions. One example of these changes can be external interventions that can be used to estimate treatment effects. The development of mathematical machinery with SCM in the last couple of decades can be used to adjust for biases such as confounding bias, selection bias, and estimate treatment effect and counterfactual \citep{bareinboim2016causal}. In this section, we provide the fundamentals of the mathematical machinery with SCM. 

\subsection{Structural Causal Model (SCM)}
An SCM $M$, represent the causal relationship between variables. $M$ consists of two sets of variables $U$ and $V$, called exogenous and endogenous variables respectively. It also includes a set of functions $f$ that assign each endogenous variable in $V$ a value based on the values of other variables in the SCM. A variable $Y$ is directly caused by $X$ if $X$ is in the function $f$ of $Y$. A variable $X$ is a cause of another variable $Y$ if it is a direct or any cause of $Y$ ~\citep{pearl2009causality}.

\begin{definition}
	\label{def_scm}
	Structural Causal Model; A structural causal model is a 4-tuple  \\$M = \langle U, V, F, P(u) \rangle$, where
	\begin{enumerate}
	\item $U$ is a set of background variables (also called exogenous) that are determined by factors outside the model.
	\item $V$ is a set $\{V_1, V_2,\hdots,V_n\}$ of endogenous variables that are determined by variables in the model, viz.\ variables in $U \cup V$.
	\item $F$ is a set of functions $\{f_1, f_2,\hdots, f_n\}$ such that each $f_i$ is a mapping from the respective domains of $U_i \cup {PA}_i $ to $V_i$ and the entire set $F$ forms a mapping from $U$ to $V$. In other words, each $f_i$ assigns a value to the corresponding $V_i \in V$, $v_i \leftarrow f_i(pa_i, u_i),$ for $i = 1, 2, \hdots n$.
	\item $P(u)$ is a probability function defined over the domain of $U$.

	\end{enumerate}
\end{definition}

Each SCM $M$ is associated with a graph, known as a \textit{graphical causal model}, $G$. Each $G$ is a directed acyclic graph (DAG). This graphical causal model $G$ consists of a set of nodes or vertices that represent the variables $U$ and $V$, and a set of edges that represents the functions in $f$. $G$ contains a node for each variable in $M$ and a directed edge from $X$ to $Y$ if $X$ is in the function $f$ of $Y$ ($X$ is a direct cause of $Y$) in $M$. Causal models and graphs, therefore, represent causal relationships and encode causal assumptions ~\citep{pearl2009causality}.

\subsection{Causal Structure Learning}
\label{subsec:sla}
Causal structure learning algorithms are a family of algorithms that estimate the causal graph or certain aspects of it from a given data set ~\citep{HeinzeDeml2017CausalSL}. Given a distribution, these algorithms find possible causal diagrams implied by the graphical structure that could generate it.
We assume the underlying graph is a directed acyclic graph (DAG). The underlying causal DAG is generally not identifiable but the Markov equivalence class of DAGs, i.e. the set of DAGs that encode the same set of d-separation relationships, can be identified ~\citep{pearl2009causality}. The Markov equivalence class of DAGs can be described by completed partially DAGs (CPDAG) ~\citep{Andersson1997}. 

There are many methods that have been developed to estimate the underlying DAG. These are broadly classified as 1) constraint-based methods, 2) score-based methods and 3) hybrid methods ~\citep{JSSv035i03}. The constraint-based methods, PC ~\citep{Sprites2000}, rankPC ~\citep{harris2013a}, and FCI ~\citep{Sprites2000}, perform statistical tests of marginal and conditional independence to check dependency. On the other hand, the score-based methods, GES ~\citep{Chickering2002b}, rankGES ~\citep{Nandy2015HighdimensionalCI}, GIES ~\citep{hauser2012a}, and rankGIES, optimize the search according to a score function.  Hybrid methods such as Max-Min Hill-Climbing (MMHC) ~\citep{Tsamardinos2006} combine conditional tests with a score-based approach. Besides these three approaches, there are other methods such as structural equation models with additional restrictions, Linear Non-Gaussian Acyclic Models (LINGAM), ~\citep{Shimizu2006} that exploit invariance properties
~\citep{Rothenhusler2015BACKSHIFTLC}.

\subsection{\textit{do}-Calculus -- Graphical Identification Criterion}
We consider a query $Q$, the task of estimating the distribution of the outcome $Y$ after intervening on the treatment variable $X$, mathematically written as $P(Y = y | do (X = x))$. The question is if we can define the conditions under which we can infer the results of the query when we have data, $X, Y, Z$ from an observational study where $X, Y,$ and $Z$ are randomly sampled. This is a standard query of policy evaluation that needs to be estimated given $P(y, x, z)$ by controlling for confounding bias ~\citep{bareinboim2016causal}.

The question at the center of causal inference is to estimate the effects of interventions such as medical treatments or policy actions ~\citep{Hunermund2019CausalIA} ~\citep{bareinboim2016causal}. Interventions in an SCM $M$ are denoted by a mathematical operator, namely the $do(\cdot)$ operator.  These interventions in SCMs, in the form of $do(X = x)$, are carried out by removing individual functions, $f_i$, from the model i.e. in the corresponding graph, severing all incoming arrows that enter the manipulated or intervening variable $X$. The distribution of $Y$ after the intervention can be defined using the counterfactual notation as,
\begin{equation}
\label{eq2_3}
P(y|do(x)) \triangleq P(Y_x = y), 
\end{equation}
where $Y_x = y$ means ``$Y$ would be equal to $y$ if $X$ had been $x$" ~\citep{pearl2009causality}. The backdoor and front-door criteria can be used to identify sets of covariates that should be adjusted for using simple graphical identification rules. However, to uncover all causal effects that can be identified from a given DAG requires symbolic machinery ~\citep{pearl1995causal}.
The problem of identification that asks whether the interventional distribution, $P(Y = y | do (X = x))$, can be estimated from the graph (along with the assumptions) and the available observational data. For queries in the form of $Q$ or \textit{do}-expression, identifiability can be decided procedurally using an algebraic method known as the \textit{do}-calculus. It allows us to manipulate probabilistic distributions, both interventional and observational, through three inference rules whenever certain separation conditions hold in the causal graph $G$ in the SCM $M$. A sufficient version of do-calculus is the adjustment criterion. 

Do-calculus is a causal inference engine that takes three inputs ~\citep{Hunermund2019CausalIA}:
\begin{enumerate}
	 \item A causal quantity $Q$, which is the query we want to answer; 
	 \item A causal graphical model $G$ that encodes the qualitative understanding about the structural dependencies between the variables under study; 
	 \item  A collection of datasets $P(v|\cdot)$ that are available, such as  observational, experimental, from selection-biased samples, or from different populations.
\end{enumerate} 

Based on these three inputs, using \textit{do}-calculus and three inference rules, we can transform interventional probabilistic sentences into equivalent expressions. 

Let $X$, $Y$, $Z$, and $W$ be the arbitrary disjoint sets of nodes in a causal graph $G$. Also, let $G_{\overline{X}}$ be the graph obtained by removing all arrows pointing to nodes in $X$. Likewise, $G_{\underline{X}}$ is a graph obtained by deleting all arrows that are emitted by $X$ in $G$. The notation $G_{\overline{X}\underline{Z}}$ can be used to represent the previous two configurations together, the removal of both arrows incoming in $X$ and arrows outgoing from $Z$. Given this notation, \textit{do}-calculus has the following three rules that are valid for every interventional distribution compatible with the causal graph $G$ ~\citep{pearl2009causality, bareinboim2016causal}.

\begin{rules}: Insertion/deletion of observations                                       
    \begin{equation}           
	    P(y | do(x), z, w) = P(y | do(x), w) \qquad \text{\normalfont if} \enskip (Y \independent Z | X, W)_{G_{\overline{X}}}.        
	\end{equation}                                             
\end{rules} \vspace{-\baselineskip}

\begin{rules}: Action/observation exchange
	\begin{equation}
	    P(y | do(x), do(z), w) = P(y | do(x), z, w) \qquad \text{\normalfont if} \enskip (Y \independent Z | X, W)_{G_{\overline{X}\underline{Z}}}.
	\end{equation}
\end{rules} \vspace{-\baselineskip}

\begin{rules}: Insertion/deletion of actions
	\singlespacing 
	\begin{equation}
	    P(y | do(x), do(z), w) = P(y | do(x), w) \qquad \text{\normalfont if} \enskip (Y \independent Z | X, W)_{G_{\overline{XZ(W)}}},
	\end{equation}
	where $Z(W)$ is the set of Z-nodes that are not ancestors of any W-node in $G_{\overline{X}}$.
\end{rules}

\smallskip 
\noindent 

The above rules of \textit{do}-calculus make it complete for general queries of the form $Q = P(y|do(x),z)$. It is guaranteed to return a solution for the identification problem, whenever such a solution exists ~\citep{pearl1995causal,Bareinboim2012}. To establish identifiability of a causal query $Q$, we need to repeatedly apply the rules of \textit{do}-calculus to $Q$ until an expression is obtained that is do-expression free. This do-free expression makes $Q$ estimable from non-experimental data.

\section{Methods}
    \label{sec:method}
We present an expert-augmented causal estimation framework to estimate the effect of an intervention from observational data as depicted in ~\autoref{fig:causal_estimation_framework}. We discuss the framework, in the following subsections, with respect to a clinical question using EHR data curated during routine care in the ICU. 
\begin{figure}[!h]
    \centering
    \includegraphics[width=1\textwidth]{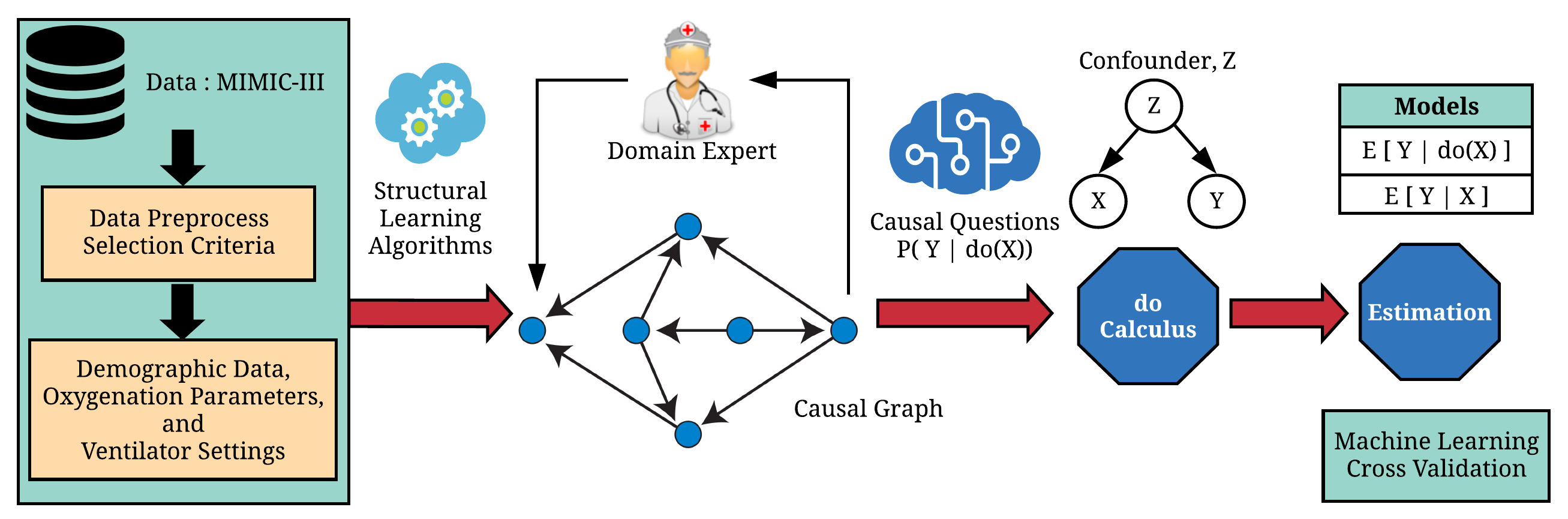}
    \caption{Schematic diagram for the causal estimation framework. The framework allows domain expertise to be encoded in the SCM.}
    \label{fig:causal_estimation_framework}
\end{figure} 

\subsection{Cohort selection}
We first consider the observational data collected during routine care such as from an EHR. EHRs contain data repositories from multiple sources or healthcare facilities. In order to identify a set of eligible patients for the experiment, inclusion-exclusion criteria are defined based on the clinical question of interest.
\subsection{Structure Learning Algorithms}
\label{subsec:sla2}
Given a data set, eligible patients for the study, the framework estimates a possible causal graph that could be generated given the distribution using structure learning algorithms (SLAs) ~\citep{CausalStrucureLearning}. These SLAs learn the underlying causal graph at different levels of granularity under different assumptions. Based on these assumptions, multiple causal graphs were estimated using SLAs. We have considered the following seven SLAs for this framework: 1) PC Algorithm, 2) RFCI, 3) GES, 4) GDS, 5) TABU Search ~\citep{Glover1998}, 6) MMHC, and 7) LINGHAM. 

\subsection{Majority Voting}
Seven different causal graphs were estimated using the SLAs discussed in \autoref{subsec:sla} and \autoref{subsec:sla2}. We then used a majority voting scheme on the causal relationships estimated using the SLAs. If a simple majority, $m/2 + 1$ of algorithms support a causal relationship between $X$ and $Y$, we consider the relationship for our model. Here $m$ is the total number of SLAs under consideration i.e. $7$. If it is supported by less than $m/2 + 1$ algorithms, we discard that functional relation from the model. 
If we consider a causal DAG $G$ has $n$ nodes, then these nodes are enumerated using integers from $1$ to $n$. Therefore, it can have $n$ x $n$ or $n^{2}$ number of directional edges each representing a causal relationship. Let's consider a matrix $V$ of $n$ x $n$ for a causal graph $G$ that stores the number of votes received for each of the causal relationships between nodes $i$ and $j$, two arbitrary nodes where $1 \leq i , j \leq n$. For each element $(i, j)$ in $V$ we consider the voting procedure in ~\autoref{eqn_voting}.
\begin{equation}
\label{eqn_voting}
  V(i, j) =
  \begin{cases}
              0 & \text{initially} \\   
    V(i, j) + 1 & \text{if a SLA supports a causal edge from $i$ to $j$} \\
    V(i, j) + 0 & \text{if a SLA does not support a causal edge from $i$ to $j$} \\
  \end{cases}
\end{equation}
Consider the pair $(u, v)$ representing a directed edge from node $u$ to node $v$ in a causal graph. After counting votes for all the causal relationships in $V$, we draw a causal graph $G$ using ~\autoref{eqn_maj_voting}. 
\begin{equation}
\label{eqn_maj_voting}
G = 
  \begin{cases}
    G \cup (u, v) & \text{if  } V(u, v) \geq \frac{m}{2} + 1  \\   
    G \cup \emptyset & \text{if  } V(u, v) \textless \frac{m}{2} + 1  \\   
  \end{cases}
\end{equation}
where, $1 \leq u, v \textless n$.

\subsection{Expert Augmented Knowledge}
To finalize the causal graph for estimation, we verify the model with expert consultation and review of the current clinical literature. This provides a systematic approach for encoding existing clinical knowledge into the graph. In this step, we review the initial estimation, through majority voting, of the causal graph $G$ by domain experts. It allows us to review causal relationships estimated by majority voting and makes the necessary adjustment based on existing evidence. We perform either of the following five operations based on the expert knowledge described in \autoref{alg:expert_knowledge}. The resultant causal graph is then reviewed again to make sure that it does not contain any cycles. Therefore, the final causal graph is always a DAG.

\begin{algorithm}[t]
\caption{ Encoding Expert Knowledge in to the Causal Graph, $G$}
\label{alg:expert_knowledge}
\textbf{Input}: $G$, $Expert Knowledge$\\
\textbf{Output}: $G$
  \begin{algorithmic}[1]
        \FOR{Each edge $(u, v)$ in $G$}
            \IF{An evidence of a causal relationship from $u$ to $v$}
                \STATE Keep the edge $(u, v)$ in $G$
            \ELSIF{An evidence of a causal relationship from $v$ to $u$}
                \STATE Change the orientation of the edge to $(v, u)$
            \ELSIF{An evidence of no causal relationship between $u$ and $v$}
                \STATE Remove the edge $(u, v)$ from $G$
            \ELSIF{No evidence of a causal relationship between $u$ and $v$}
                \STATE Keep the edge $(u, v)$ in $G$
            \ENDIF
        \ENDFOR
        \FOR{Each node $u$ in $G$}
            \FOR{Each node $v$ in $G - \{u\}$}
                \IF{the edge $(u, v)$ not in $G$}
                    \IF{An evidence of a causal relationship from $u$ to $v$}
                        \STATE Add the edge $(v, u)$ to $G$
                    \ENDIF
                \ENDIF
            \ENDFOR
        \ENDFOR
  \end{algorithmic}
\end{algorithm}

\subsection{Causal Questions}
With the graph encoded with the causal relationships, we can now ask causal questions to understand how variables influence each other. This allows us to run virtual experiments to estimate treatment effects.
The probability that $Y = y$ when we intervene to make $X = x$ is denoted by, $P( Y = y | do(X = x))$. That is to say if everyone in a population had their $X$ value fixed at $x$ then what is the population distribution of $Y$. Using do-expressions and a causal graph, we can ask causal questions and untangle the causal relationships. For our causal questions, $X$ is the treatment variable and $Y$ is the outcome.

\subsection{Causal Identification}
We used causal identification formula such as \textit{backdoor adjustment} derived using the rules of \textit{do-calculus} to find a set of confounding variables, $Z$ that must be adjusted depending on the causal query in the form of $P( Y | do(X))$ and the graphical model \citep{pearl2009causality}. The model automatically finds the confounding variables, can distinguish between confounders and mediators, and provides the causal identification formula. The causal identification formula provides a mathematical transformation between the observational reality and the corresponding experimental reality. It is possible that certain graphical model will not be able to provide the causal identification formula, and in those cases, the method will not be able to answer the query of interest. The resultant formula can be used to evaluate the causal effect. 
\subsection{Estimation}
Once we have these three sets of variables, treatment $X$, confounders $Z$, and outcome $Y$, we can compute the conditional probabilities in the adjustment formula using the probability distribution to compute the interventional effect. Computing each of the conditional probabilities in the adjustment formula is computationally expensive and may not lead to feasible results with small sample sizes. Instead, we use machine learning algorithms to estimate the probability distribution from the data, that is we use treatment $X$ and confounders $Z$ as input to an ML algorithm to estimate the outcome $Y$. For this step, we split our data into training and testing sets, each containing $X$, $Y$, and $Z$. We train two ML models, one based on observational data and another based on experimental data, using the training dataset. We use the testing set to predict $Y$ from $X$ and $Z$ using the trained ML models. After that, we use bootstrapping to randomly subsample $X$ and predicted $Y$ for $k$ iterations. This gives us an estimation of the effect of $X$ on $Y$ and used to estimate the mean and confidence interval.

\section{Results}
    There are approximately 2 to 3 million patients in the US that receive invasive MV in an ICU ~\citep{Wunsch2010, Adhikari2010} with an estimated annual cost of \$15-27 billion.  MV is highly associated with morbidity ~\citep{Kahn2010} and mortality ~\citep{Metnitz2009}. It is important to maintain safe levels of tissue oxygenation estimated through peripheral oxygen saturation (SpO2) ~\citep{Vincent2013} for MV patients in the ICU, most of whom also receive supplemental oxygen therapy (OT) ~\citep{Panwar2016}.  

\begin{figure}[!t]
    \centering
    \includegraphics[width=0.9\textwidth]{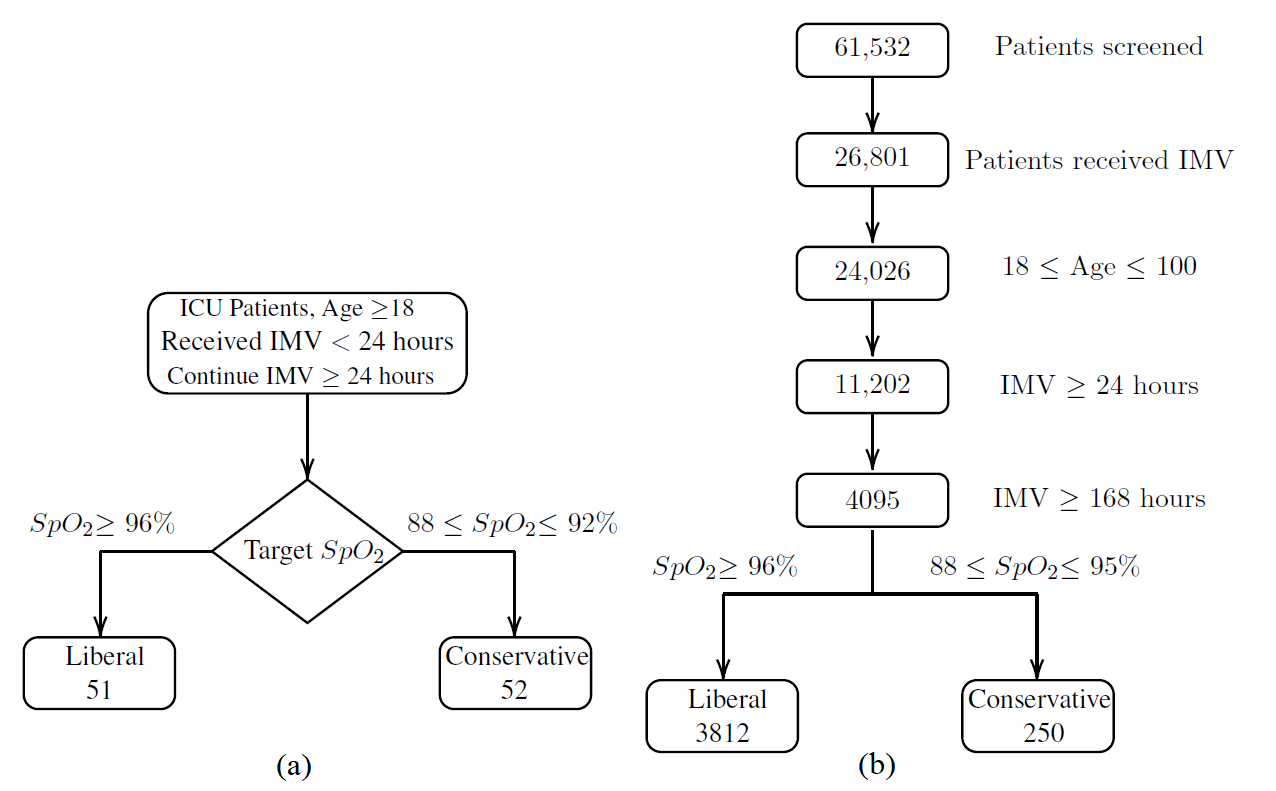}
    \caption{Inclusion-exclusion criteria for (a) pilot RCT and (b) for the observational study with the MIMIC database.}
    \label{fig:inclusion-exclusion-criteria}
\end{figure}

Current practice and recommendations related to oxygenation targets for ICU patients are not based on strong evidence. Rather, it is primarily based on normal physiologic values in healthy adults at sea level, $SpO_2$, $SaO_2$, and $PaO_2$, which are approximately $90\%$, $95-97\%$, and $88-100mm Hg$, respectively ~\citep{Crapo1999}. For acutely ill patients $SpO_2$ target recommendations vary from $90\%$ to $94-98\%$. Also, for patients with acute respiratory distress syndrome (ARDS), $SpO_2$ targets of $88-95\%$ are considered acceptable levels ~\citep{NEJM200005043421801, Meade2008}. This liberal approach to oxygen therapy is now being challenged due to the increasing recognition of the potential harm of excessive $FiO_2$ in the form of hyperoxemia and tissue hyperoxia ~\citep{Panwar2013,  SUZUKI2013647}.

Recent trials have been inconclusive in determining the effect difference between the liberal and conservative oxygen therapy on mortality \citep{Panwar2016, Girardis2016} with a need for large trials. 
Another recent study found no significant effect on the number of ventilator-free days for the use of conservative oxygen therapy compared to conventional oxygen therapy \citep{NEJMoa1903297}. In the clinical practice guideline for oxygen therapy for acutely ill patients, it is strongly recommended not to administer oxygen therapy at higher than 96\% saturation with a few exceptions ~\citep{Siemieniukk4169}. It is also strongly recommended not to provide oxygen therapy at or above 93\% saturation for patients with acute stroke or myocardial infarction. We have presented oxygenation strategies for multiple studies based on target $SpO_{2}$ in ~\autoref{tab:study_demographics}.

Our goal is to determine the feasibility of a conservative oxygenation strategy as an alternative to the conventional liberal oxygenation approach. We use the computational model, discussed in \autoref{sec:method}, developed using causal inference methods to estimate the effect of different oxygenation strategies for adult patients receiving MV in the ICU. We used SCMs to identify and represent causal relationships between variables. The overall framework is presented in ~\autoref{fig:causal_estimation_framework}. %
We describe the experimental details in the following subsections. 
 
\subsection{Selected Cohort}

\begin{table}[h]
\centering
\caption{Summary statistics of the patients. The first two columns have summary statistics for the cohort selected from MIMIC database and the last two columns have summary statistics for the OT-RCT.}
\label{tab:summary-statistics-table}
\resizebox{0.7\textwidth}{!}{%
\begin{tabular}{ccccc}
\hline
  & \multicolumn{2}{c}{\textbf{SCM - VRCT}} & \multicolumn{2}{c}{\textbf{OT-RCT}}                                                           \\ \cline{2-5}
\textbf{Variable} & \textbf{Conservative} & \textbf{Liberal} & \textbf{Conservative} & \textbf{Liberal} \\ \hline
Age               & 60.82(15.6)                    & 62.94(16.22)                          & 62.4(14.9)              & 62.4(17.4)         \\ 
Male Sex          & 150 (60\%)                     & 2197(57.09\%)                         & 32(62\%)                & 33(65\%)           \\ 
BMI               & 28.1 (10.5)                 & 27.8 (10.7)                        & 27.6(10.3)              & 27.6(10.1)         \\ \hline
\multicolumn{5}{l}{Diagnosis type}                                                                                                      \\ 
Trauma            & 25(10\%)                       & 466(12.22\%)                          & 3(6\%)                  & 2(4\%)             \\ 
Medical           & 15(6\%)                        & 152(4\%)                              & 39(75\%)                & 41(80\%)           \\ 
Surgery           & 8(3.2\%)                       & 93(2.43\%)                            & 10(19\%)                & 8(16\%)            \\ \hline
APSIII            & 50.44(36-61.75)                & 54.05(38-67)                          & 79.5(61-92.5)           & 70(50-84)          \\ 
SOFA              & 6.96(3.43)                     & 6.27(3.73)                            & 7.9(2.9)                & 7.4(3.1)           \\ 
Smoker            & 36(14.4\%)                     & 345(9.05\%)                           & 10(19\%)                & 14(27\%)           \\ 
COPD              & 21(8.4\%)                      & 154(4.04\%)                           & 11(21\%)                & 5(10\%)            \\ 
Ischemic HD       & 11(4.4\%)                      & 151(3.96\%)                           & 6(12\%)                 & 5(10\%)            \\ 
ARDS              & 13(5.2\%)                      & 130(3.4\%)                            & 17(33\%)                & 10(20\%)           \\ 
SpO2              & 94.25(3.14)                    & 97.88(3.42)                           & 95(3)                   & 96(3)              \\ 
FiO2              & 0.59(0.18)                     & 0.47(0.14)                            & 0.44(0.2)               & 0.44(0.18)         \\ 
SaO2              & 91.39(9.20)                    & 91.78(11.84)                          & 95.5(3)                 & 96(2.7)            \\ 
PaO2              & 89.80(69-98)                   & 121.32(86-140)                        & 81(68-109)              & 82(75-104)         \\ 
PaCO2             & 49.19(13.03)                   & 41.78(9.67)                           & 38(7)                   & 39(6)              \\ 
pH                & 7.37(0.09)                     & 7.39(0.08)                            & 7.36(0.07)              & 7.37(0.07)         \\ 
Lactate           & 2.23(1.2-2.5)                  & 2.46(1.2-2.7)                         & 1.99(1.4-2.9)           & 1.65(1.2-2.6)      \\ 
Hemoglobin        & 10.07(1.92)                    & 9.86(1.68)                            & 110(23)                 & 115(23)            \\ 
PEEP              & 10.8(4.82)                     & 7.02(3.31)                            & 8.2(3)                  & 7.3(3)             \\ 
VT                & 505.48(133.3)                  & 525.56(2241.31)                       & 8(1.8)                  & 8(1.9)             \\ 
Peak Air Prs.     & 29.83(9.39)                    & 24.98(8.76)                           & 22(6)                   & 21(5)              \\ \hline
No. of Patients    & 250                            & 3812                                  & 52                      & 51                 \\ 
\end{tabular}%
}
\end{table}

We closely followed the study guideline and selection criteria used in OT-RCT. This trial was conducted at four multidisciplinary ICUs in Australia, New Zealand, and France. The goal was to investigate whether a conservative oxygenation strategy is a feasible alternative to a liberal oxygenation strategy among ICU patients requiring invasive MV. A total of 103 adult patients, age greater than or equal to 18 years, were eligible for the study as they were given invasive MV for less than 24 hours and their attending clinician expected the invasive MV to continue for at least the next 24 hours. These patients were randomly treated with either a conservative oxygenation strategy with target $SpO_{2}$ of $88-92\%$ and a liberal oxygenation strategy with target $SpO_{2}$ of greater than or equal to $96\%$. There were 52 and 51 patients in the conservative and liberal arms respectively. The adjusted hazard ratio for 90-day mortality in the conservative arm was $0.77 (95\% CI,0.40-1.50, P = 0.44)$ overall. There were no significant between-group differences  and no harm was associated with the use of the  conservative oxygenation strategy.

\begin{figure}[!t]
    \centering
    \resizebox{.6\textwidth}{!}{
        \begin{tikzpicture}
        \begin{axis}[
            height=10cm,
        	width=13cm,
        	grid=minor,
        	xmin=84,   xmax=100,
        	ymin=0,   ymax=33,
        	legend style={at={(0.5,-0.15)},
		legend pos=north west,legend columns=2},
		xlabel={$SpO_{2}$ (\%)},
        ylabel={\% Time receiving mechanical ventilation}
        ]
        \addplot[color=purple, mark=*] table [x=spo2_levels, y=conservative, col sep=comma] {data.csv};
        \addlegendentry{Conservative ~\citep{Panwar2016}}
        \addplot[color=green, mark=*] table [x=spo2_levels, y=liberal_mimic, col sep=comma] {data.csv};
        \addlegendentry{Liberal: MIMIC}
        \addplot[color=blue, mark=*] table [x=spo2_levels, y=liberal, col sep=comma] {data.csv};
        \addlegendentry{Liberal: ~\citep{Panwar2016}}
        \addplot[color=red, mark=*] table [x=spo2_levels, y=conservative_mimic, col sep=comma] {data.csv};
        \addlegendentry{Conservative: MIMIC}

        \end{axis}
        \end{tikzpicture}
    }
    \caption{Pooled frequency histogram of the percentage of time spent at each $SpO_{2}$ level in both oxygenation groups for OT-RCT \citep{Panwar2016} and our study, SCM-VRCT. $SpO_{2}$ is oxygen saturation as measured by pulse oximetry.}
    \label{fig:spo2_time}
\end{figure}

For our study, we considered a large ICU database containing data routinely collected from patients in the US. Adult patients of age at least 18 years old were considered. We used the MIMIC-III dataset for model development and evaluation\citep{johnson2016mimic}. The database comprises a total of 61,532 ICU stays at Beth Israel Deaconess Medical Center in Boston, MA in which 53,432 stays are for adult patients and 8,100 stays are for neonatal patients. The database includes data logged using the CareVue and Metavision electronic health record (EHR) systems spanning from June 2001 to October 2012, a little over 11 years.

We included adult patients with eligibility criteria described in ~\autoref{fig:inclusion-exclusion-criteria} (b). We closely followed the inclusion-exclusion criteria in ~\autoref{fig:inclusion-exclusion-criteria} (a) employed in OT-RCT ~\citep{Panwar2016}. After the exclusion of ineligible cases, we included 4,062 ICU patients from the database. Out of 4,062 patients, 3,812 patients received liberal oxygenation and 250 patients received conservative oxygenation. Patient demographics and clinical characteristics are shown in ~\autoref{tab:summary-statistics-table}. We also present the demographics and clinical characteristics of the patients from the OT-RCT in ~\autoref{tab:summary-statistics-table} for comparison. Patients spent the majority of the time within the intended target range in both liberal and conservative groups. The pooled frequency histogram of the percentage of time spent at each $SpO_{2}$ level for both  groups is shown in ~\autoref{fig:spo2_time}. 

Our inclusion-exclusion criteria ensure that our study, SCM-VRCT has a similar  distribution of $SpO_{2}$ in the liberal and conservative arms for observational data compared to the OT-RCT. We extracted a set of 26 variables, including demographics, ventilator settings, and oxygenation parameters, from the dataset. Data on oxygenation parameters and ventilator settings were extracted every 4 hours and for each patient, we included at least 168 hours of measurements.

\begin{figure}[t]
	\centering
    \includegraphics[width=0.9\textwidth]{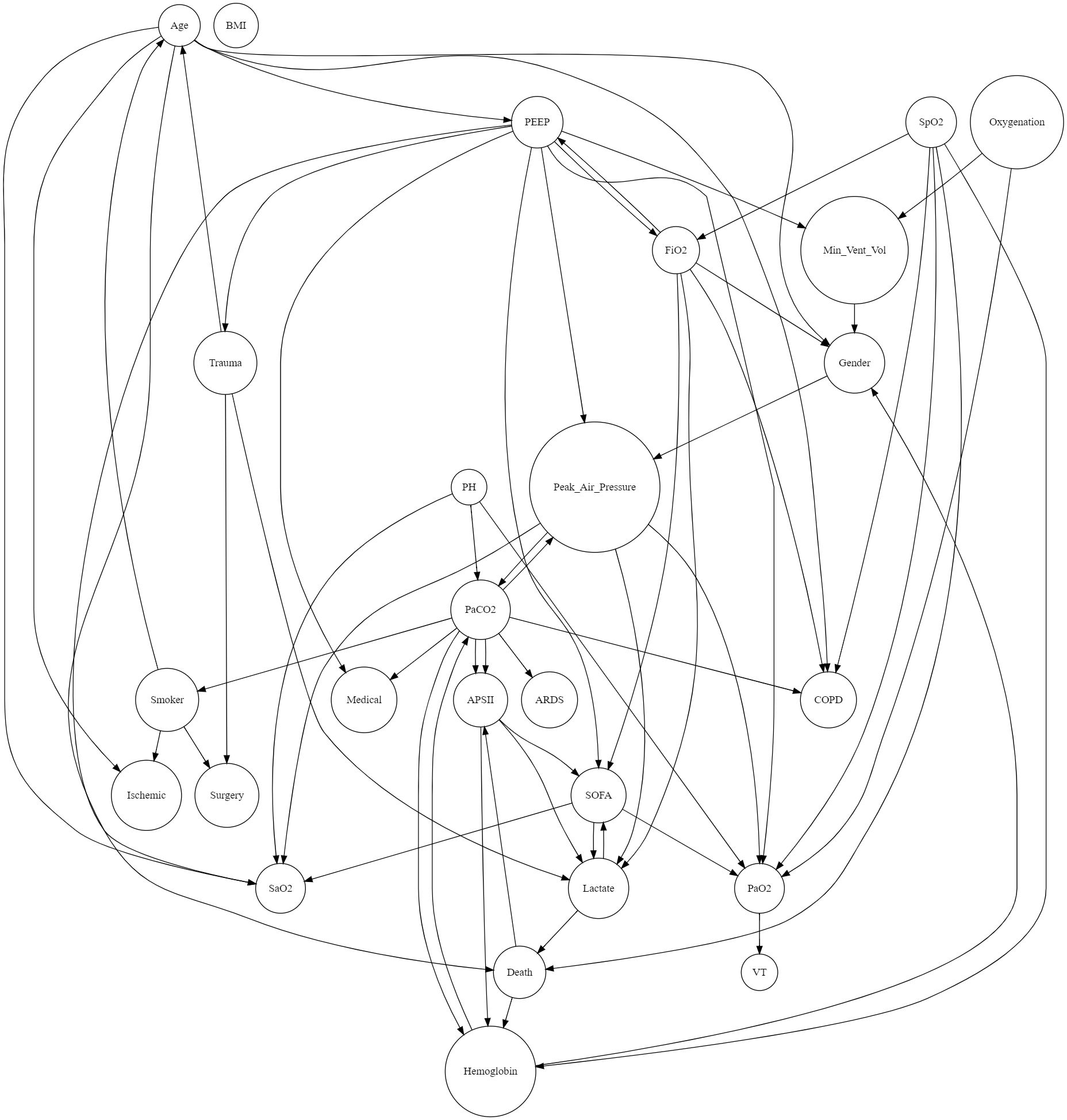}
	\caption{Causal Graph $G$ based on the majority voting of SLAs. $oxygenation$ is the treatment variable, and $death$ is the outcome variable.}
	\label{causal_graph_maj}
\end{figure}

\subsection{Structure Learning Algorithm}
Our study has focused on estimating the effect of different oxygenation strategies by establishing the causal relationship between measured oxygenation parameters, ventilator settings, disease state, and demographic information. Demographic data and computed mean were then used as features to estimate causal graphs using SLAs. For each algorithm, the data for the variables in ~\autoref{tab:summary-statistics-table} were used to learn the causal relationship among them and then estimated a causal graph.  We estimated 7 causal graphs using SLAs we discussed in \autoref{subsec:sla}.

\subsection{Majority Voting}
Based on the estimation by SLAs, we computed the matrix $V$ for voting, using ~\autoref{eqn_voting}. The resultant matrix is shown in ~\autoref{tab:maj_voting}, which has 26 rows and 26 columns. Each row has a variable for which we have the vote count in the columns for all 26 variables.

For example, the first row has the number of votes for variable \textit{age} and the second row has the number of votes for \textit{gender}. The value \textit{6} at row 5 (variable \textit{trauma}) and column 4 (variable \textit{surgery}) means that the causal relationship from \textit{trauma} to \textit{surgery} has received 4 out of the total number of votes, 7. As this vote count is greater than or equal to $7/2 + 1$ or $4$, the majority voting, we include this causal relationship in our final causal graph. The estimated graph using the majority voting (\autoref{eqn_voting}) based on the vote count in \autoref{tab:maj_voting} is shown in ~\autoref{causal_graph_maj}.

\subsection{Encoding Expert Knowledge}
We developed the causal graph from the result of the majority voting of estimated graphs by structural learning algorithms. We then incorporated clinical domain knowledge by consulting the domain expert, a critical care physician. We reviewed each causal relationship estimated by the algorithms. This review phase of each of the edges from variable $X$ to variable $Y$ involves three operations: 1) keep the edge if it is supported by the domain knowledge, 2) change the direction or orientation of the edge if that is supported by the domain knowledge, and 3) remove the edge if it is not supported by the domain knowledge. 
~\autoref{tab:domain_knowledge} shows the process of incorporating the domain expert’s input. $Node_{from}$ $\,\to\,$ $Node_{to}$ means $Node_{from}$ causes $Node_{to}$. For example, in row 8, $age$ causes $COPD$. The column, majority voting (MJV) is ``Yes" which means that it was estimated by the majority voting of the structural learning algorithms, whereas a ``No" indicates it was not. The \textit{Expert Knowledge} column represents the addition of a new causal relationship, removal of an edge, or change of the orientation of an existing causal edge followed by the reference of the domain knowledge in the next column. In total, we added 16 new edges, removed 10 edges, and changed directions of 14 edges as a result of the process. We then review the graph to ensure that the graph is acyclic with clinician consultation. This process allows us to make sure that the causal graph encodes the clinical knowledge along with the data generating process.

\begin{landscape}

\begin{table}[]
\centering
\caption{Vote Counts For All Causal Edges}
\label{tab:maj_voting}
\begin{adjustbox}{width=1.2\textwidth}
\small
\begin{tabular}{lllllllllllllllllllllllllll}
\\\hline
        & \rotatebox[origin=c]{90}{ Age} & \rotatebox[origin=c]{90}{ Gender} & \rotatebox[origin=c]{90}{ BMI } & \rotatebox[origin=c]{90}{ Surgery} & \rotatebox[origin=c]{90}{ Trauma} & \rotatebox[origin=c]{90}{ Medical} & \rotatebox[origin=c]{90}{ APSiii} & \rotatebox[origin=c]{90}{ SOFA} & \rotatebox[origin=c]{90}{ Smoker} & \rotatebox[origin=c]{90}{ COPD} & \rotatebox[origin=c]{90}{ Ischemic HD} & \rotatebox[origin=c]{90}{ ARDS} & \rotatebox[origin=c]{90}{ Death} & \rotatebox[origin=c]{90}{ Oxygenation} & \rotatebox[origin=c]{90}{ $SpO_{2}$} & \rotatebox[origin=c]{90}{ $FiO_{2}$} & \rotatebox[origin=c]{90}{ $SaO_{2}$} & \rotatebox[origin=c]{90}{ $PaO_{2}$} & \rotatebox[origin=c]{90}{ $PaCO_{2}$} & \rotatebox[origin=c]{90}{ pH} & \rotatebox[origin=c]{90}{ Lactate}	& \rotatebox[origin=c]{90}{ Hemoglobin} &	\rotatebox[origin=c]{90}{ Peep}	& \rotatebox[origin=c]{90}{ VT}	& \rotatebox[origin=c]{90}{ Peak Air Pressure }	& \rotatebox[origin=c]{90}{ Minute Ventilation }
\\\hline

Age & 0 & 4 & 0 & 0 & 3 & 0 & 0 & 0 & 1 & 6 & 5 & 0 & 6 & 0 & 1 & 0 & 5 & 2 & 0 & 1 & 2 & 0 & 4 & 0 & 2 & 0 \\
Gender & 3 & 0 & 0 & 0 & 0 & 0 & 0 & 2 & 0 & 0 & 0 & 0 & 0 & 0 & 0 & 2 & 0 & 0 & 0 & 0 & 0 & 3 & 0 & 0 & 4 & 2 \\
BMI & 0 & 0 & 0 & 0 & 0 & 0 & 1 & 0 & 0 & 0 & 0 & 0 & 0 & 0 & 0 & 0 & 0 & 0 & 0 & 0 & 0 & 0 & 0 & 0 & 0 & 0 \\
Surgery & 0 & 0 & 0 & 0 & 2 & 0 & 0 & 0 & 3 & 0 & 0 & 0 & 1 & 0 & 0 & 0 & 0 & 0 & 3 & 0 & 0 & 0 & 0 & 0 & 2 & 0 \\
Trauma & 4 & 1 & 0 & 6 & 0 & 0 & 0 & 1 & 0 & 0 & 0 & 0 & 0 & 0 & 0 & 0 & 0 & 0 & 1 & 0 & 5 & 0 & 3 & 0 & 0 & 0 \\
Medical & 0 & 0 & 0 & 0 & 0 & 0 & 0 & 0 & 0 & 0 & 0 & 0 & 0 & 0 & 0 & 0 & 0 & 0 & 3 & 0 & 0 & 3 & 2 & 0 & 0 & 0 \\
APSiii & 1 & 1 & 2 & 0 & 0 & 0 & 0 & 6 & 0 & 0 & 0 & 0 & 3 & 1 & 2 & 2 & 0 & 3 & 3 & 0 & 5 & 4 & 0 & 0 & 0 & 0 \\
SOFA & 0 & 2 & 0 & 2 & 2 & 0 & 3 & 0 & 0 & 2 & 0 & 0 & 1 & 1 & 1 & 3 & 5 & 4 & 0 & 3 & 4 & 2 & 3 & 0 & 1 & 0 \\
Smoker & 4 & 0 & 0 & 5 & 0 & 0 & 0 & 0 & 0 & 2 & 4 & 0 & 0 & 0 & 0 & 0 & 0 & 0 & 2 & 0 & 0 & 0 & 0 & 0 & 0 & 0 \\
COPD & 2 & 0 & 0 & 0 & 0 & 0 & 0 & 0 & 2 & 0 & 0 & 0 & 0 & 0 & 1 & 2 & 0 & 1 & 1 & 0 & 3 & 0 & 0 & 0 & 0 & 0 \\
Ischemic HD & 3 & 0 & 0 & 0 & 0 & 0 & 0 & 0 & 2 & 0 & 0 & 0 & 0 & 0 & 0 & 0 & 0 & 0 & 2 & 0 & 0 & 0 & 0 & 0 & 0 & 0 \\
ARDS & 0 & 0 & 0 & 0 & 0 & 0 & 0 & 0 & 0 & 0 & 0 & 0 & 0 & 0 & 0 & 0 & 0 & 0 & 1 & 0 & 0 & 0 & 0 & 0 & 0 & 0 \\
Death & 2 & 0 & 0 & 0 & 0 & 0 & 5 & 0 & 0 & 0 & 0 & 0 & 0 & 0 & 1 & 1 & 0 & 0 & 0 & 0 & 3 & 4 & 0 & 0 & 2 & 0 \\
Oxygenation & 1 & 0 & 0 & 0 & 2 & 0 & 1 & 1 & 0 & 2 & 0 & 0 & 2 & 0 & 3 & 1 & 0 & 5 & 0 & 2 & 2 & 2 & 1 & 0 & 0 & 4 \\
SpO\textsubscript{2} & 2 & 0 & 0 & 0 & 1 & 0 & 1 & 2 & 1 & 4 & 0 & 0 & 6 & 2 & 0 & 4 & 2 & 5 & 2 & 2 & 0 & 5 & 3 & 0 & 2 & 2 \\
FiO\textsubscript{2} & 0 & 5 & 0 & 0 & 0 & 0 & 1 & 6 & 2 & 5 & 0 & 0 & 2 & 1 & 3 & 0 & 0 & 1 & 2 & 2 & 5 & 1 & 4 & 2 & 3 & 0 \\
SaO\textsubscript{2} & 3 & 0 & 0 & 0 & 0 & 0 & 0 & 3 & 0 & 0 & 0 & 0 & 0 & 0 & 0 & 0 & 0 & 0 & 0 & 2 & 0 & 0 & 0 & 0 & 1 & 0 \\
PaO\textsubscript{2} & 1 & 0 & 0 & 0 & 0 & 0 & 1 & 3 & 0 & 1 & 0 & 0 & 0 & 2 & 3 & 1 & 0 & 0 & 0 & 2 & 0 & 3 & 1 & 5 & 3 & 0 \\
PaCO\textsubscript{2} & 0 & 0 & 0 & 3 & 0 & 5 & 4 & 0 & 5 & 5 & 5 & 6 & 0 & 0 & 0 & 1 & 2 & 0 & 0 & 3 & 0 & 4 & 0 & 0 & 4 & 0 \\
pH & 1 & 0 & 0 & 0 & 0 & 1 & 0 & 0 & 0 & 0 & 1 & 1 & 0 & 0 & 0 & 0 & 5 & 6 & 4 & 0 & 2 & 0 & 1 & 0 & 0 & 0 \\
Lactate & 0 & 0 & 0 & 0 & 3 & 0 & 3 & 4 & 0 & 3 & 0 & 0 & 5 & 0 & 0 & 2 & 0 & 0 & 0 & 3 & 0 & 0 & 0 & 0 & 1 & 1 \\
Hemoglobin & 0 & 5 & 0 & 0 & 0 & 2 & 3 & 1 & 0 & 1 & 0 & 2 & 3 & 0 & 2 & 1 & 0 & 2 & 4 & 2 & 2 & 0 & 0 & 0 & 0 & 0 \\
Peep & 3 & 2 & 0 & 0 & 5 & 6 & 1 & 5 & 0 & 0 & 0 & 3 & 0 & 1 & 3 & 5 & 4 & 6 & 2 & 1 & 1 & 0 & 0 & 0 & 5 & 4 \\
VT & 0 & 0 & 0 & 0 & 0 & 0 & 0 & 0 & 0 & 0 & 0 & 0 & 0 & 0 & 0 & 0 & 0 & 1 & 0 & 0 & 0 & 1 & 0 & 0 & 1 & 1 \\
Peak AP & 1 & 0 & 0 & 0 & 0 & 0 & 0 & 0 & 0 & 1 & 0 & 1 & 1 & 0 & 0 & 0 & 7 & 4 & 4 & 1 & 7 & 0 & 3 & 0 & 0 & 0 \\
Min. Vent. & 0 & 5 & 0 & 0 & 0 & 0 & 0 & 0 & 0 & 3 & 0 & 0 & 0 & 1 & 1 & 1 & 0 & 0 & 0 & 0 & 1 & 0 & 3 & 0 & 0 & 0 \\

\end{tabular}%
\end{adjustbox}
\end{table}

\end{landscape}

\begin{table}[h]
\centering
\caption{Expert Knowledge encoding in the causal graph}
\label{tab:domain_knowledge}
\resizebox{\textwidth}{!}{%
\begin{tabular}{lllll}
\hline
Node\textsubscript{from} & Node\textsubscript{to} & MJV\footnotemark & Expert Knowledge & Literature (Evidence)  \\\hline
Age &	Gender &	Yes &	Removed & Domain Expert\\
Age &	Trauma &	Yes &	Changed Orientation & Domain Expert; ~\citep{Ottochian2009} \\  
Age &	Smoker &	Yes &	Changed Orientation & Domain Expert \\
APSiii &	Death &	Yes &	Changed Orientation & Domain Expert; ~\citep{GallJR1984} \\
APSiii &	Peak AP &	No &	Added & Domain Expert \\
ARDS &	PaCO\textsubscript{2} &	Yes &	Changed Orientation & Domain Expert; ~\citep{Pham2017} \\
ARDS &	Medical &	No &	Added & Domain Expert \\
BMI &	Oxygenation &	No &	Added & Domain Expert; ~\citep{Pham2017} \\
BMI &	Peak AP &	No &	Added & Domain Expert; ~\citep{Pham2017} \\
BMI &	PaCO\textsubscript{2} & No &	Added & Domain Expert; ~\citep{Pham2017} \\
COPD &	FiO\textsubscript{2} &	Yes &	Changed Orientation & Domain Expert; ~\citep{Pham2017} \\
COPD &	PaCO\textsubscript{2} &	Yes &	Changed Orientation & Domain Expert; ~\citep{Pham2017} \\
COPD &	Medical &	No &	Added & Domain Expert; \\
COPD &	Peak AP &	No &	Added & Domain Expert; ~\citep{Pham2017} \\
FiO\textsubscript{2} &	Gender &	Yes &	Removed & Domain Expert \\
FiO\textsubscript{2} &	Oxygenation &	No &	Added & Domain Expert; ~\citep{Pham2017} \\
Gender &	Min. Vent. &	Yes &	Changed Orientation & Domain Expert; ~\citep{Pham2017} \\
Gender &	Hemoglobin &	Yes &	Changed Orientation & Domain Expert  \\
Hemoglobin &	Death &	Yes &	Changed Orientation & Domain Expert  \\
Hemoglobin &	APSiii &	Yes &	Changed Orientation & Domain Expert  \\
Hemoglobin &	Oxygenation &	No &	Added & Domain Expert \\
SpO\textsubscript{2} &	Medical &	No &	Added & Domain Expert \\
Lactate &	Medical &	No &	Added & Domain Expert;  \\
Min Vent. &	Oxygenation &	Yes &	Changed Orientation & Domain Expert; ~\citep{Pham2017} \\
Oxygenation &	Death &	No &	Added & Domain Expert \\
PaCO\textsubscript{2} &	Ischemic HD &	Yes &	Removed & Domain Expert \\
PaCO\textsubscript{2} &	Hemoglobin &	Yes \footnotemark &	Removed & Domain Expert \\
PaCO\textsubscript{2} &	Peak AP &	Yes \footnotemark[\value{footnote}] &	Removed & Domain Expert \\
PaO\textsubscript{2} &	Oxygenation &	Yes &	Changed Orientation & Domain Expert; ~\citep{Pham2017} \\
Peak AP &	Oxygenation &	No &	Added & Domain Expert; ~\citep{Pham2017} \\
PEEP &	SOFA &	Yes &	Removed & Domain Expert  \\
PEEP &	Medical &	Yes &	Removed & Domain Expert  \\
PEEP &	Peak AP &	Yes &	Removed & Domain Expert \\
PEEP &	Min. Vent. &	Yes &	Removed & Domain Expert;  \\
PEEP &	FiO\textsubscript{2} &	Yes\footnotemark[\value{footnote}] &	Changed Orientation & Domain Expert; ~\citep{Carpio2020} \\
Smoker &	Surgery &	Yes &	Removed & Domain Expert \\
Smoker &	PaCO\textsubscript{2} &	Yes &	Changed Orientation & Domain Expert; ~\citep{MUNRO200621} \\
Smoker &	Medical &	No &	Added & Domain Expert \\
SOFA &	Lactate &	Yes\footnotemark[\value{footnote}] &	Changed Orientation & Domain Expert; ~\citep{Jansen2009} \\
SpO\textsubscript{2} &	Oxygenation &	No &	Added & Domain Expert; ~\citep{Pham2017} \\
VT &	Oxygenation &	No &	Added & Domain Expert; ~\citep{Pham2017} \\
\end{tabular}%
}
\end{table}
\footnotetext[1]{Majority Voting}
\footnotetext[2]{Bidirectional edge}

The final graph after incorporating domain knowledge with expert reviews is shown in ~\autoref{causal_graph_dom}. The graph represents the causal relationship between the 26 variables under consideration after majority voting using SLAs and the encoding of the expert knowledge. 
\begin{figure}[t]
	\centering
	\includegraphics[width=1\textwidth]{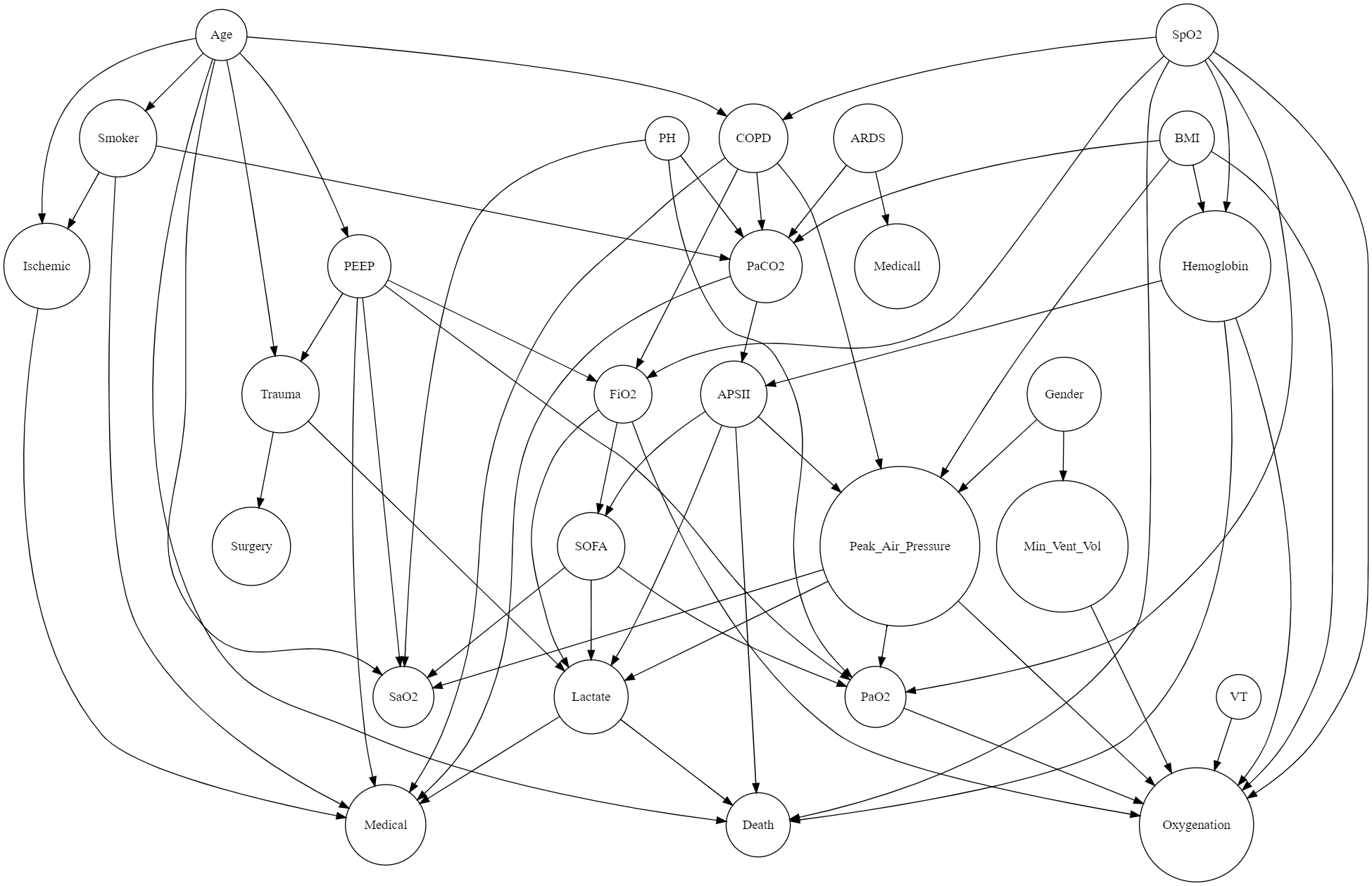}
	\caption{Causal Graph $G$ after incorporating domain knowledge through expert reviews. $oxygenation$ is the treatment variable, and $death$ is the outcome variable.}
	\label{causal_graph_dom}
\end{figure}

\subsection{Conservative versus liberal oxygenation strategies}
The percentage time spent on any mandatory mode of MV during the study time period in the conservative and liberal group is 51\%, hence patients spent the majority of the time within the intended target range for both of the groups. Overall, patients in the conservative group spent more time off- target than the patients in the liberal group. Mean $SpO_{2}$, $PaO_{2}$, and $FiO_{2}$ were well separated for all patients between these groups. Patients in the conservative group spent more time at a $SpO_{2}$ of $90-95\%$ than those in the liberal group (\autoref{fig:spo2_time}). 

We performed a series of experiments based on the causal graph and the selected cohort from the observational data. For each of the following oxygenation strategies, discussed in the subsequent subsections, we have obtained a formula with an admissible set by adjustment using the rules of \textit{do-calculus}. To illustrate the full picture, the derivation of the query in \autoref{sec:severity} presented in the appendix, ~\autoref{apx:query} and ~\autoref{apx:derivation}. We also present the summary of the experimental results for CB-OBS and SCM-VRCT in \autoref{tab:result_vrct}

\subsubsection{Effect on mortality for conservative and liberal oxygenation strategies}
\sloppy
In our first experiment, we estimated the effect of conservative and liberal oxygenation on mortality, $P_{oxygenation} (death)$. The expectations of mortality for conservative and liberal oxygenation strategies are shown in ~\autoref{fig:query_oxy_spo2_death} for both the CB-OBS and SCM-VRCT. For the CB-OBS, we found that the expectation of mortality is higher, 0.54, for the conservative oxygenation than the expectation of the mortality, 0.362 (95\% CI [0.347, 0.377]) for the liberal oxygenation, as shown in \autoref{fig:oxy_death_cb_os}. However, for the SCM-VRCT, the expectation for mortality for the conservative oxygenation strategy is also 0.54 (95\% CI [0.521, 0.558]) which is the same as the expectation we estimated in the CB-OBS. The expected mortality for the liberal approach in SCM-VRCT is 0.34 (95\% CI [0.33, 0.353]), as shown in \autoref{fig:oxy_death_scm_vrct}. The variance for the effect of oxygenation on mortality in SCM-VRCT is significantly lower, 0.01 compared to the results in CB-OBS, 0.49. We also present the effect of $SpO_{2}$ on mortality for SCM-VRCT in \autoref{fig:spo2_death_scm_vrct} 

\subsubsection{Age-specific effect on mortality for oxygenation strategies} 
\sloppy
In the second covariate-specific experiment, we estimated the effect of oxygenation on mortality conditioned on age, $P_{oxygenation} (death | age)$. The expectations of the mortality for conservative and liberal oxygenation for both the CB-OBS and SCM-VRCT are shown in ~\autoref{fig:p_oxy_death_age_apsiii_sofa}. For CB-OBS, we found that the age of the patient does not have any effect on the expected mortality. The expectation of mortality in  SCM-VRCT for the liberal oxygenation strategy conditioned on age is 0.298 (95\% CI [0.285, 0.313]), see \autoref{fig:oxy_death_age_scm_vrct}, which is lower compare to non-age specific CB-OBS, 0.362 (95\% CI [0.347, 0.377]). On the other hand, the expectations of the mortality for conservative oxygenation for SCM-VRCT and CB-OBS are 0.529 (95\% CI [0.505, 0.55]) and 0.54 (95\% CI [0.478, 0.60]), respectively.

\begin{figure}[t]
\centering
\subfigure[SCM-VRCT]{
\includegraphics[width=.3\textwidth]{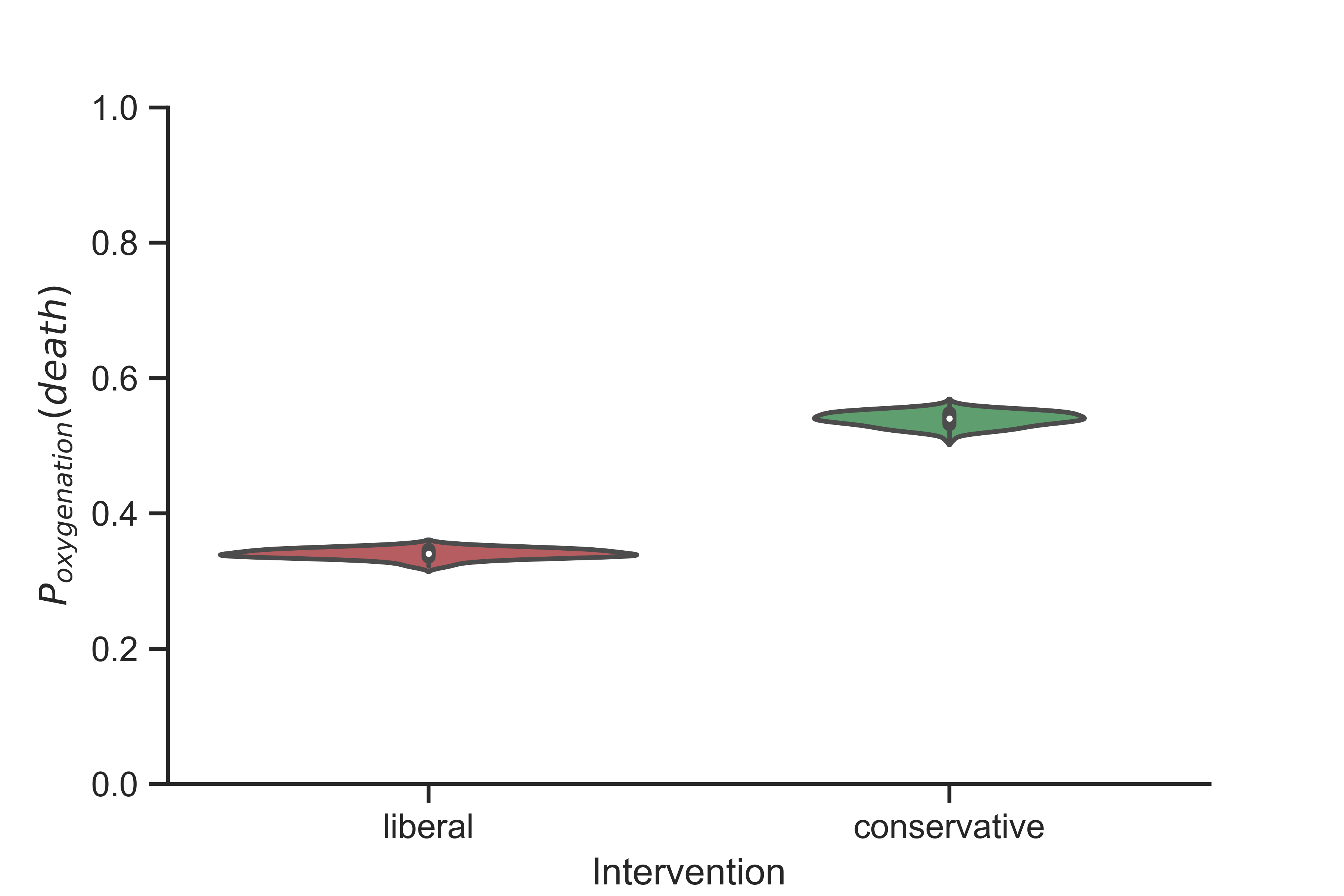}
\label{fig:oxy_death_scm_vrct}
}
\subfigure[SCM-VRCT]{
\includegraphics[width=.3\textwidth]{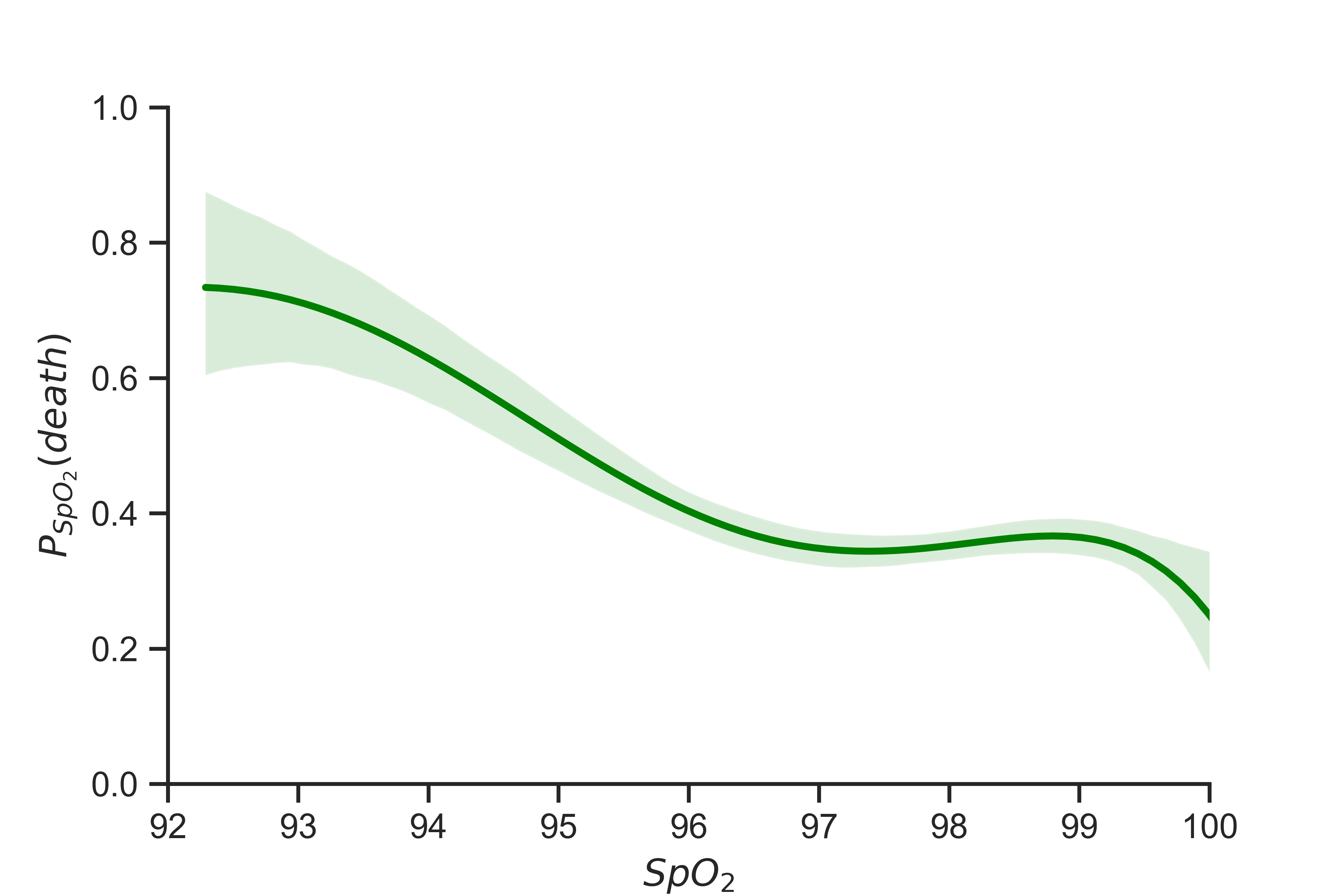}
\label{fig:spo2_death_scm_vrct}
}
\subfigure[CB-OBS]{
\includegraphics[width=.3\textwidth]{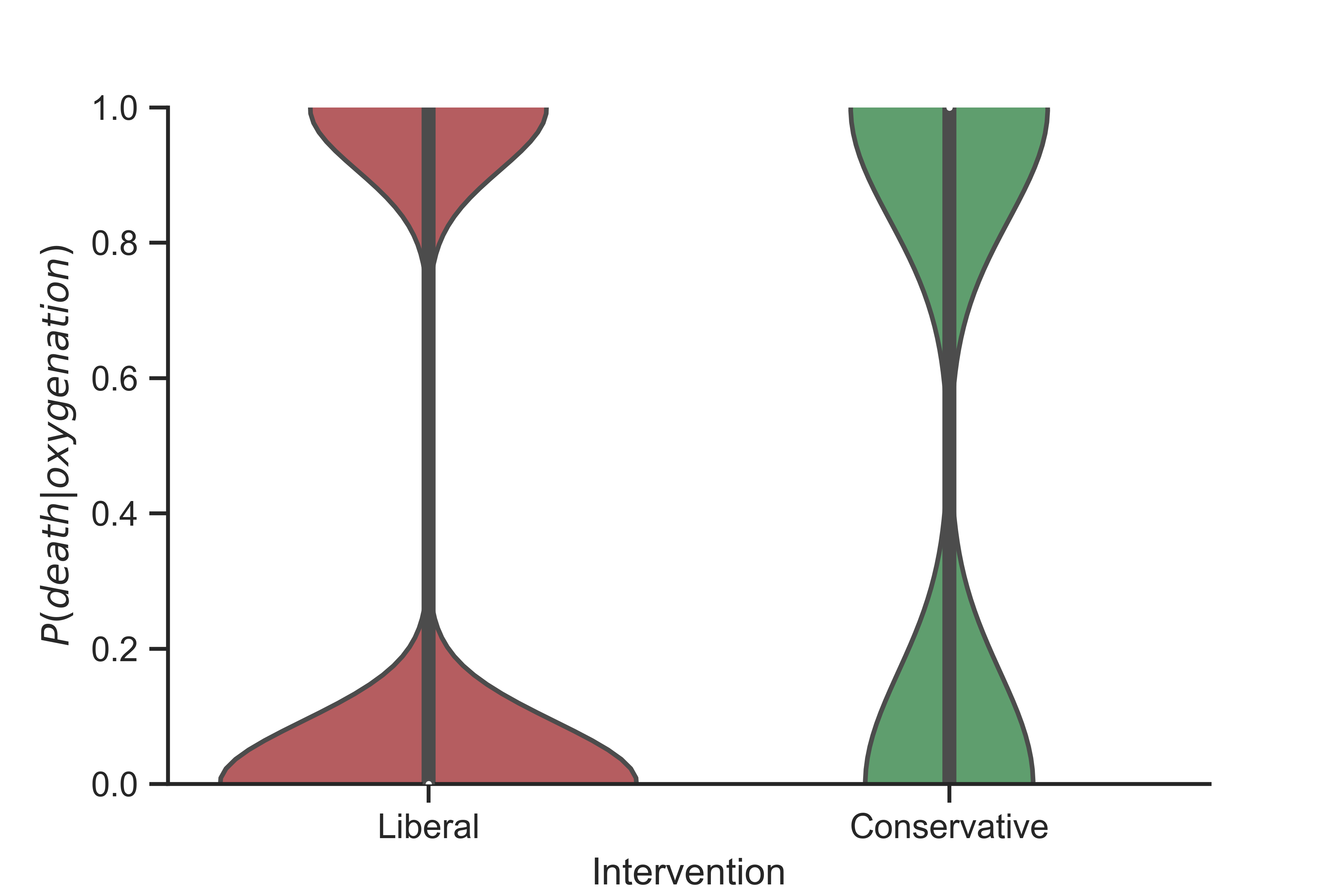}
\label{fig:oxy_death_cb_os}
}

\caption{Estimation of queries on the causal graph for both observational and experimental dataset. (a) $P_{oxygenation}(death)$, (b) $P_{SpO_{2}}(death)$ and (c) $P(death | oxygenation)$}
\label{fig:query_oxy_spo2_death}
\end{figure}

\begin{table}[b]
\centering
\caption{Summary Results of CB-OBS and SCM-VRCT}
\label{tab:result_vrct}
\resizebox{\textwidth}{!}{%
\begin{tabular}{lccccccc}
\multicolumn{1}{c}{\textbf{Virtual Experiments}} & \textbf{Study} & \multicolumn{3}{c}{\textbf{Liberal}} & \multicolumn{3}{c}{\textbf{Conservative}} \\ \midrule
 & \multicolumn{1}{l}{} & \multicolumn{1}{l}{Mean} & \multicolumn{1}{l}{Std} & \multicolumn{1}{l}{Upper 95\% CI} & \multicolumn{1}{l}{Mean} & \multicolumn{1}{l}{Std} & \multicolumn{1}{l}{Upper 95\% CI} \\ 
$P(death | Oxygenation)$ & CB-OBS & 0.362 & 0.481 & 0.377 & 0.54 & 0.499 & 0.60 \\
$P_{Oxygenation}(death)$ & SCM-VRCT & 0.34 & 0.01 & 0.355 & 0.54 & 0.01 & 0.56 \\
$P_{Oxygenation}(death   | age )$ & SCM-VRCT & 0.298 & 0.01 & 0.313 & 0.529 & 0.01 & 0.55 \\
$P_{Oxygenation}(death   | apsiii, sofa )$ & SCM-VRCT & 0.298 & 0.01 & 0.311 & 0.54 & 0.01 & 0.563
\end{tabular}%
}
\end{table}

\subsubsection{Disease severity-specific effect on mortality for oxygenation strategies}\label{sec:severity}
\sloppy
We also estimated the effect of oxygen therapy on mortality conditioned on APSIII and SOFA scores, $P_{oxygenation} (death | APSiii, SOFA)$. The expectation of mortality for conservative and liberal oxygenation for CB-OBS and SCM-VRCT are shown in ~\autoref{fig:p_oxy_death_age_apsiii_sofa}. In CB-OBS, as we discussed in the previous section, the expectation of mortality is higher, 0.54 (95\% CI [0.478, 0.60]), for the conservative oxygenation than the expectation of mortality, 0.362 (95\% CI [0.347, 0.377]) for the liberal oxygenation. The conservative strategy in SCM-VRCT results in 0.539 (95\% CI [0.518, 0.562]), see \autoref{fig:oxy_death_apsiii_sofa_scm_vrct}. However, in the SCM-VRCT, the expectation of the mortality in the liberal strategy is 0.298 (95\% CI [0.0.281, 0.311]) which is significantly lower than the expectation we estimated in the CB-OBS. Therefore, given the severity of illness based on the physiological scores, the liberal approach lowers the expectation of mortality. This finding is consistent with the results of several other studies reported in \autoref{tab:study_outcomes}.
\begin{figure}[t]
\centering
\subfigure[SCM-VRCT]{
\includegraphics[width=.3\textwidth]{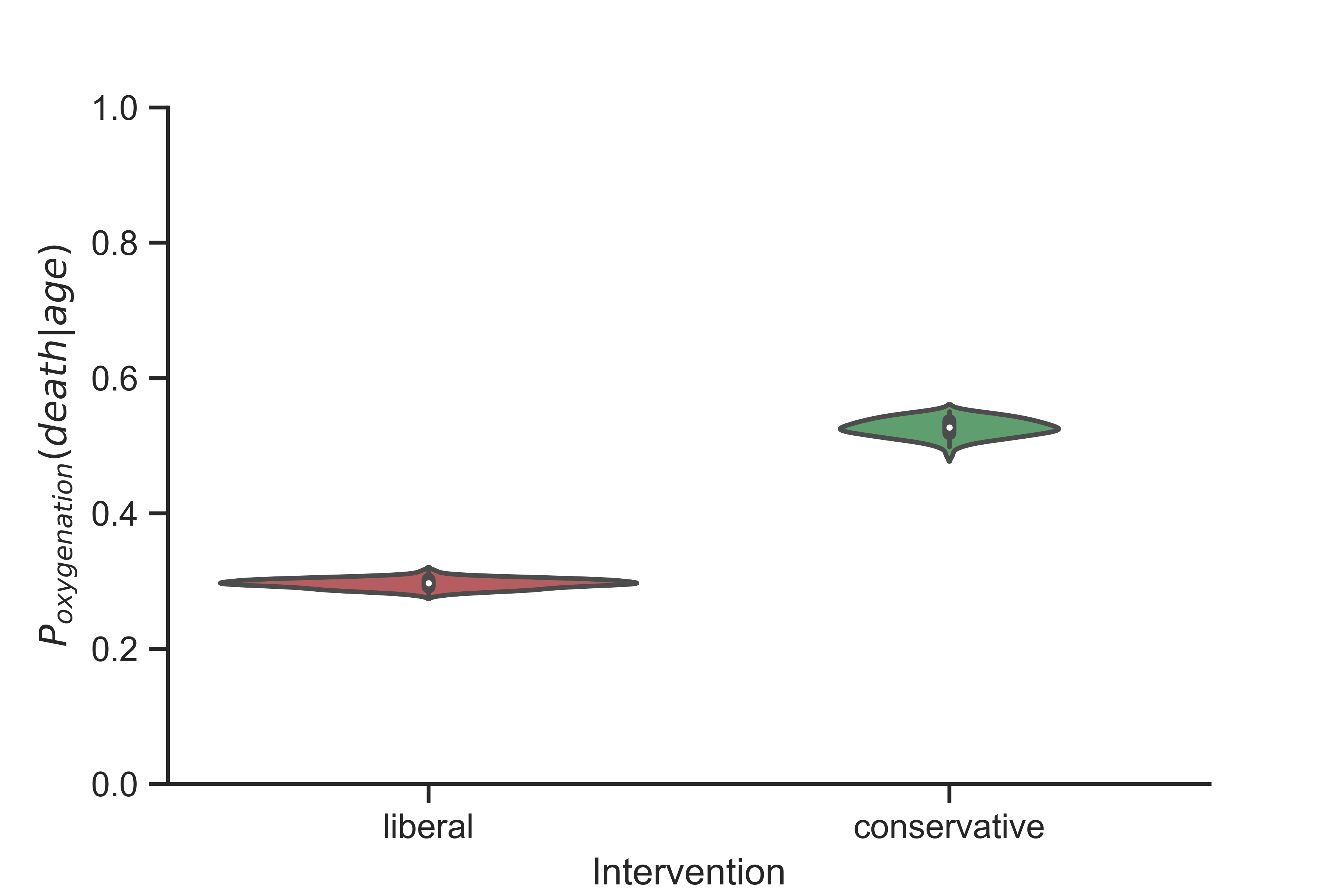}
\label{fig:oxy_death_age_scm_vrct}
}
\subfigure[SCM-VRCT]{
\includegraphics[width=.3\textwidth]{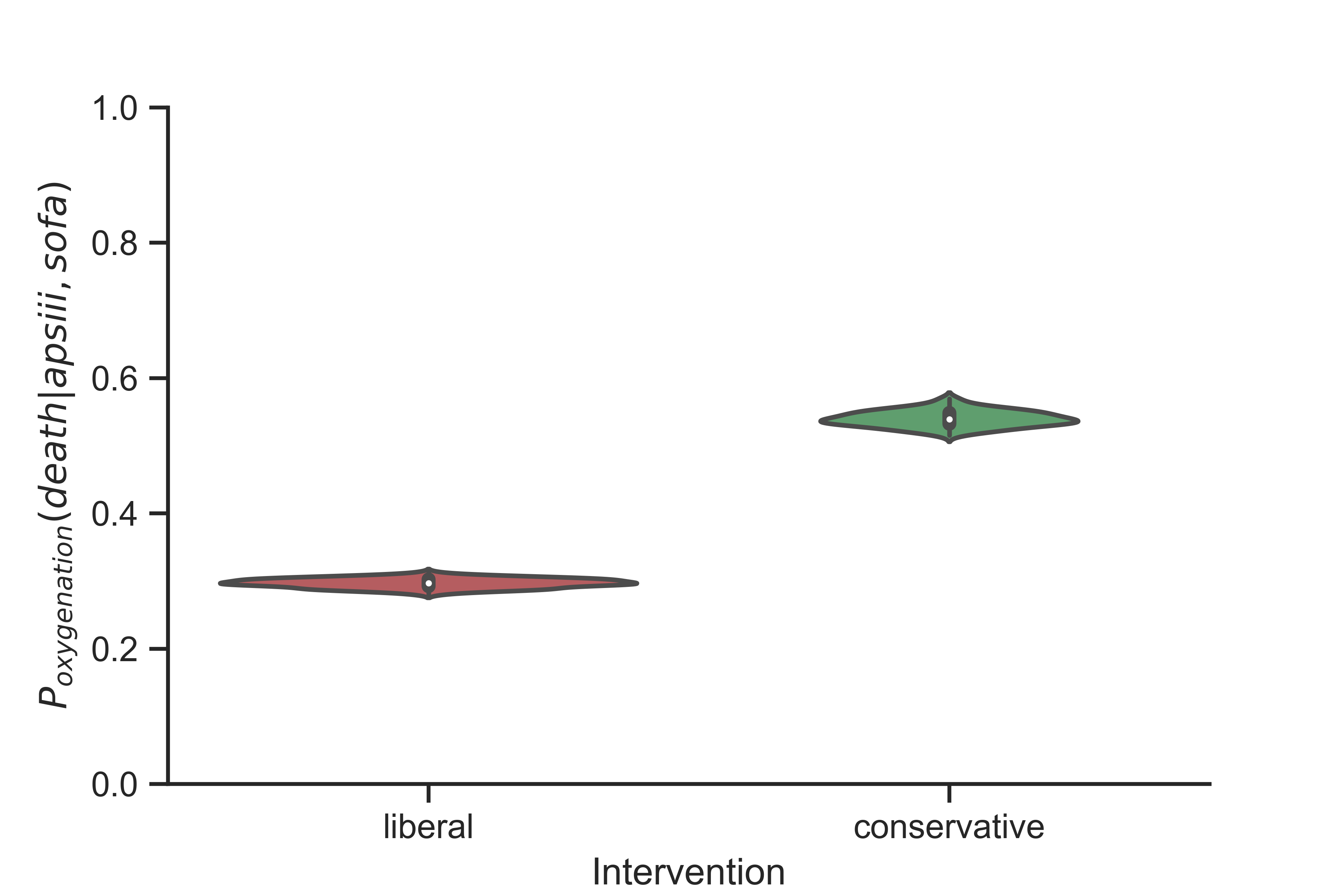}
}
\subfigure[CB-OBS]{
\includegraphics[width=.3\textwidth]{Figures/Observational_distribution.png}
\label{fig:oxy_death_apsiii_sofa_scm_vrct}
}

\caption{Estimation of queries on the causal graph for both observational and experimental dataset. (a) $P_{oxygenation}(death | age)$, (b) $P_{oxygenation}(death | apsiii, sofa)$ and (c) $P(death | oxygenation)$}
\label{fig:p_oxy_death_age_apsiii_sofa}
\end{figure}

\subsubsection{Disease severity-specific effect on mortality conditioned on specific SOFA score}
\label{sec:severity_specific_sofa_score}
\sloppy
Finally, we estimated the effect of oxygenation on mortality conditioned on two different ranges of SOFA score, $P_{oxygenation} (death | SOFA <= 10)$ and $P_{oxygenation} (death | SOFA > 10)$. Although the expectation of mortality decreases as we increase the $SpO_2$, as shown in ~\autoref{fig:oxy_death_apsiii_sofa_scm_vrct}, but the expectations of mortality for conservative and liberal oxygenation are different for two different ranges of SOFA scores. In the lower range of SOFA score $(SOFA <= 10)$, we found that the expectation of mortality is higher, 0.42 for the conservative oxygenation than the expectation of mortality, 0.35 for the liberal oxygenation. However, in the upper range of SOFA score $(SOFA > 10)$, we found that the expectation of mortality is higher, 0.58 for the conservative oxygenation than the expectation of mortality, 0.48 for the liberal oxygenation. One important observation here is that the expected outcome is higher for  patients with a SOFA score greater than 10 compared to patients with a SOFA score smaller than or equal to 10. Therefore, given the severity of illness based on the SOFA score, there is no significant difference in the mortality for different oxygenation targets. This finding also supports the feasibility of a larger study to investigate the effect of conservative oxygenation to treat patients requiring invasive MV in ICUs.

\begin{table}[t]
\centering
\caption{Comparative Analysis of Study Demographics}
\label{tab:study_demographics}
\resizebox{\textwidth}{!}{%
\begin{tabular}{@{}lp{3cm}p{3cm}p{3cm}p{3cm}p{3cm}p{3cm}p{3cm}@{}}
\toprule
Study & \cite{NEJMoa1903297} & \cite{Panwar2016}  & \cite{Girardis2016} & \cite{eastwood2016conservative}  & \cite{SUZUKI2013647} & \multicolumn{2}{c}{MIMIC} \\ \midrule
Type of study & RCT & RCT & RCT & RCT & RCT & CB-OBS & SCM-VRCT \\
\midrule
Included Patients &  Age \textgreater = 18 admitted to ICU on MV with MV time \textless 2 hours & Age   $\geq$ 18 admitted to ICU on MV with expected MV time $\geq$ 24 hours & Age $\geq$ 18 admitted to   ICU with expected ICU LOS $\geq$ 72 hours & Age $\geq$ 18 admitted to ICU on MV post CA & Age   $\geq$ 18 admitted to ICU on MV with expected MV time $\geq$ 48 hours& Age   $\geq$ 18 admitted to ICU on MV with expected MV time $\geq$ 24 hours & Age   $\geq$ 18 admitted to ICU on MV with expected MV time $\geq$ 24 hours \\
\\ \midrule
\multicolumn{8}{l}{SpO2 Minimum   Target}
\\ \midrule
Conservative & 90\% & 88\% & 94\% & 88\% & 90\% & 88\% & 88\% \\
Conventional   (Liberal) & 97\% & 96\% & 97\% & 97\% & 97\% & 96\% & 96\% \\
\\ \midrule
\multicolumn{8}{l}{No of patients}
\\ \midrule
Conservative & 484 & 52 & 216 & 50 & 54 & 250 & 2028 \\
Conventional & 481 & 51 & 218 & 50 & 51 & 3812 & 2034 \\
\\ \midrule
\multicolumn{8}{l}{Age, SD(IQR)} \\
\\ \midrule
Conservative & 58.1(16.2) & 62.4 (14.9) & 63 (51-74) & 65(59-77) & 56 (16) & 60.82 (15.6) & 60.78 (0.55) \\
Conventional & 57.5 (16.1) & 62.4 (17.4) & 65 (52-76) & 67(50-71) & 59 (17) & 62.94 (16.22) & 62.95(0.34) \\
\\ \midrule
\multicolumn{8}{l}{Female Gender (\%)} \\
\\ \midrule
Conservative & 178 (36.8\%) & 20 (38.5\%) & 95 (44.0\%) & 21 (42.0\%) & 22 (40.7\%) & 100 (40.0\%) & 801(10\%) \\
Conventional & 179 (37.2\%) & 18 (35.3\%) & 93 (42.7\%) & 16 (32.0\%) & 13 (25.5\%) & 1615 (42.4\%) & 847 (42\%) \\
\\ \midrule
\multicolumn{8}{l}{Mechanical Ventilation   N, (\%)} \\
\\ \midrule
Conservative & 484(100\%) & 52   (100\%) & 143   (66.2\%) & 50   (100\%) & 54 (100\%)  & 250 (100\%) & 2000(100\%) \\
Conventional & 481(100\%) & 51   (100\%) & 148   (67.9\%) & 50   (100\%) & 51 (100\%) & 3812 (100\%) & 2000(100\%) \\
\\ \midrule
\multicolumn{8}{l}{APACHE III score   (IQR)} \\
\\ \midrule
Conservative & NR (APACHE II) & 79.5   (61-92.5) & NR & 121   (105-142) & 62 (49-92)  & NR (APS III) & NR (APS III) \\
Conventional & NR (APACHE II) & 70   (50-84) & NR \footnotemark & 125   (107-141) & 68   (42-94) & NR (APS III) & NR (APS III) \\ \bottomrule
\end{tabular}%
}
\end{table}
\footnotetext{Not reported}

\begin{table}[t]
\centering
\caption{Comparative Analysis of Study Outcome Measures}
\label{tab:study_outcomes}
\resizebox{\textwidth}{!}{%
\begin{tabular}{@{}llllllcc@{}}
\toprule
Study & \cite{NEJMoa1903297} & \cite{Panwar2016}  & \cite{Girardis2016} & \cite{eastwood2016conservative}  & \cite{SUZUKI2013647} & \multicolumn{2}{c}{MIMIC} \\
 &  &  &  &  &  & CB-OBS & SCM-VRCT \\ \midrule
\multicolumn{8}{l}{ICU Mortality, N   (\%)} \\ \midrule
Conservative & NR & 13   (25\%) & 25   (11.6\%) & 15   (30\%) & NR & NR &  NR\\
Conventional & NR & 12   (24\%) & 44   (20.2\%) & 16   (32\%) & NR & NR &  NR\\ \midrule
\multicolumn{8}{l}{28-day Mortality,   N (\%)} \\ \midrule
Conservative & NR & NR & NR & 27 (54\%) & 9   (16.7\%) & 122 (48.8\%) &  NR \\
Conventional & NR & NR & NR & 26   (52\%) & 16   (31.4\%) & 949 (24.89\%) & NR \\ \midrule
\multicolumn{8}{l}{90-day   Mortality, N (\%)} \\ \midrule
Conservative & 166   (34.7\%) & 21   (40\%) & NR & NR & NR & 135   (54\%) &  1078(53.9\%)\\
Conventional & 156   (32.5\%) & 19   (37\%) & NR & NR & NR & 1380   (36.2\%) & 682(34.1\%) \\ \midrule
\multicolumn{8}{l}{180-day   Mortality, N (\%)} \\ \midrule
Conservative & 170   (35.7\%) & NR & NR & NR & NR & NR & NR \\
Conventional & 164   (34.5\%) & NR & NR & NR & NR & NR &  NR\\ \midrule
\multicolumn{8}{l}{Hospital   Mortality, N (\%)} \\ \midrule
Conservative & NR & NR & 52   (24.2\%) & 28   (56\%) & NR & NR & NR \\ 
Conventional & NR & NR & 74   (33.9\%) & 27   (54\%) & NR & NR &  NR\\ \midrule
\multicolumn{8}{l}{ICU   LOS, Days (IQR)} \\ \midrule
Conservative & NR & 9   (5-13) & 6 (4-10) & 4   (2-7) & NR & NR & NR \\
Conventional & NR & 7   (4-12) & 6   (4-11) & 5   (4-9) & NR & NR & NR \\ \midrule
\multicolumn{8}{l}{Hospital LOS,   Days (IQR)} \\ \midrule
Conservative & NR & 20   (10-25) & 21   (13-38) & 9   (3-17) & NR & NR & NR \\
Conventional & NR & 16   (7-30) & 21   (12-34) & 9   (4-24) & NR & NR &  NR\\ \midrule
\multicolumn{8}{l}{New   infections, N (\%) - Search database for Bacteremia (Girardis looked at this),   Nosocomial} \\ \midrule
Conservative & NR & NR & 39   (18.1) & NR & 31   (57.4) & NR & NR \\
Conventional & NR & NR & 50   (22.9) & NR & 28   (54.9) & NR & NR \\ \midrule 
\multicolumn{8}{l}{New   non-respiratory organ failure, N (\%) – 3 organs, if possible: Creatine   baseline vs. after 2 weeks or 30 days; Troponin baseline vs. after 1 week;   Lactate baseline vs. 1 week} \\ \midrule
Conservative & NR & NR & 38   (17.6) & NR & 16   (29.6) & NR & NR \\
Conventional & NR & NR & 58   (26.6) & NR & 22   (43.1) & NR & NR\\ \bottomrule

\end{tabular}%
}
\end{table}
\section{Discussion}
    The study was conducted and validated on the MIMIC III ICU database, which contains data routinely collected from adult patients in the US. We included adult patients fulfilling the requirement for the study based on the OT-RCT study. After the exclusion of ineligible cases, we included 4,062 patients for the study. The time-series data of oxygenation parameters and ventilator settings were then used to compute the mean mortality over the entire period. Patient demographics and clinical characteristics are shown in ~\autoref{tab:summary-statistics-table}. We extracted a set of 26 variables, 24 variables following the OT-RCT, including demographics, oxygenation parameters and ventilator settings, and 2 variables, one for oxygenation (treatment) and another for mortality (outcome).

We assessed the feasibility of a conservative oxygenation strategy (target $SpO_{2}$ 88-95\%) compared to a liberal oxygenation strategy (target $SpO_{2} \geq$ 96\%) during MV for adult ICU patients. The study protocol followed the protocol in the OT-RCT. To achieve treatment feasibility, this study implemented a clear separation in the mean $SpO_{2}, SaO_{2}, PaO_{2}, $ and $FiO_{2}$ values between the groups. 
The point estimate for 90-day mortality was lower with the liberal oxygenation strategy. This is not consistent with the findings of the OT-RCT. However, in the SCM-VRCT, the expectation of the mortality in the liberal strategy is significantly lower than the expectation we estimated in the CB-OBS. We conclude that given the severity of illness based on the physiological scores, the liberal approach lowers the expectation of mortality. This finding is consistent with the results of several other studies reported in \autoref{tab:study_outcomes}.. The result of the estimation are presented in ~\autoref{fig:query_oxy_spo2_death} (a). Another important observation is that the expected outcome is higher for  patients with a SOFA score greater than 10 compared to patients with a SOFA score smaller than or equal to 10. That is, given the severity of illness based on the SOFA score, there is no significant difference in the mortality for different oxygenation targets. This finding also supports the feasibility of a larger study to investigate the effect of conservative oxygenation to treat patients requiring invasive MV in ICUs.

Several methodological limitations that we have in the current model are currently being addressed by us for future work with the theoretical foundations. These include time-varying causal model with SCM, survival analysis with SCM, and incorporating unobserved confounders in the model, among others.

It is important to note that although we demonstrated this method on oxygenation therapy for MV patients in the ICU, the technique can be applied generally for a broad set of clinical questions.  We are currently applying this method for a range of questions including estimating the impact of antipsychotic drugs on delirium in the ICU and estimating the impact of timing of antibiotics for patients with chronic pulmonary obstructive disease. 

\section{Acknowledgement}
    
This study was partially supported by the Regenstrief Center for Healthcare Engineering at Purdue University. We would like to express our sincere gratitude to Professor Elias Bareinboim for his insights on the methodological framework. Part of the analysis (Appendix \autoref{apx:derivation}) was achieved with the \href{causalfusion}{causalfusion.net} software developed by Dr. Bareinboim. 
\newpage
\bibliography{main.bib}

\newpage
\appendix
    \begin{appendices}
\makeatletter
\newenvironment{shiftedflalign}{
	\start@align\tw@\st@rredfalse\m@ne
	\qquad\qquad
}{
	\endalign
}
\newenvironment{shiftedflalign*}{
	\start@align\tw@\st@rredtrue\m@ne
	\qquad\qquad
}{
	\endalign
}
\makeatother

\noindent
\bf{Appendices: Derivation of}
{\bf Identifiability of the causal effect of $oxygenation$ on $death$ given $\left\{ apsiii, sofa \right\}$}
\renewcommand{\thesubsection}{\Alph{subsection}}


\section{Formulation}
\label{apx:query}

\ \\
To keep the derivation steps brief and concise, we define the following sets of variables in ~\autoref{tab:variableSet} that are combinations of the 24 variables we have in our model. The large number of variables, 24 variables, we have in the model makes the subscripts in the equations long and clumsy. We use these sets in the derivation section.

\begin{table}[hp]
\centering
\caption{The Variable Sets Defined for the Derivation}
\label{tab:variableSet}
\resizebox{0.8\textwidth}{!}{%
\begin{tabular}{ll}
\textbf{Set} & \textbf{Variables} \\
$U=$ & $\left\{\substack{age,apsiii,ards,bmi,copd,death,fio2,gender,hemoglobin,ischemicHd,lactate,\\medical,minVentVol,oxygenation,paco2,pao2,peakAirPressure,\\peep,ph,sao2,smoker,sofa,spo2,surgery,trauma,vt}\right\}$\\
$A=$ & $\left\{\substack{oxygenation,apsiii,sofa,smoker,copd,spo2,fio2,hemoglobin,peep,peakAirPressure}\right\}$\\
$B=$ & $\left\{\substack{smoker,copd,spo2,fio2,hemoglobin,peep,peakAirPressure}\right\}$\\
$C=$ & $\left\{\substack{age,gender,bmi,trauma,smoker,copd,ards,spo2,fio2,paco2,ph,hemoglobin,peep,\\peakAirPressure,lactate}\right\}$\\
$D=$ & $\left\{\substack{age,gender,bmi,trauma,apsiii,sofa,smoker,copd,ards,death,spo2,fio2,\\paco2,ph,hemoglobin,peep,peakAirPressure,lactate}\right\}$\\
$E=$ & $\left\{\substack{bmi,ards,age,smoker,ph,spo2,copd,paco2,gender,hemoglobin,apsiii,\\peakAirPressure,peep,trauma,fio2,lactate}\right\}$\\
$F=$ & $\left\{\substack{bmi,ards,age,smoker,ph,spo2,copd,paco2,gender,hemoglobin,apsiii,peep,\\peakAirPressure,trauma,fio2,lactate,minVentVol,sofa,pao2,vt,oxygenation}\right\}$\\
$G=$ & $\left\{\substack{lactate,fio2,trauma,peep,peakAirPressure,apsiii,hemoglobin,gender,paco2,\\copd,spo2,ph,smoker,age,ards,bmi}\right\}$\\
$H=$ & $\left\{\substack{oxygenation,vt,pao2,sofa,minVentVol,lactate,fio2,trauma,peep,age,ards,\\peakAirPressure,apsiii,hemoglobin,gender,paco2,copd,spo2,ph,smoker,bmi}\right\}$\\
$I=$ & $\left\{\substack{fio2,trauma,peep,peakAirPressure,apsiii,hemoglobin,gender,paco2,\\copd,spo2,ph,smoker,age,ards,bmi}\right\}$\\
$J=$ & $\left\{\substack{gender,trauma,surgery,smoker,ischemicHd,hemoglobin,ph,minVentVol,ards,bmi,\\medical,paco2,apsiii,peakAirPressure,lactate,sofa,sao2,pao2,vt,oxygenation,death}\right\}$\\
$K=$ & $\left\{\substack{age,copd,peep,fio2,trauma,surgery,smoker,ischemicHd,ph,minVentVol,ards,bmi,\\paco2,apsiii,peakAirPressure,lactate,medical,sofa,sao2,pao2,vt,oxygenation,death}\right\}$\\
$L=$ & $\left\{\substack{ph,spo2,gender,hemoglobin,copd,fio2,surgery,smoker,ischemicHd,\\minVentVol,ards,bmi,paco2,apsiii,peakAirPressure,lactate,\\medical,sofa,sao2,pao2,vt,oxygenation,death}\right\}$\\
$M=$ & $\left\{\substack{peep,trauma,gender,hemoglobin,fio2,surgery,ischemicHd,minVentVol,\\apsiii,peakAirPressure,lactate,medical,sofa,sao2,\\pao2,vt,oxygenation,death}\right\}$\\
$N=$ & $\left\{\substack{gender,peep,trauma,surgery,smoker,ischemicHd,fio2,hemoglobin,ph,\\minVentVol,ards,bmi,paco2,apsiii,peakAirPressure,lactate,medical,\\sofa,sao2,pao2,vt,oxygenation,death}\right\}$\\
$O=$ & $\left\{\substack{bmi,ards,age,smoker,ph,spo2,copd,paco2,gender,hemoglobin,apsiii,\\peakAirPressure,peep,trauma,fio2}\right\}$\\
$P=$ & $\left\{\substack{peep,trauma,fio2,surgery,ischemicHd,minVentVol,peakAirPressure,\\lactate,medical,sofa,sao2,pao2,vt,oxygenation,death}\right\}$\\
$Q=$ & $\left\{\substack{bmi,ards,age,smoker,ph,spo2,copd,paco2,gender,hemoglobin,apsiii}\right\}$\\
$R=$ & $\left\{\substack{peep,trauma,fio2,surgery,ischemicHd,minVentVol,lactate,medical,\\sofa,sao2,pao2,vt,oxygenation,death}\right\}$\\
$S=$ & $\left\{\substack{apsiii,hemoglobin,gender,paco2,copd,spo2,ph,smoker,age,ards,bmi}\right\}$
\end{tabular}%
}
\end{table}

The qualitative knowledge of causal relationships in the domain is represented by a causal model shown in 
The treatment variable is $oxygenation$ and the outcome variable is $death$.
We show that the causal effect $do(oxygenation = oxygenation)$ on $death$ given $\left\{ apsiii, sofa \right\}$, written as \\$P_{oxygenation}\left(death \middle| apsiii,sofa\right)$, is identifiable from a distribution over the observed variables $P\left(U\right)$.

\section{Derivation}
\label{apx:derivation}
\begin{theorem}
	The causal effect of $oxygenation$ on $death$ given $\left\{ apsiii, sofa \right\}$ is identifiable from $P_{}\left(U\right)$ and is given by the formula
	\[ P_{oxygenation}\left(death \middle| apsiii,sofa\right) = \sum_{B}{P\left(death \middle| A\right)P\left(B\right)} \]
\end{theorem}

\begin{proof}
\begin{subequations}
\begin{align}
	& P_{oxygenation}\left(death \middle| apsiii,sofa\right) \label{eq1} \\
	&= \frac{P_{oxygenation}\left(death,apsiii,sofa\right)}{P_{oxygenation}\left(apsiii,sofa\right)} \label{eq2} \\
	&=  P_{oxygenation,pao2,vt,minVentVol}\left(death,apsiii,sofa\right) \label{eq3} \\
	&= \sum_{C}{P_{surgery,ischemicHd,minVentVol,medical,sao2,pao2,vt,oxygenation}\left(D\right)} \label{eq4} \\
\end{align}
\end{subequations}
	Eq. (\ref{eq2}) follows from the definition of conditional probability and Eq. (\ref{eq3})  from the third rule of do-calculus with the independence\\  $\left(pao2,vt,minVentVol \perp death,apsiii,sofa | oxygenation\right)_{G_{\overline{oxygenation,pao2,vt,minVentVol}}}$ (refer to Fig. \ref{figure2}).
	Eq. (\ref{eq4}) follows from summing over $\left\{ C \right\}$ and Eq. (\ref{eq5})  from C-component factorization. \\\\
	\textbf{Task 1: Compute $P_{U - \{age\}}\left(age\right)$} 
	\begin{shiftedflalign}
		& P_{U - \{age\}}\left(age\right) & \label{eq6} \\
		&= P_{}\left(age\right) \label{eq7}
	\end{shiftedflalign}
	Eq. (\ref{eq7}) follows from the third rule of do-calculus with the independence\\ $\left(U - \{age\} \perp age\right)_{G_{\overline{U - \{age\}}}}$ (refer to Fig. \ref{figure3}). \\\\
	\textbf{Task 2: Compute $P_{U - \{peep\}}\left(peep\right)$} 
	\begin{shiftedflalign}
		& P_{U - \{peep\}}\left(peep\right) & \label{eq8} \\
		&= {{P_{age}\left(peep\right)}} \label{eq9} \\
		&= {{P_{}\left(peep \middle| age\right)}} \label{eq10}
	\end{shiftedflalign}
	Eq. (\ref{eq9}) follows from the third rule of do-calculus with the independence \\$\left(U - \{age, peep\} \perp peep | age\right)_{G_{\overline{U - \{peep\}}}}$ (refer to Fig. \ref{figure4}).
	Eq. (\ref{eq10}) follows from the second rule of do-calculus with the independence $\left(age \perp peep\right)_{G_{\overline{}\underline{age}}}$ (refer to Fig. \ref{figure5}). \\\\
	\textbf{Task 3: Compute $P_{U - \{gender\}}\left(gender\right)$} 
	\begin{shiftedflalign}
		& P_{U - \{gender\}}\left(gender\right) & \label{eq11} \\
		&= P_{}\left(gender\right) \label{eq12}
	\end{shiftedflalign}
	Eq. (\ref{eq12}) follows from the third rule of do-calculus with the independence \\$\left(U - \{gender\} \perp gender\right)_{G_{\overline{U - \{gender\}}}}$ (refer to Fig. \ref{figure3}). \\\\
	\textbf{Task 4: Compute $P_{U - \{SpO_{2}\}}\left(spo2\right)$} 
	\begin{shiftedflalign}
		& P_{U - \{SpO_{2}\}}\left(spo2\right) & \label{eq13} \\
		&= P_{}\left(spo2\right) \label{eq14}
	\end{shiftedflalign}
	Eq. (\ref{eq14}) follows from the third rule of do-calculus with the independence\\ $\left(U - \{SpO_{2}\} \perp spo2\right)_{G_{\overline{U - \{SpO_{2}\}}}}$ (refer to Fig. \ref{figure3}). \\\\
	\textbf{Task 5: Compute $P_{U - \{copd\}}\left(copd\right)$} 
	\begin{shiftedflalign}
		& P_{U - \{copd\}}\left(copd\right) & \label{eq15} \\
		&= {{P_{spo2,age}\left(copd\right)}} \label{eq16} \\
		&= {{P_{}\left(copd \middle| age,spo2\right)}} \label{eq17}
	\end{shiftedflalign}
	Eq. (\ref{eq16}) follows from the third rule of do-calculus with the independence \\$\left(N \perp copd | spo2,age\right)_{G_{\overline{U - \{copd\}}}}$ (refer to Fig. \ref{figure6}).
	Eq. (\ref{eq17}) follows from the second rule of do-calculus with the independence \\ $\left(age,spo2 \perp copd\right)_{G_{\overline{}\underline{age,spo2}}}$ (refer to Fig. \ref{figure7}). \\\\
	\textbf{Task 6: Compute $P_{U - \{fio2\}}\left(fio2\right)$} 
	\begin{shiftedflalign}
		& P_{U - \{fio2\}}\left(fio2\right) & \label{eq18} \\
		&= {{P_{spo2,age,copd,peep}\left(fio2\right)}} \label{eq19} \\
		&= {{P_{}\left(fio2 \middle| peep,copd,age,spo2\right)}} \label{eq20}
	\end{shiftedflalign}
	Eq. (\ref{eq19}) follows from the third rule of do-calculus with the independence \\$\left(J \perp fio2 | spo2,age,copd,peep\right)_{G_{\overline{U - \{fio2\}}}}$ (refer to Fig. \ref{figure8}).
	Eq. (\ref{eq20}) follows from the second rule of do-calculus with the independence $\left(peep,copd,age,spo2 \perp fio2\right)_{G_{\overline{}\underline{peep,copd,age,spo2}}}$ (refer to Fig. \ref{figure9}). \\\\
	\textbf{Task 7: Compute $P_{U - \{hemoglobin\}}\left(hemoglobin\right)$} 
	\begin{shiftedflalign}
		& P_{U - \{hemoglobin\}}\left(hemoglobin\right) & \label{eq21} \\
		&= {{P_{spo2,gender}\left(hemoglobin\right)}} \label{eq22} \\
		&= {{P_{}\left(hemoglobin \middle| gender,spo2\right)}} \label{eq23}
	\end{shiftedflalign}
	Eq. (\ref{eq22}) follows from the third rule of do-calculus with the independence \\$\left(K \perp hemoglobin | spo2,gender\right)_{G_{\overline{U - \{hemoglobin\}}}}$ (refer to Fig. \ref{figure10}).
	Eq. (\ref{eq23}) follows from the second rule of do-calculus with the independence $\left(gender,spo2 \perp hemoglobin\right)_{G_{\overline{}\underline{gender,spo2}}}$ (refer to Fig. \ref{figure11}). \\\\
	\textbf{Task 8: Compute $P_{U - \{ph\}}\left(ph\right)$} 
	\begin{shiftedflalign}
		& P_{U - \{ph\}}\left(ph\right) & \label{eq24} \\
		&= P_{}\left(ph\right) \label{eq25}
	\end{shiftedflalign}
	Eq. (\ref{eq25}) follows from the third rule of do-calculus with the independence \\$\left(U - \{ph\} \perp ph\right)_{G_{\overline{U - \{ph\}}}}$ (refer to Fig. \ref{figure3}). \\\\
	\textbf{Task 9: Compute $P_{U - \{trauma\}}\left(trauma\right)$} 
	\begin{shiftedflalign}
		& P_{U - \{trauma\}}\left(trauma\right) & \label{eq26} \\
		&= {{P_{age,peep}\left(trauma\right)}} \label{eq27} \\
		&= {{P_{}\left(trauma \middle| peep,age\right)}} \label{eq28}
	\end{shiftedflalign}
	Eq. (\ref{eq27}) follows from the third rule of do-calculus with the independence \\$\left(L \perp trauma | age,peep\right)_{G_{\overline{U - \{trauma\}}}}$ (refer to Fig. \ref{figure12}).
	Eq. (\ref{eq28}) follows from the second rule of do-calculus with the independence $\left(peep,age \perp trauma\right)_{G_{\overline{}\underline{peep,age}}}$ (refer to Fig. \ref{figure13}). \\\\
	\textbf{Task 10: Compute $P_{U - \{smoker\}}\left(smoker\right)$} 
	\begin{shiftedflalign}
		& P_{U - \{smoker\}}\left(smoker\right) & \label{eq29} \\
		&= {{P_{age}\left(smoker\right)}} \label{eq30} \\
		&= {{P_{}\left(smoker \middle| age\right)}} \label{eq31}
	\end{shiftedflalign}
	Eq. (\ref{eq30}) follows from the third rule of do-calculus with the independence \\$\left(U - \{age, smoker\} \perp smoker | age\right)_{G_{\overline{U - \{smoker\}}}}$ (refer to Fig. \ref{figure14}).
	Eq. (\ref{eq31}) follows from the second rule of do-calculus with the independence $\left(age \perp smoker\right)_{G_{\overline{}\underline{age}}}$ (refer to Fig. \ref{figure5}). \\\\
	\textbf{Task 11: Compute $P_{U - \{ards\}}\left(ards\right)$} 
	\begin{shiftedflalign}
		& P_{U - \{ards\}}\left(ards\right) & \label{eq32} \\
		&= P_{}\left(ards\right) \label{eq33}
	\end{shiftedflalign}
	Eq. (\ref{eq33}) follows from the third rule of do-calculus with the independence\\ $\left(U - \{ards\} \perp ards\right)_{G_{\overline{U - \{ards\}}}}$ (refer to Fig. \ref{figure3}). \\\\
	\textbf{Task 12: Compute $P_{U - \{bmi\}}\left(bmi\right)$} 
	\begin{shiftedflalign}
		& P_{U - \{bmi\}}\left(bmi\right) & \label{eq34} \\
		&= P_{}\left(bmi\right) \label{eq35}
	\end{shiftedflalign}
	Eq. (\ref{eq35}) follows from the third rule of do-calculus with the independence\\ $\left(U - \{bmi\} \perp bmi\right)_{G_{\overline{U - \{bmi\}}}}$ (refer to Fig. \ref{figure3}). \\\\
	\textbf{Task 13: Compute $P_{U - \{paco2\}}\left(paco2\right)$} 
	\begin{shiftedflalign}
		& P_{U - \{paco2\}}\left(paco2\right) & \label{eq36} \\
		&= {{P_{bmi,ards,age,smoker,ph,spo2,copd}\left(paco2\right)}} \label{eq37} \\
		&= {{P_{}\left(paco2 \middle| copd,spo2,ph,smoker,age,ards,bmi\right)}} \label{eq38}
	\end{shiftedflalign}
	Eq. (\ref{eq37}) follows from the third rule of do-calculus with the independence \\$\left(M \perp paco2 | bmi,ards,age,smoker,ph,spo2,copd\right)_{G_{\overline{U - \{paco2\}}}}$ (refer to Fig. \ref{figure15}).
	Eq. (\ref{eq38}) follows from the second rule of do-calculus with the independence \\ $\left(copd,spo2,ph,smoker,age,ards,bmi \perp paco2\right)_{G_{\overline{}\underline{copd,spo2,ph,smoker,age,ards,bmi}}}$ (refer to Fig. \ref{figure16}). \\\\
	\textbf{Task 14: Compute $P_{U - \{apsiii\}}\left(apsiii\right)$} 
	\begin{shiftedflalign}
		& P_{U - \{apsiii\}}\left(apsiii\right) & \label{eq39} \\
		&= {{P_{Q - \{apsiii\}}\left(apsiii\right)}} \label{eq40} \\
		&= {{P_{}\left(apsiii \middle| S - \{apsiii\}\right)}} \label{eq41}
	\end{shiftedflalign}
	Eq. (\ref{eq40}) follows from the third rule of do-calculus with the independence \\$\left(P \perp apsiii | Q - \{apsiii\}\right)_{G_{\overline{U - \{apsiii\}}}}$ (refer to Fig. \ref{figure17}).
	Eq. (\ref{eq41}) follows from the second rule of do-calculus with the independence $\left(S - \{apsiii\} \perp apsiii\right)_{G_{\overline{}\underline{S - \{apsiii\}}}}$ (refer to Fig. \ref{figure18}). \\\\
	\textbf{Task 15: Compute $P_{U - \{peakAirPressure\}}\left(peakAirPressure\right)$} 
	\begin{shiftedflalign}
		& P_{U - \{peakAirPressure\}}\left(peakAirPressure\right) & \label{eq42} \\
		&= {{P_{Q}\left(peakAirPressure\right)}} \label{eq43} \\
		&= {{P_{}\left(peakAirPressure \middle| S\right)}} \label{eq44}
	\end{shiftedflalign}
	Eq. (\ref{eq43}) follows from the third rule of do-calculus with the independence \\ $\left(R \perp peakAirPressure | Q\right)_{G_{\overline{U - \{peakAirPressure\}}}}$ (refer to Fig. \ref{figure19}).
	Eq. (\ref{eq44}) follows from the second rule of do-calculus with the independence $\left(S \perp peakAirPressure\right)_{G_{\overline{}\underline{S}}}$ (refer to Fig. \ref{figure20}). \\\\
	\textbf{Task 16: Compute $P_{U - \{lactate\}}\left(lactate\right)$} 
	\begin{shiftedflalign}
		& P_{U - \{lactate\}}\left(lactate\right) & \label{eq45} \\
		&= {{P_{O}\left(lactate\right)}} \label{eq46} \\
		&= {{P_{}\left(lactate \middle| I\right)}} \label{eq47}
	\end{shiftedflalign}
	Eq. (\ref{eq46}) follows from the third rule of do-calculus with the independence \\ $\left(\substack{surgery,ischemicHd,minVentVol,medical,sofa,\\sao2,pao2,vt,oxygenation,death \perp lactate | O}\right)_{G_{\overline{U - \{lactate\}}}}$ (refer to Fig. \ref{figure21}).
	Eq. (\ref{eq47}) follows from the second rule of do-calculus with the independence $\left(I \perp lactate\right)_{G_{\overline{}\underline{I}}}$ (refer to Fig. \ref{figure22}). \\\\
	\textbf{Task 17: Compute $P_{U - \{death\}}\left(death\right)$} 
	\begin{shiftedflalign}
		& P_{U - \{death\}}\left(death\right) & \label{eq48} \\
		&= {{P_{F}\left(death\right)}} \label{eq49} \\
		&= {{P_{}\left(death \middle| H\right)}} \label{eq50}
	\end{shiftedflalign}
	Eq. (\ref{eq49}) follows from the third rule of do-calculus with the independence\\ $\left(surgery,ischemicHd,medical,sao2 \perp death | F\right)_{G_{\overline{U - \{death\}}}}$ (refer to Fig. \ref{figure23}).
	Eq. (\ref{eq50}) follows from the second rule of do-calculus with the independence $\left(H \perp death\right)_{G_{\overline{}\underline{H}}}$ (refer to Fig. \ref{figure24}). \\\\
	\textbf{Task 18: Compute $P_{U - \{sofa\}}\left(sofa\right)$} 
	\begin{shiftedflalign}
		& P_{U - \{sofa\}}\left(sofa\right) & \label{eq51} \\
		&= {{P_{E}\left(sofa\right)}} \label{eq52} \\
		&= {{P_{}\left(sofa \middle| G\right)}} \label{eq53}
	\end{shiftedflalign}
	Eq. (\ref{eq52}) follows from the third rule of do-calculus with the independence \\ $\left(\substack{minVentVol,pao2,vt,oxygenation,death,surgery,\\ischemicHd,medical,sao2 \perp sofa | E}\right)_{G_{\overline{U - \{sofa\}}}}$ (refer to Fig. \ref{figure25}).
	Eq. (\ref{eq53}) follows from the second rule of do-calculus with the independence $\left(G \perp sofa\right)_{G_{\overline{}\underline{G}}}$ (refer to Fig. \ref{figure26}). \\\\
	Substituting Eq. (\ref{eq7}), Eq. (\ref{eq10}), Eq. (\ref{eq12}), Eq. (\ref{eq14}), Eq. (\ref{eq17}), Eq. (\ref{eq20}), Eq. (\ref{eq23}), Eq. (\ref{eq25}), Eq. (\ref{eq28}), Eq. (\ref{eq31}), Eq. (\ref{eq33}), Eq. (\ref{eq35}), Eq. (\ref{eq38}), Eq. (\ref{eq41}), Eq. (\ref{eq44}), Eq. (\ref{eq47}), Eq. (\ref{eq50}), and Eq. (\ref{eq53}) back into Eq. (\ref{eq5}), we get
	\begin{flalign}
		P_{oxygenation}\left(death \middle| apsiii,sofa\right) = \sum_{B}{P\left(death \middle| A\right)P\left(B\right)} \label{eq54}
	\end{flalign}
\end{proof}

\section{Figures}

The subgraphs used in the derivation of the causal effect of $oxygenation$ on $death$ given $\left\{ apsiii, sofa \right\}$ are as follows:

\begin{figure}[h]
	\centering
	\includegraphics[width=0.6\textwidth]{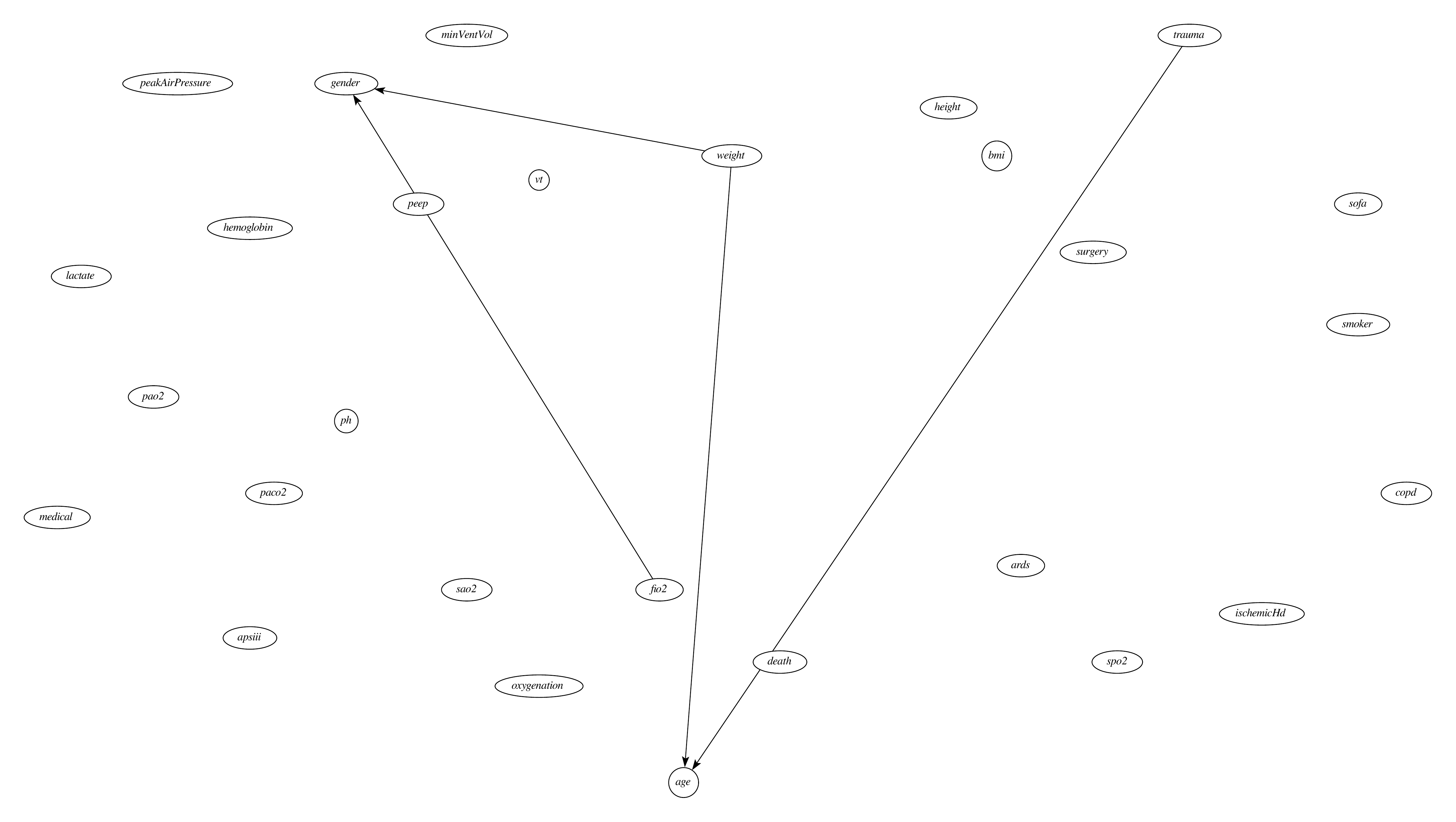}
	\caption{Causal Graph $G_{\overline{oxygenation,pao2,vt,minVentVol}}$.}
	\label{figure2}
\end{figure}

\begin{figure}[h]
	\centering
	\includegraphics[width=0.6\textwidth]{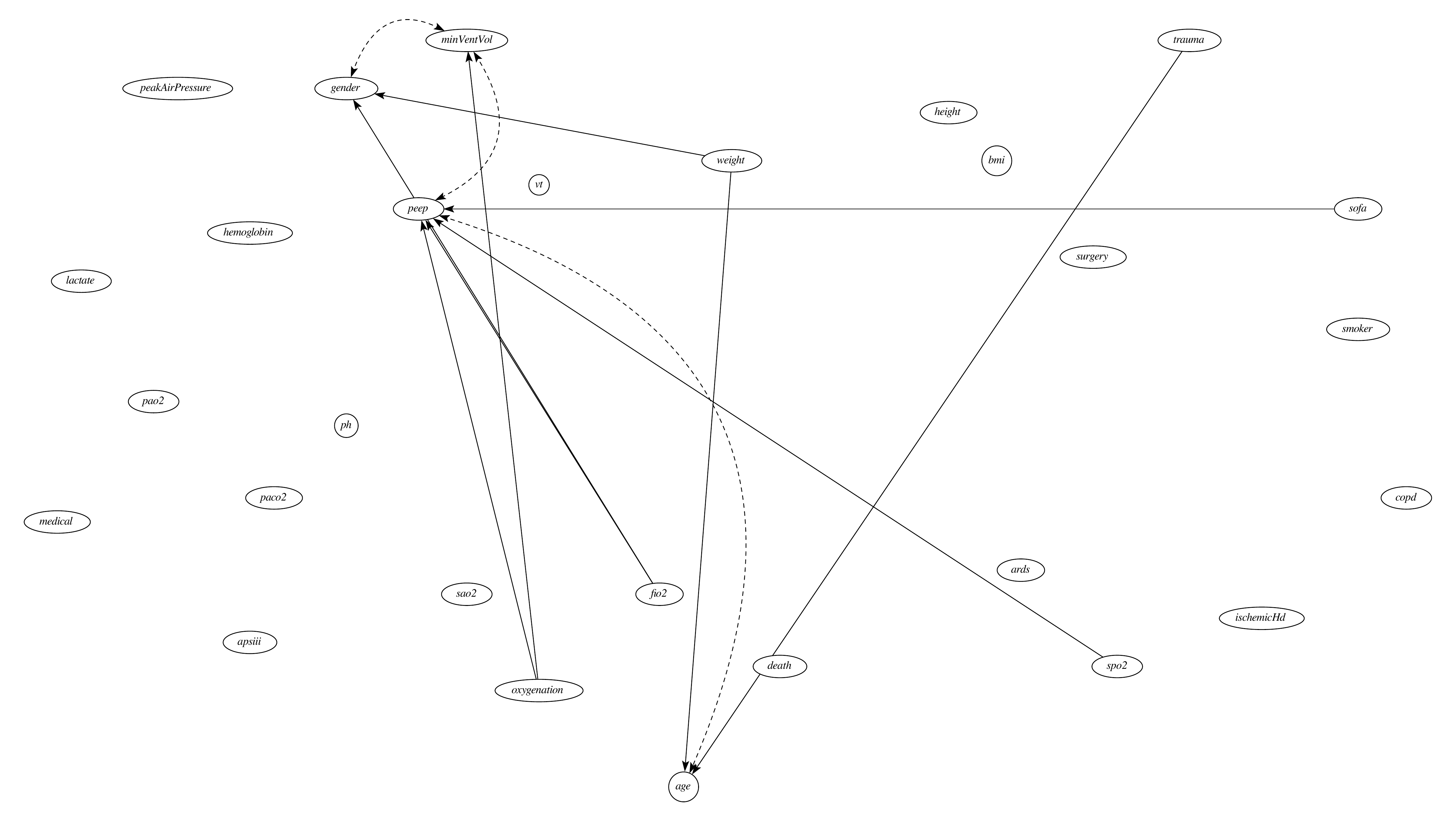}
	\caption{Causal Graph $G_{\overline{U - \{age\}}}, G_{\overline{U - \{gender\}}}, G_{\overline{U - {SpO_{2}}}}, G_{\overline{U - \{ph\}}}, G_{\overline{U - \{ards\}}}, G_{\overline{U - \{bmi\}}}$.}
	\label{figure3}
\end{figure}

\begin{figure}[h]
	\centering
	\includegraphics[width=0.6\textwidth]{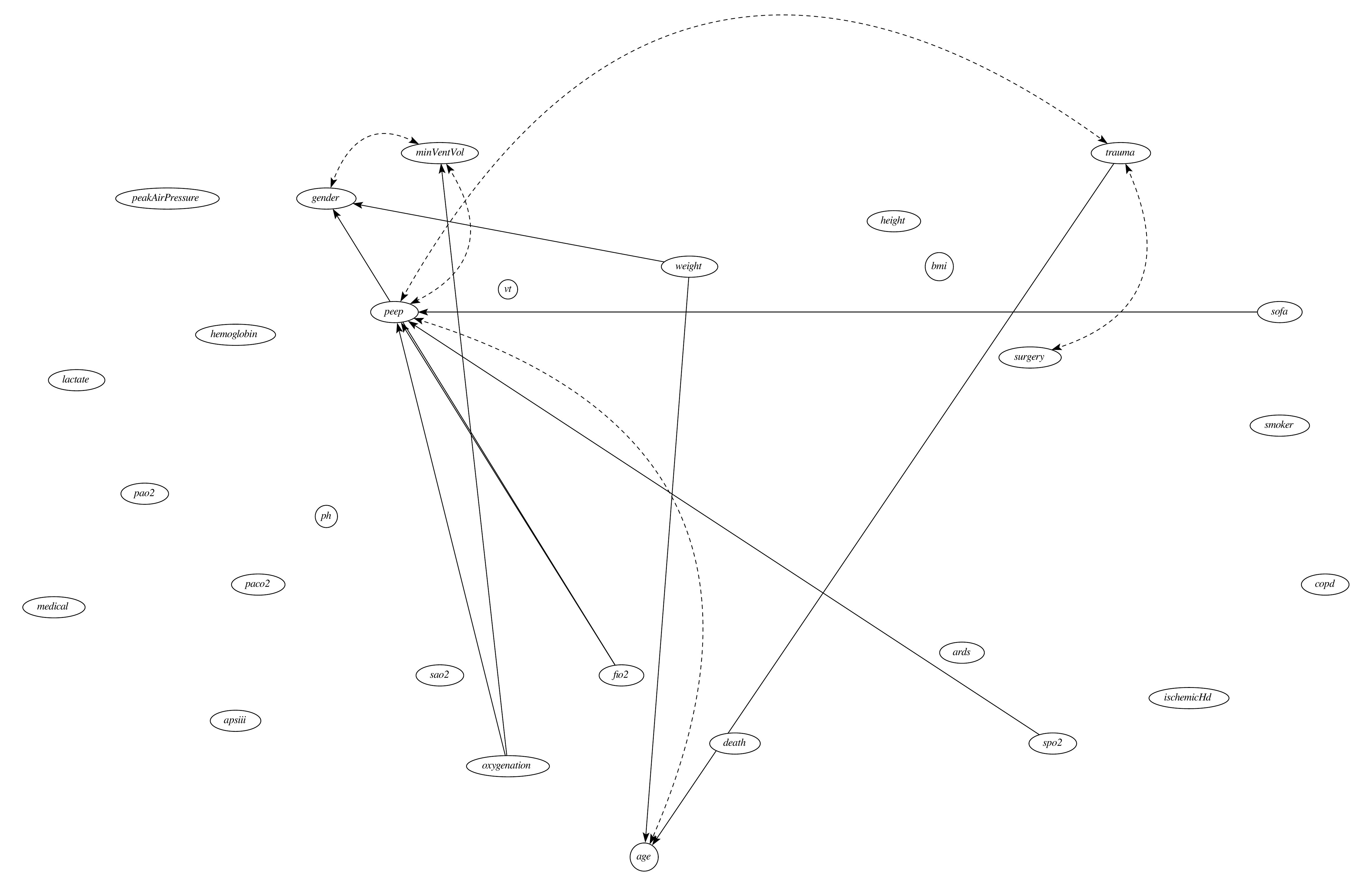}
	\caption{Causal Graph $G_{\overline{U - \{peep\}}}$.}
	\label{figure4}
\end{figure}

\begin{figure}[h]
	\centering
	\includegraphics[width=0.6\textwidth]{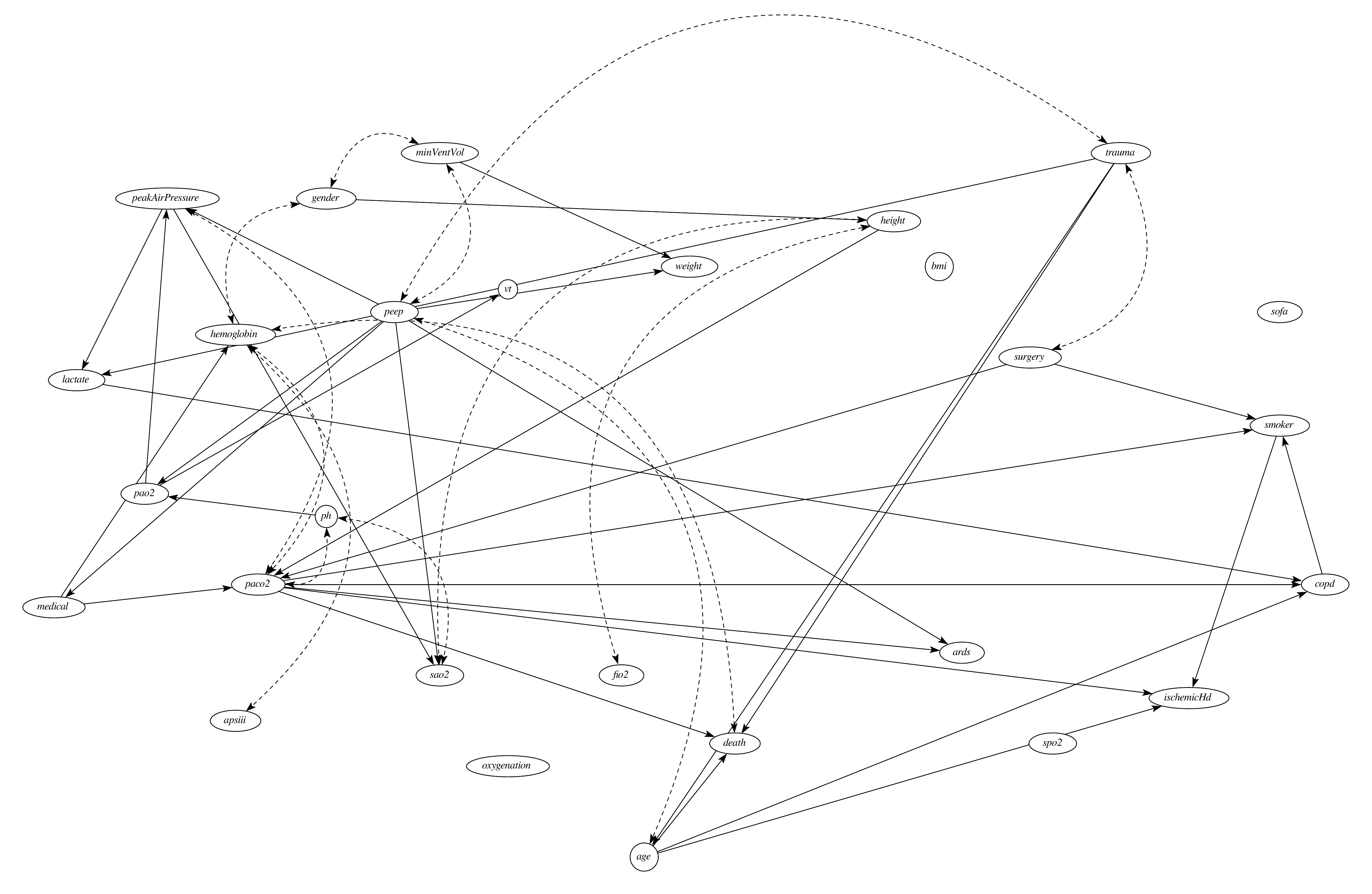}
	\caption{Causal Graph $G_{\overline{}\underline{age}}$.}
	\label{figure5}
\end{figure}

\begin{figure}[h]
	\centering
	\includegraphics[width=0.6\textwidth]{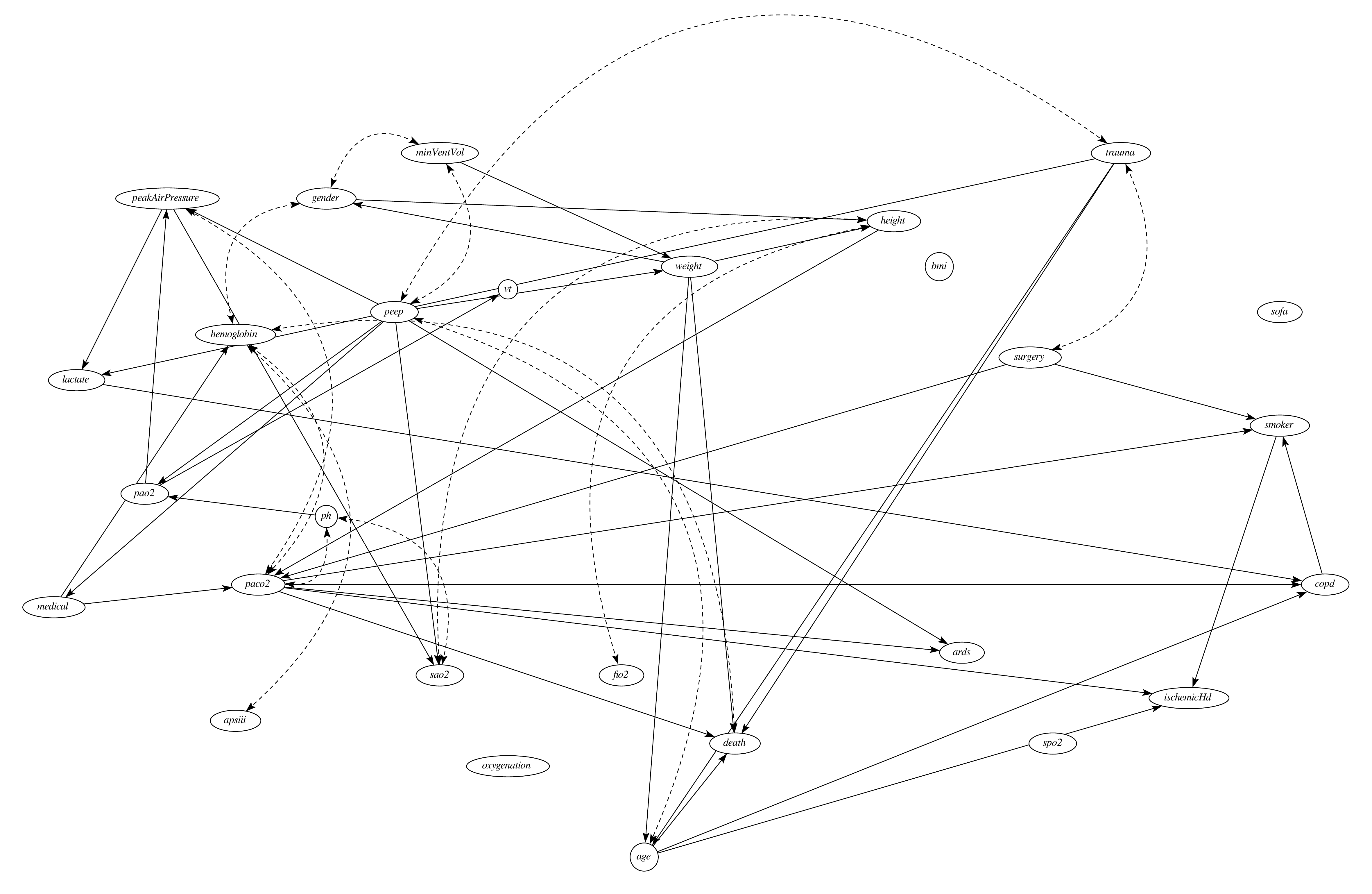}
	\caption{Causal Graph $G_{\overline{U - \{copd\}}}$.}
	\label{figure6}
\end{figure}

\begin{figure}[h]
	\centering
	\includegraphics[width=0.6\textwidth]{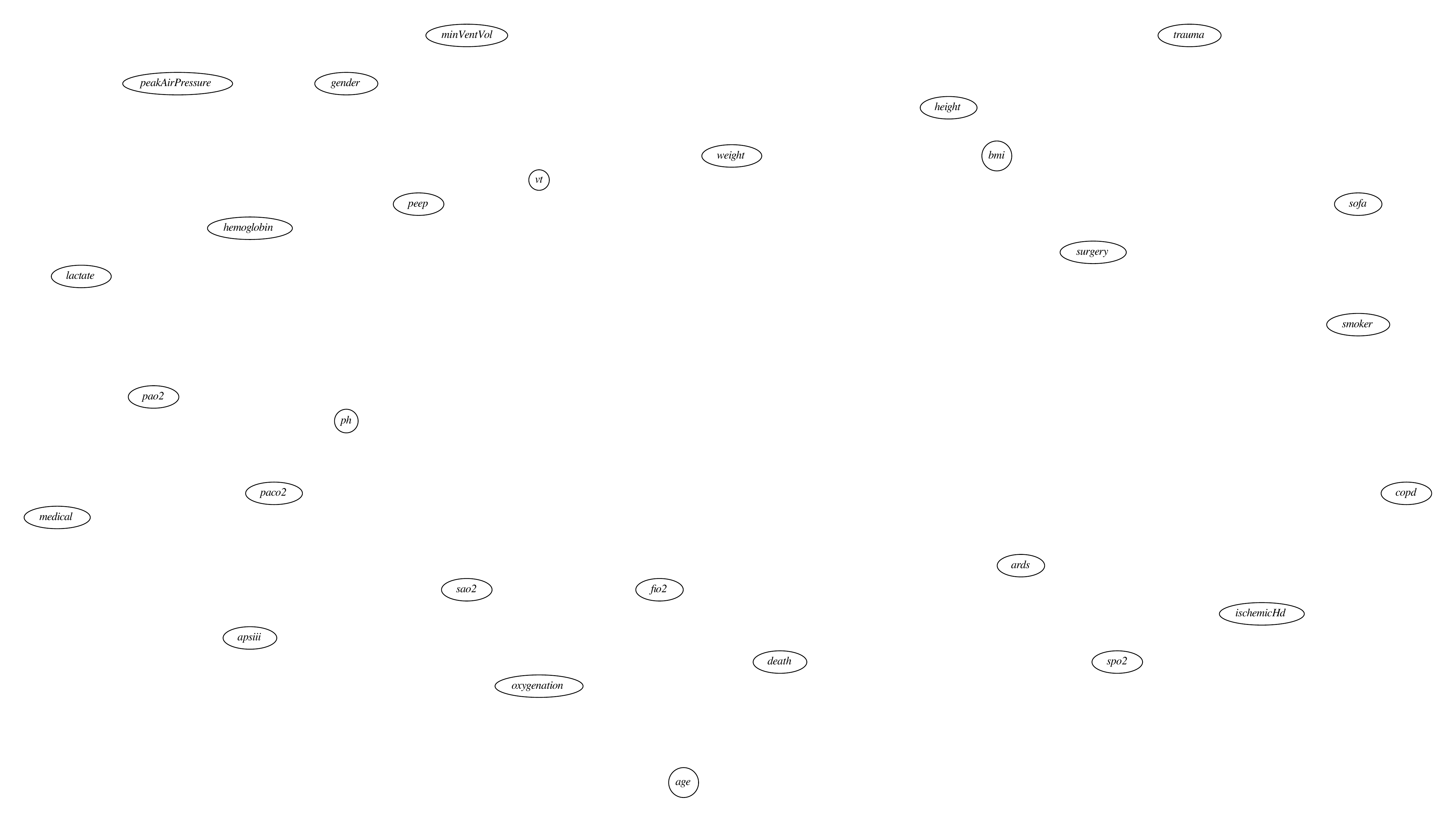}
	\caption{Causal Graph $G_{\overline{}\underline{age,spo2}}$.}
	\label{figure7}
\end{figure}

\begin{figure}[h]
	\centering
	\includegraphics[width=0.6\textwidth]{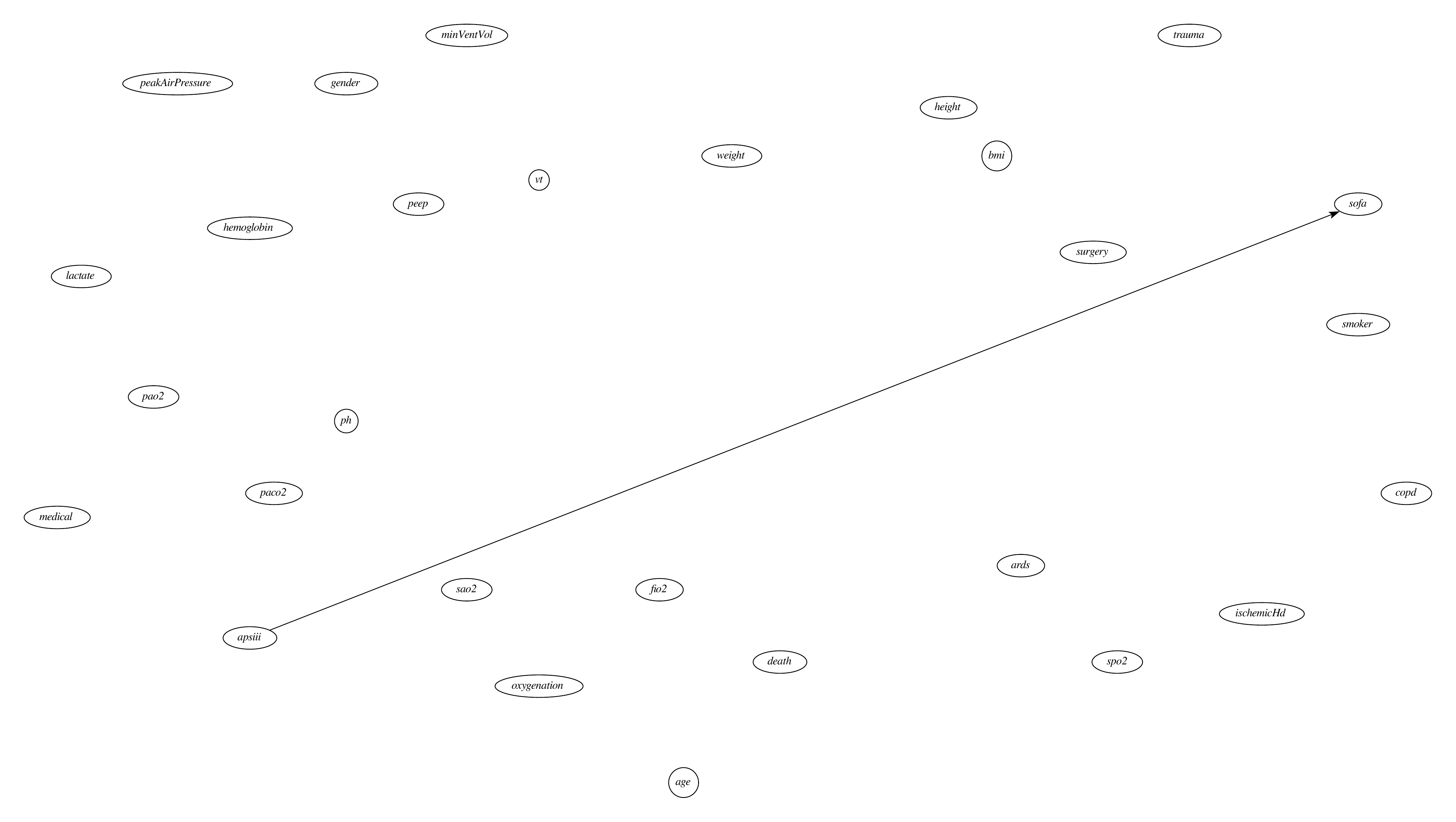}
	\caption{Causal Graph $G_{\overline{U - \{fio2\}}}$.}
	\label{figure8}
\end{figure}

\begin{figure}[h]
	\centering
	\includegraphics[width=0.6\textwidth]{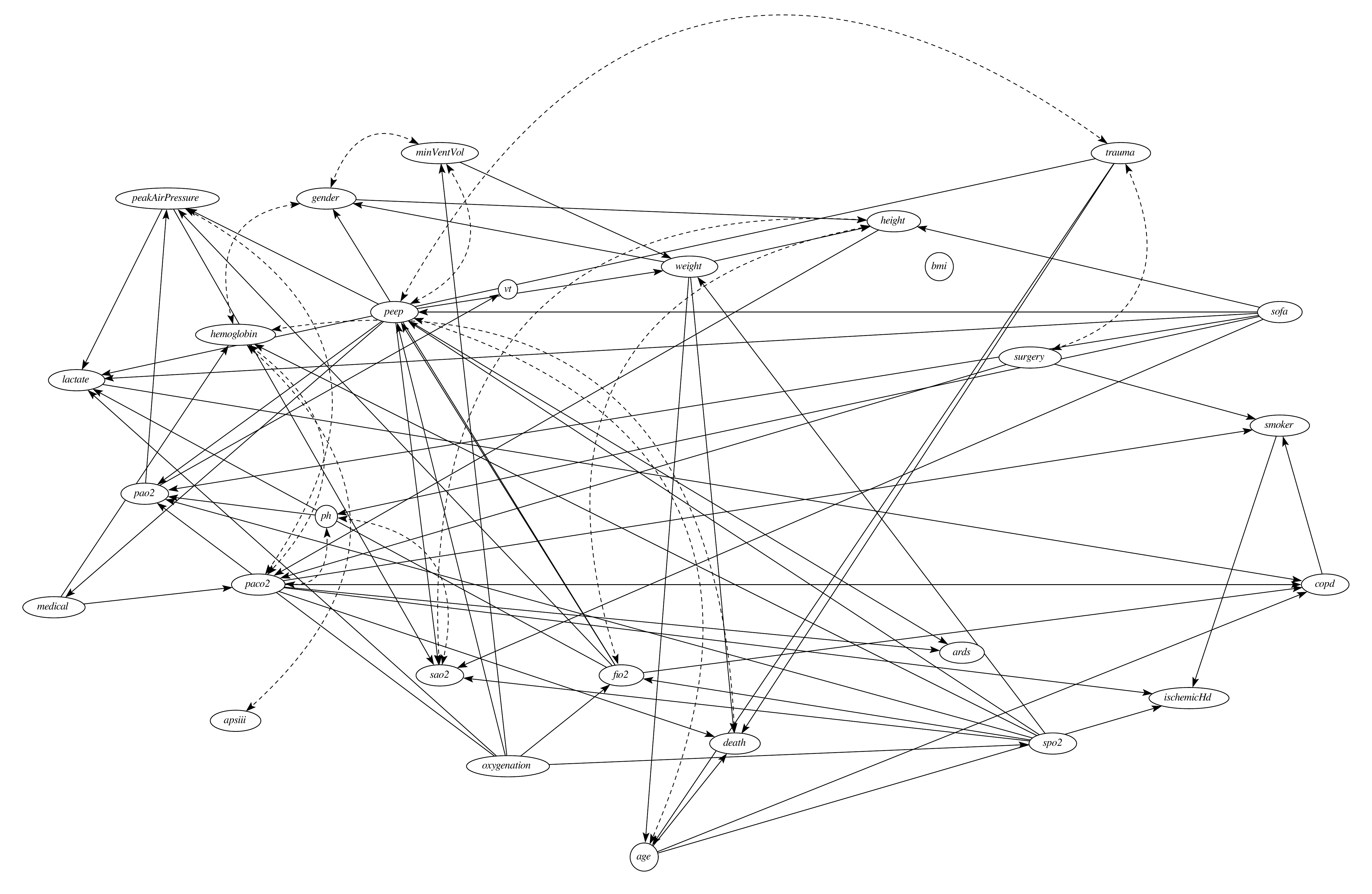}
	\caption{Causal Graph $G_{\overline{}\underline{peep,copd,age,spo2}}$.}
	\label{figure9}
\end{figure}

\begin{figure}[h]
	\centering
	\includegraphics[width=0.6\textwidth]{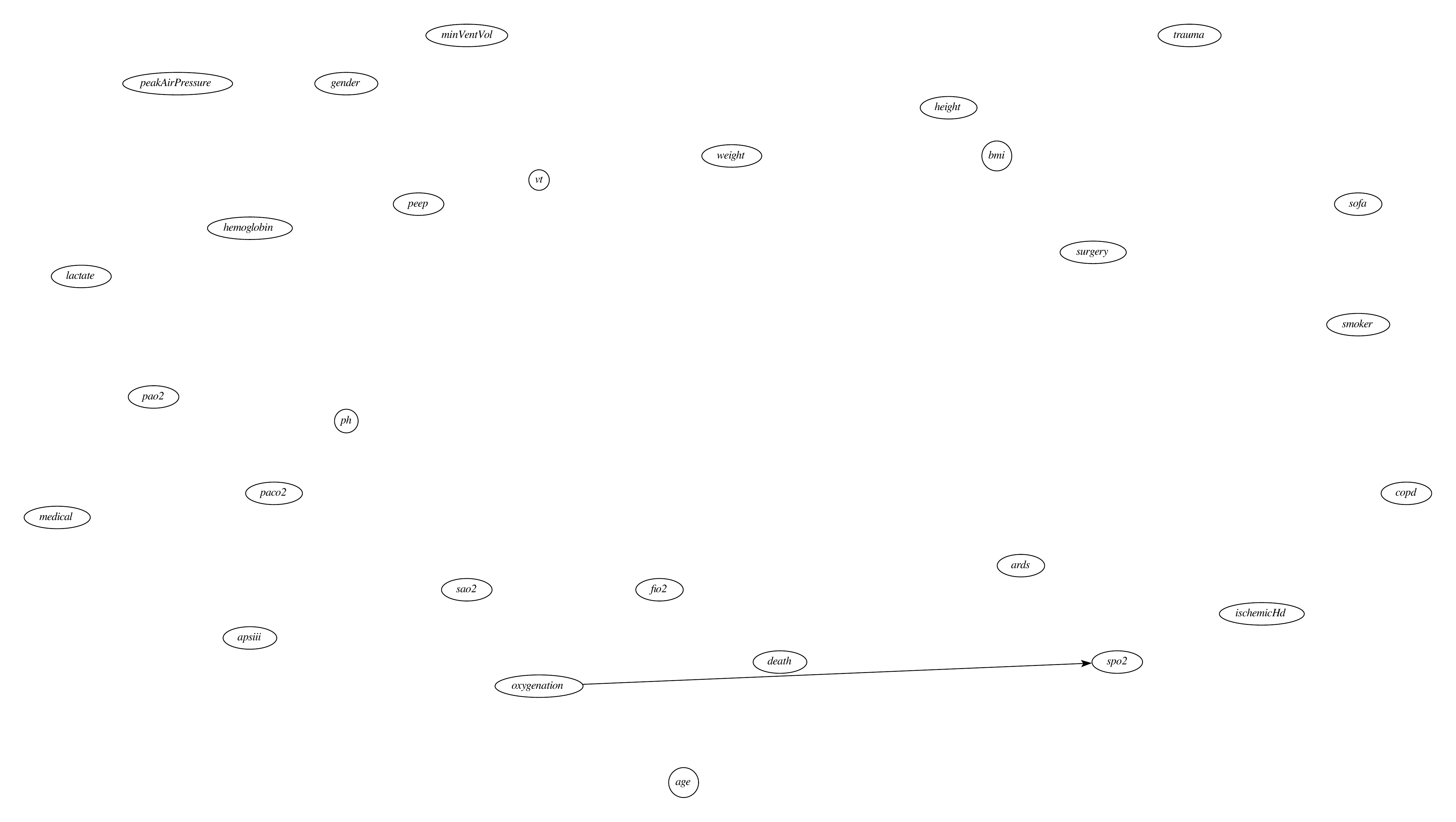}
	\caption{Causal Graph $G_{\overline{U - \{hemoglobin\}}}$.}
	\label{figure10}
\end{figure}

\begin{figure}[h]
	\centering
	\includegraphics[width=0.6\textwidth]{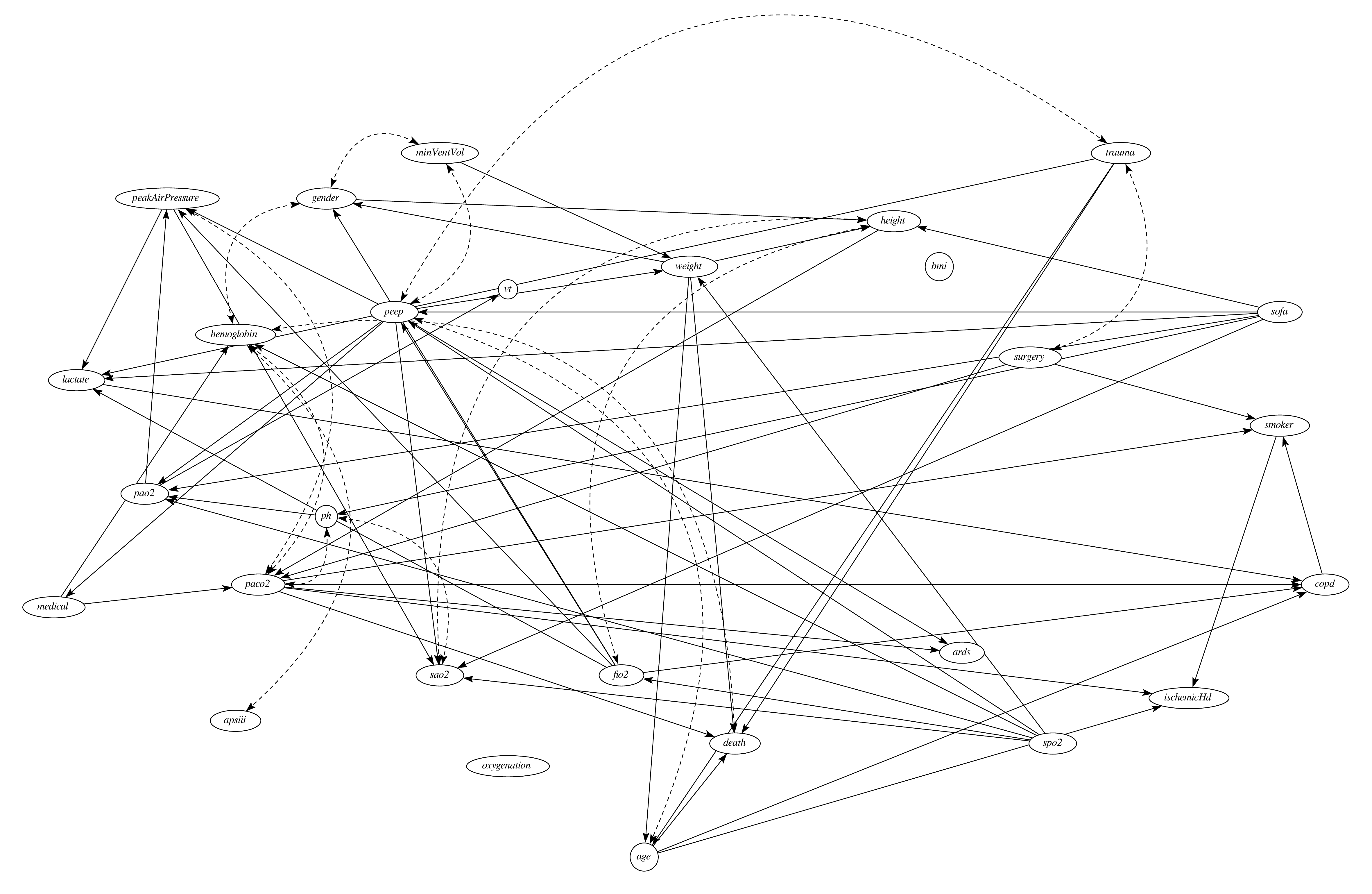}
	\caption{Causal Graph $G_{\overline{}\underline{gender,spo2}}$.}
	\label{figure11}
\end{figure}

\begin{figure}[h]
	\centering
	\includegraphics[width=0.6\textwidth]{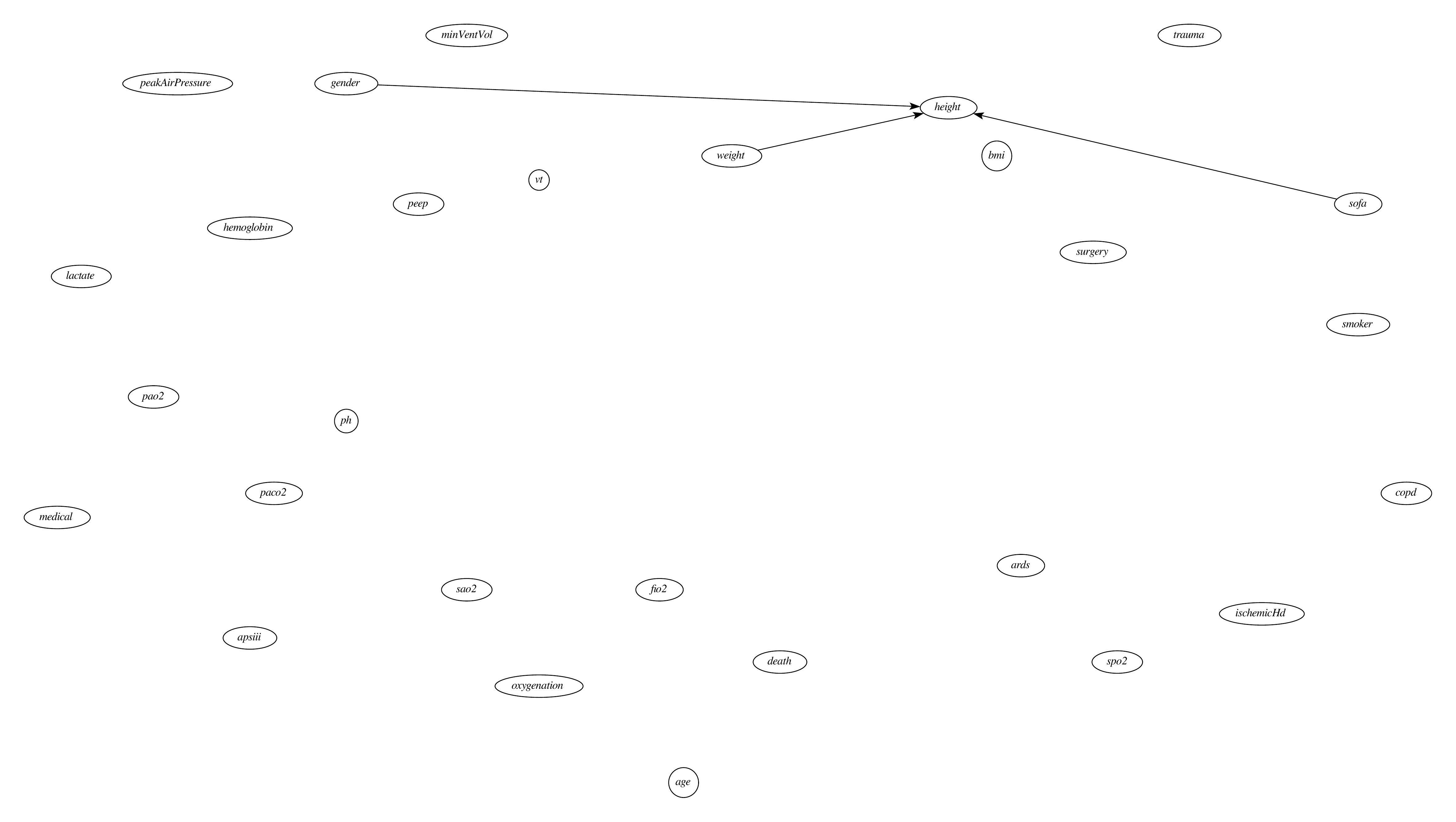}
	\caption{Causal Graph $G_{\overline{U - \{trauma\}}}$.}
	\label{figure12}
\end{figure}

\begin{figure}[h]
	\centering
	\includegraphics[width=0.6\textwidth]{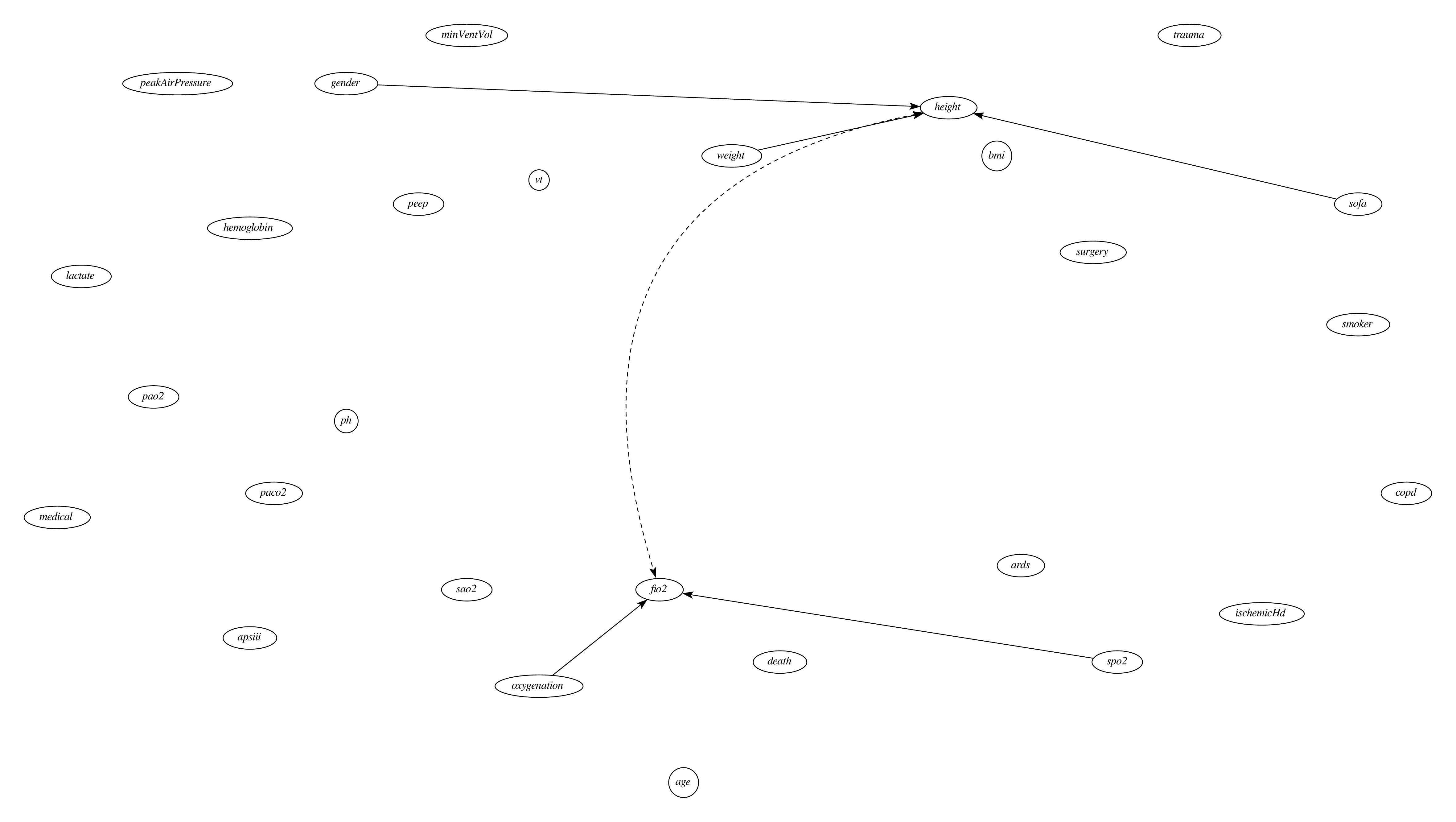}
	\caption{Causal Graph $G_{\overline{}\underline{peep,age}}$.}
	\label{figure13}
\end{figure}

\begin{figure}[h]
	\centering
	\includegraphics[width=0.6\textwidth]{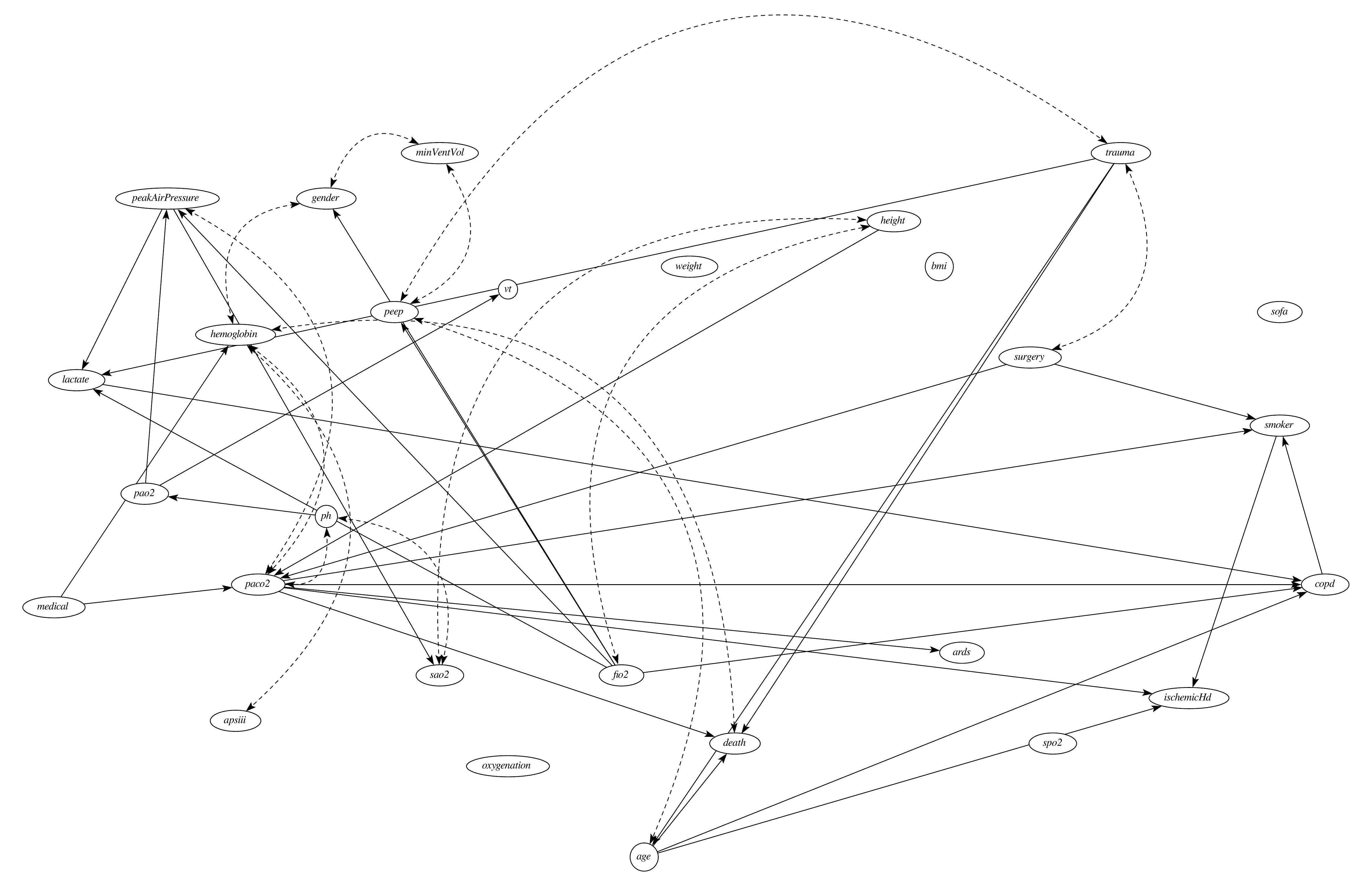}
	\caption{Causal Graph $G_{\overline{U - \{smoker\}}}$.}
	\label{figure14}
\end{figure}

\begin{figure}[h]
	\centering
	\includegraphics[width=0.6\textwidth]{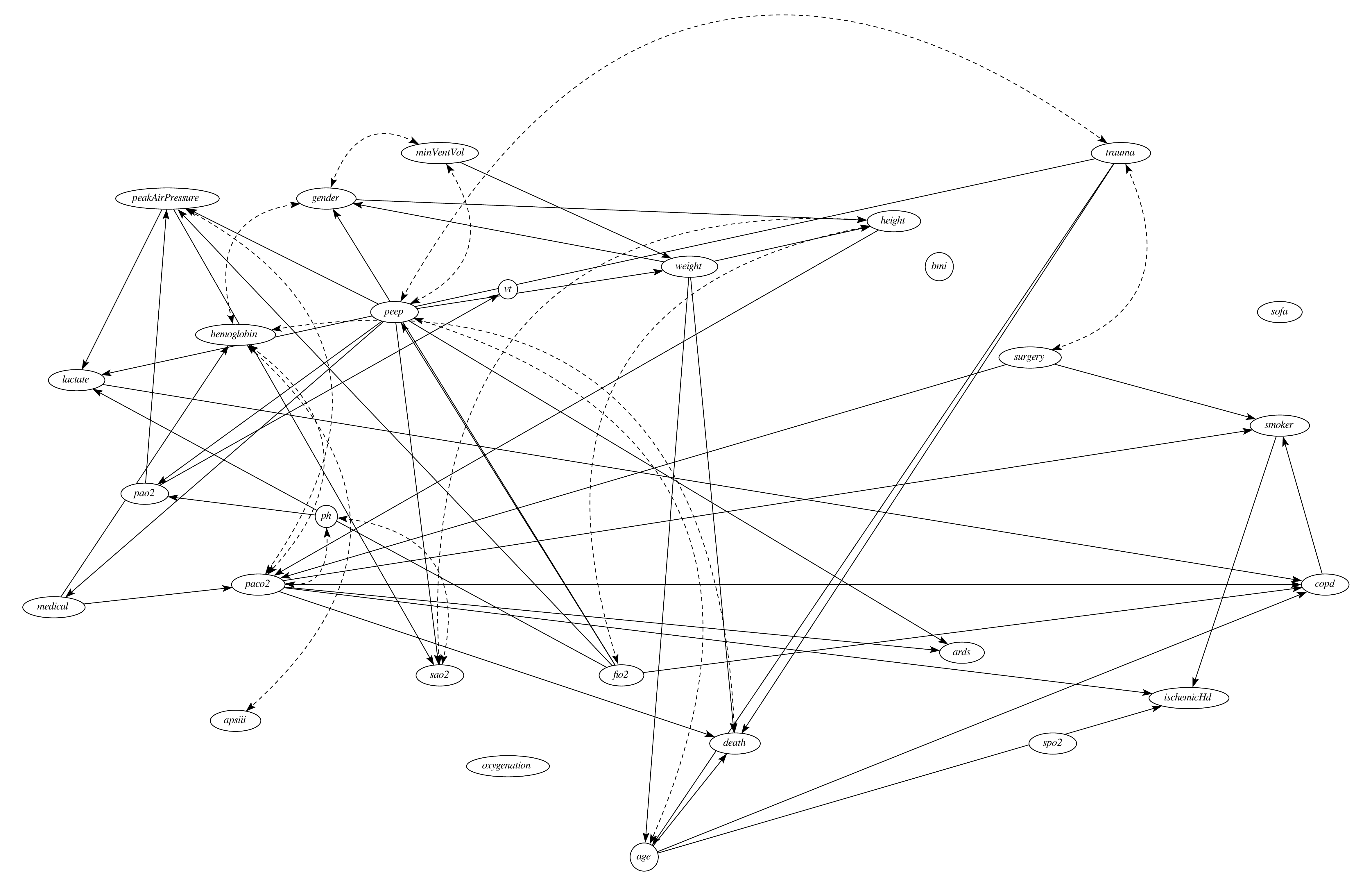}
	\caption{Causal Graph $G_{\overline{U - \{paco2\}}}$.}
	\label{figure15}
\end{figure}

\begin{figure}[h]
	\centering
	\includegraphics[width=0.6\textwidth]{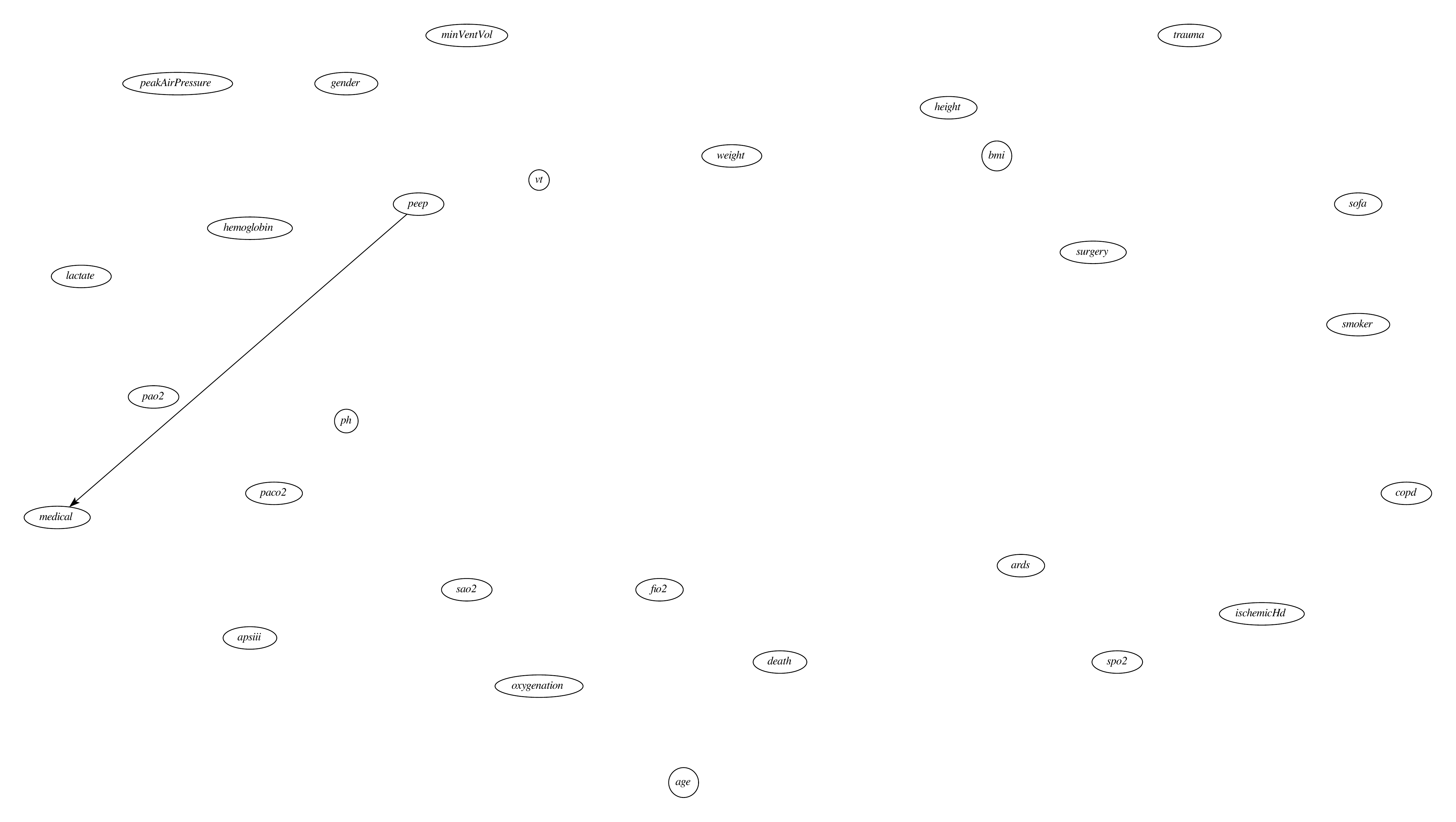}
	\caption{Causal Graph $G_{\overline{}\underline{copd,spo2,ph,smoker,age,ards,bmi}}$.}
	\label{figure16}
\end{figure}

\begin{figure}[h]
	\centering
	\includegraphics[width=0.6\textwidth]{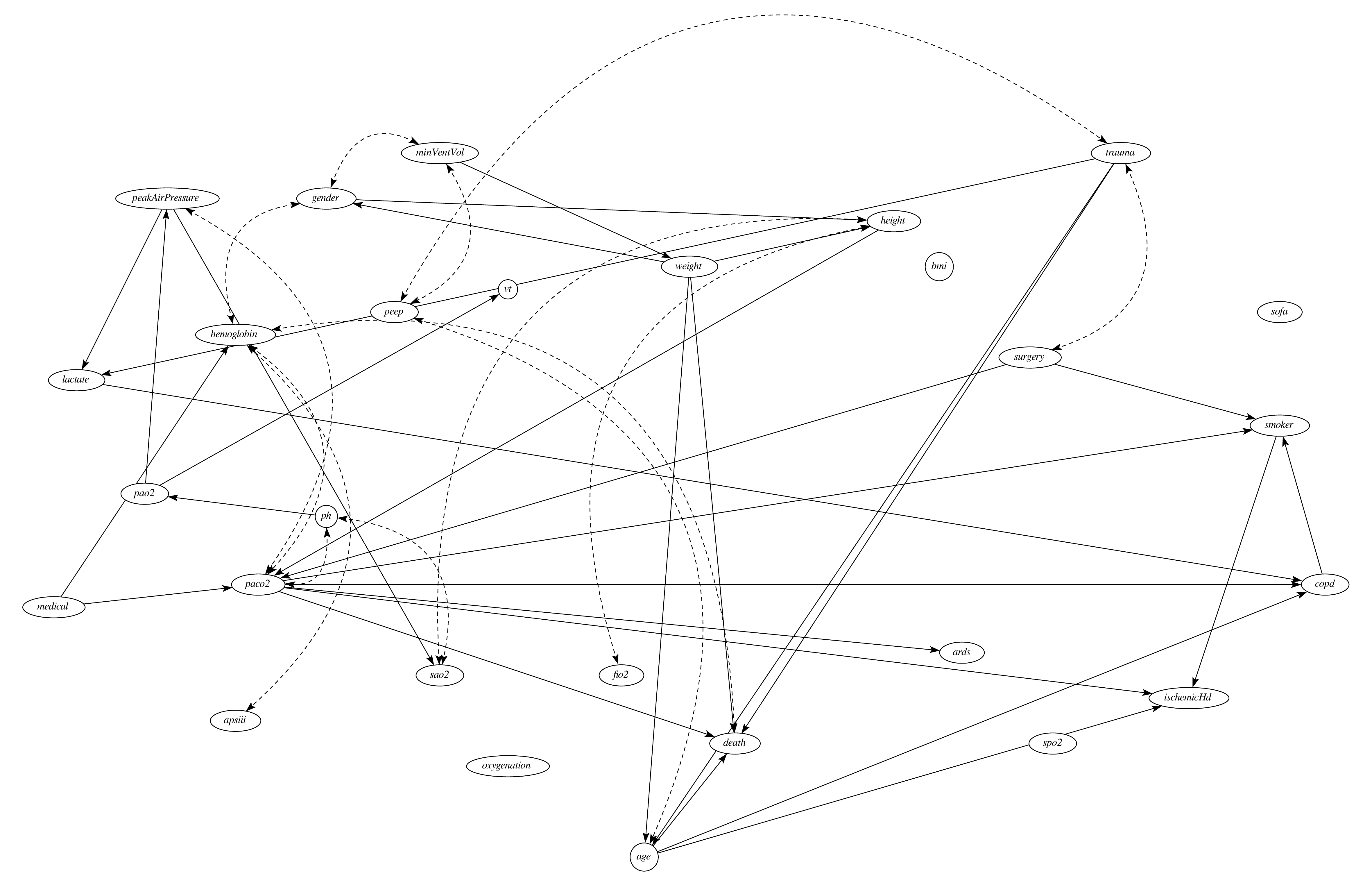}
	\caption{Causal Graph $G_{\overline{U - \{apsiii\}}}$.}
	\label{figure17}
\end{figure}

\begin{figure}[h]
	\centering
	\includegraphics[width=0.6\textwidth]{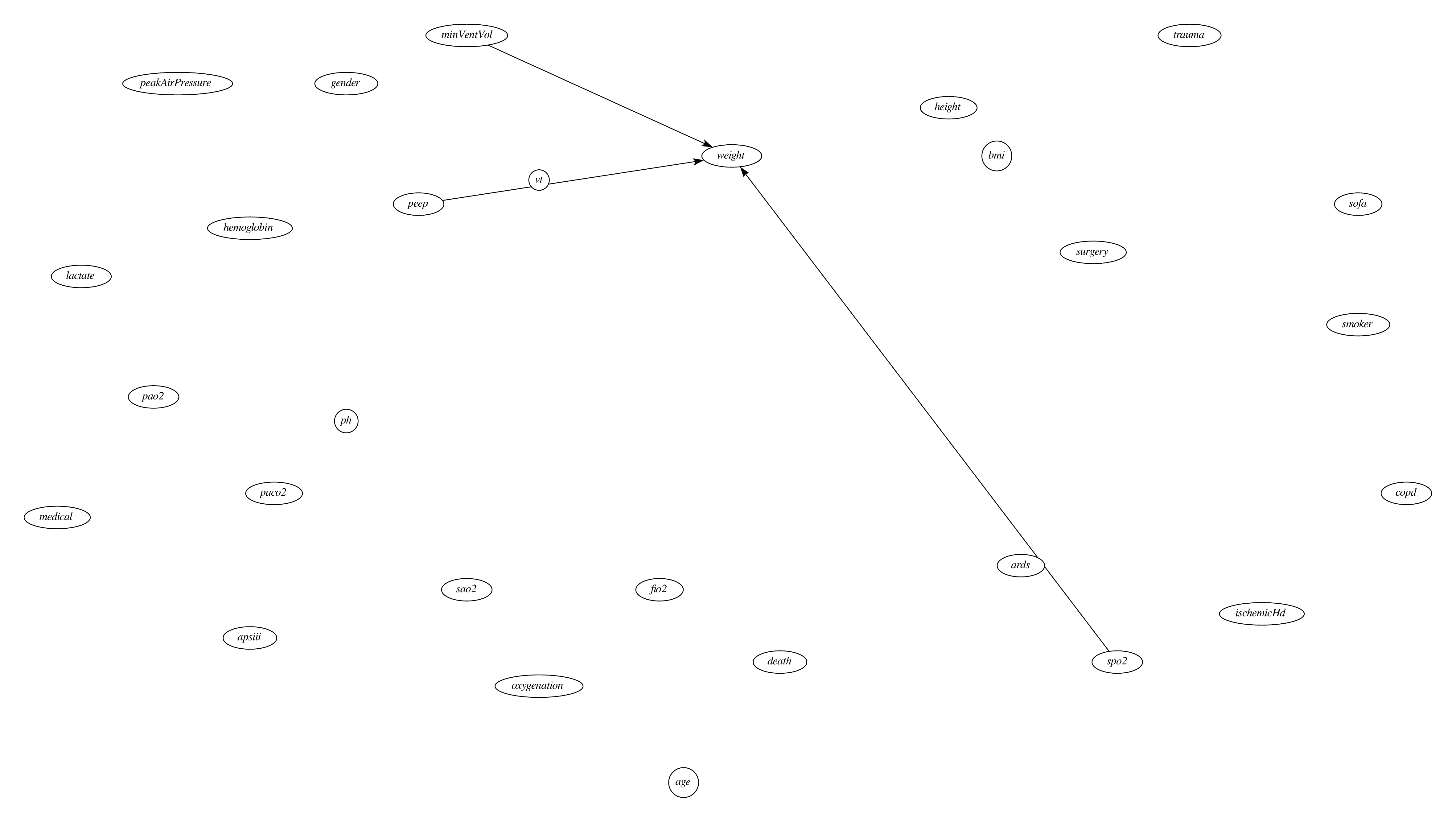}
	\caption{Causal Graph $G_{\overline{}\underline{S - \{apsiii\}}}$.}
	\label{figure18}
\end{figure}

\begin{figure}[h]
	\centering
	\includegraphics[width=0.6\textwidth]{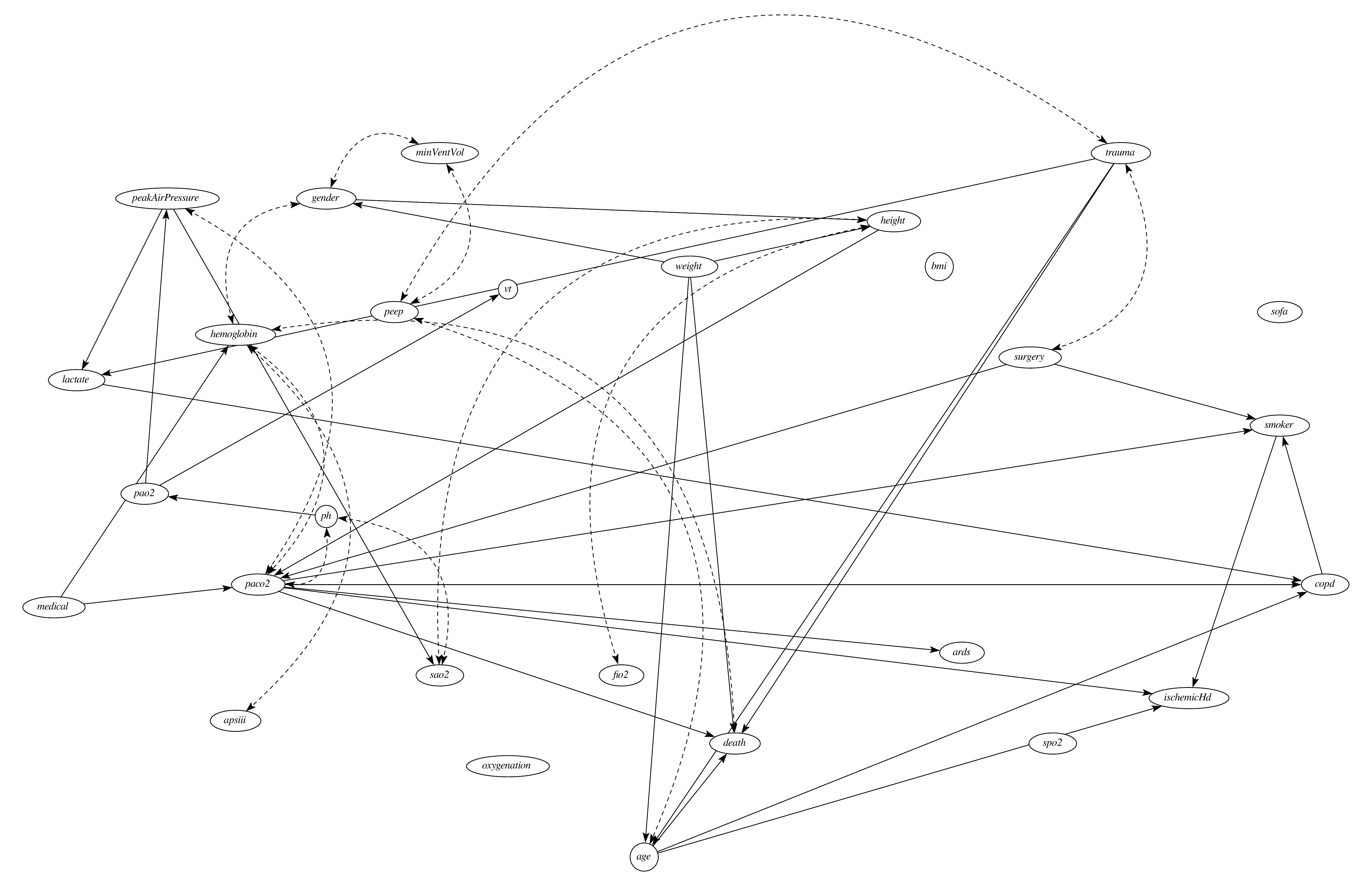}
	\caption{Causal Graph $G_{\overline{U - \{peakAirPressure\}}}$.}
	\label{figure19}
\end{figure}

\begin{figure}[h]
	\centering
	\includegraphics[width=0.6\textwidth]{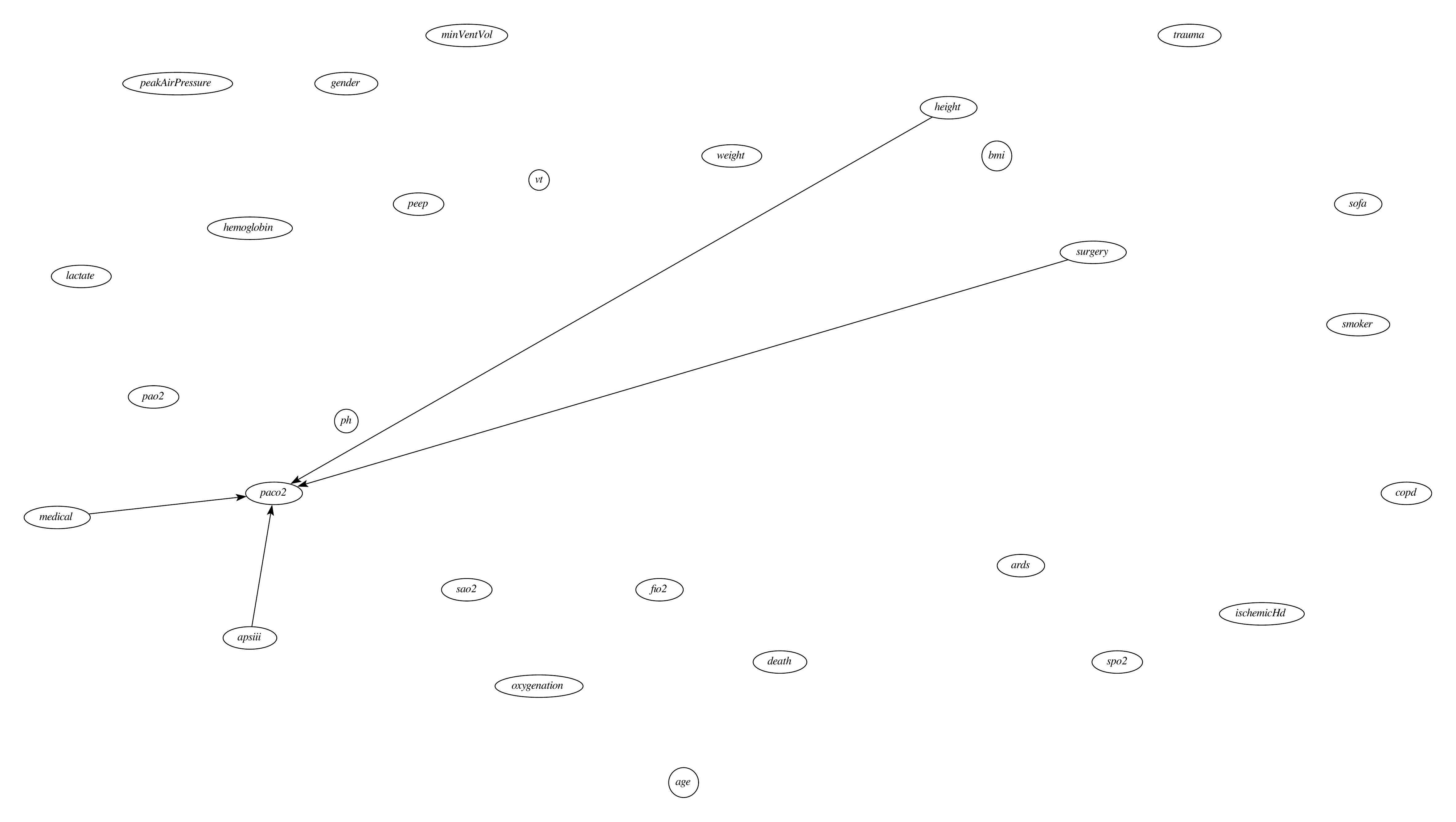}
	\caption{Causal Graph $G_{\overline{}\underline{S}}$.}
	\label{figure20}
\end{figure}

\begin{figure}[h]
	\centering
	\includegraphics[width=0.6\textwidth]{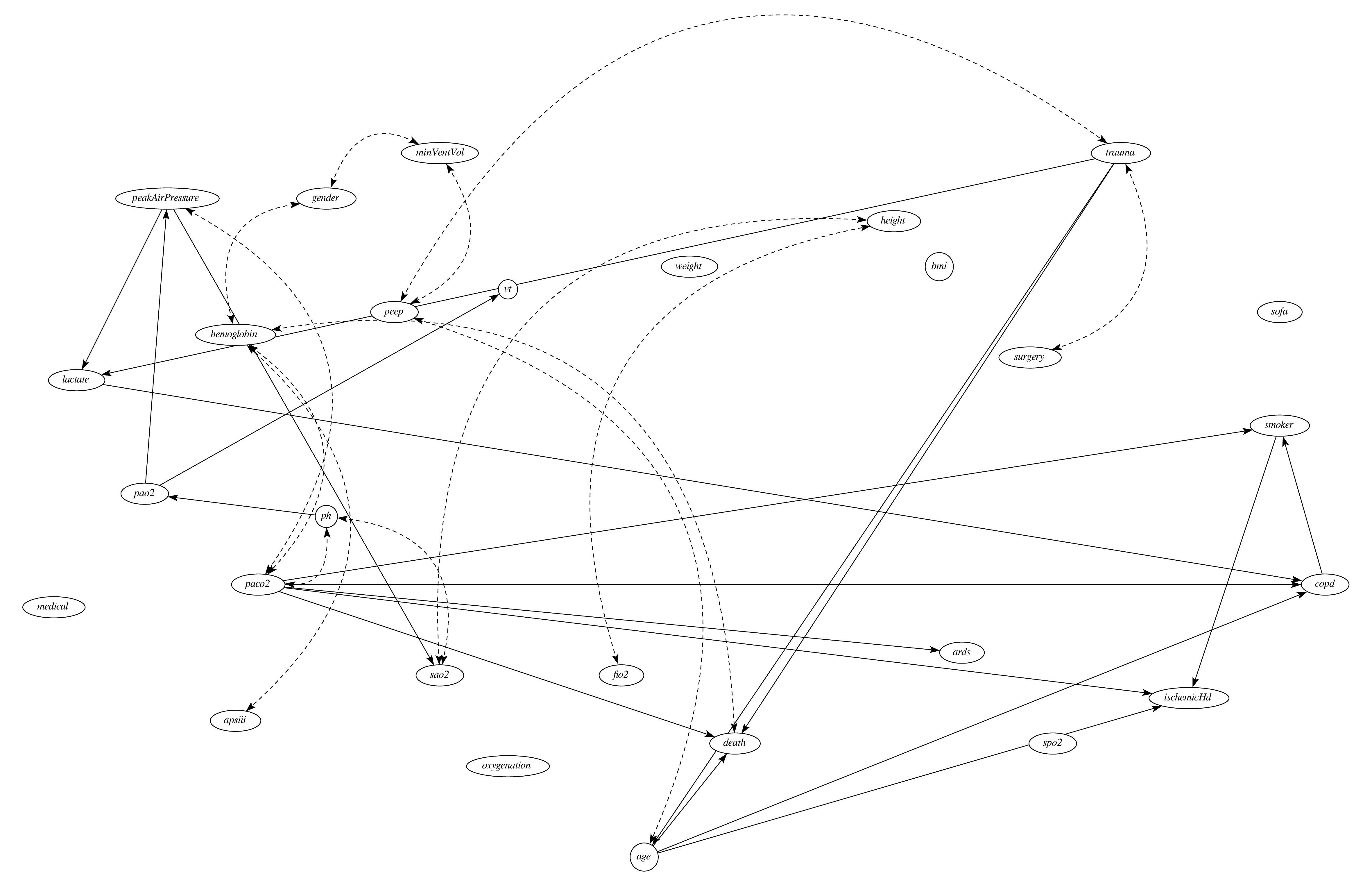}
	\caption{Causal Graph $G_{\overline{U - \{lactate\}}}$.}
	\label{figure21}
\end{figure}

\begin{figure}[h]
	\centering
	\includegraphics[width=0.6\textwidth]{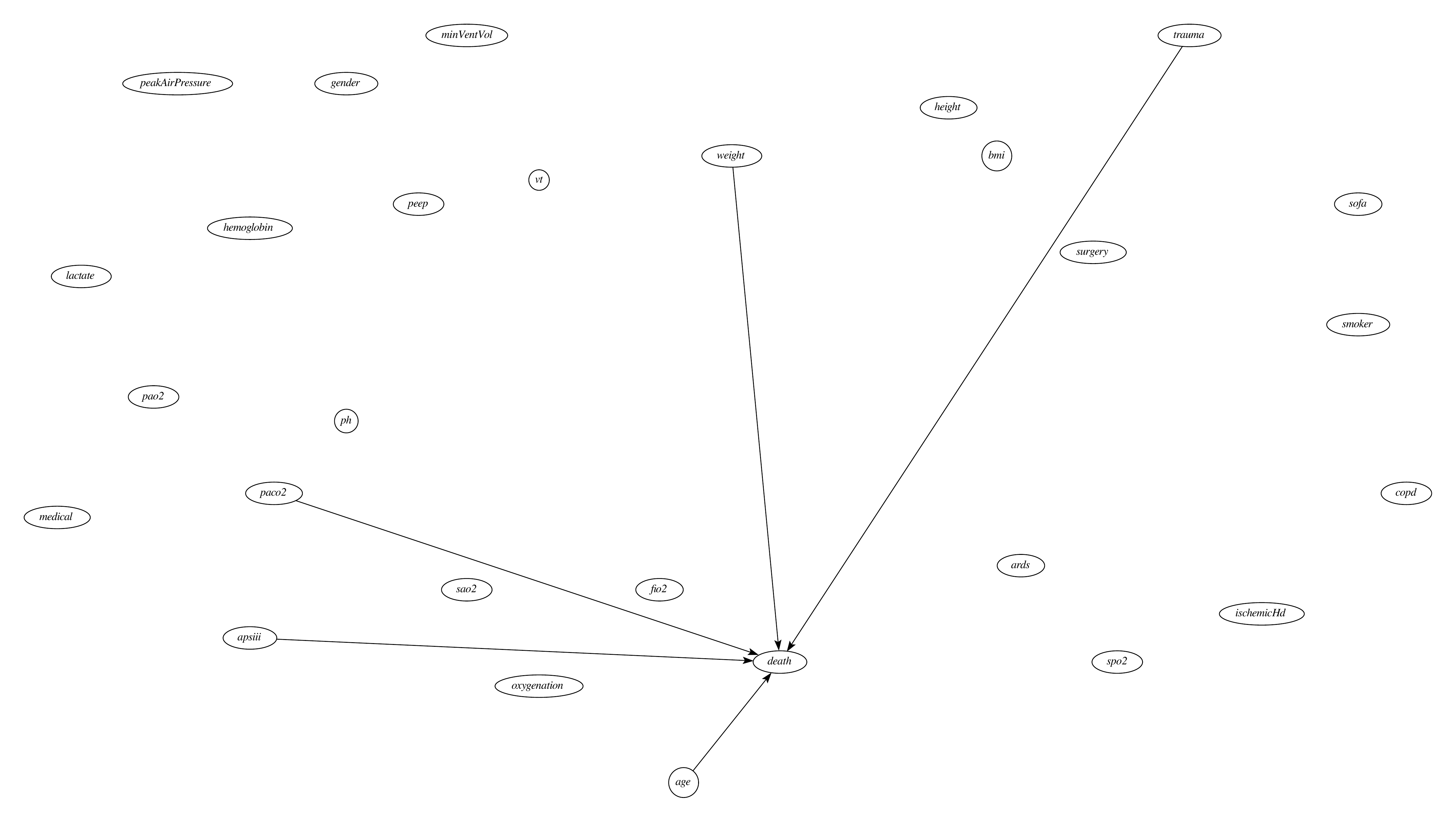}
	\caption{Causal Graph $G_{\overline{}\underline{I}}$.}
	\label{figure22}
\end{figure}

\begin{figure}[h]
	\centering
	\includegraphics[width=0.6\textwidth]{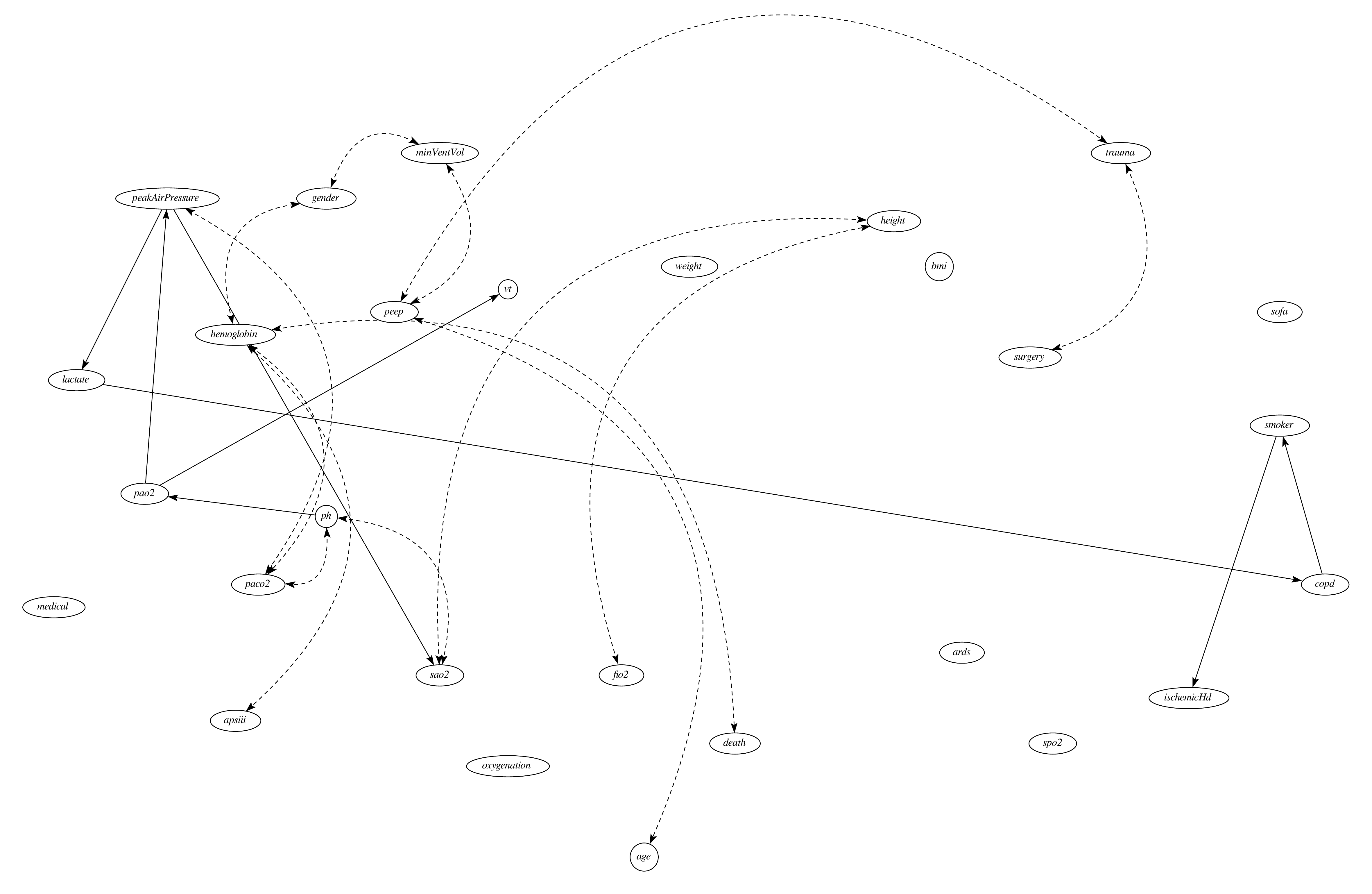}
	\caption{Causal Graph $G_{\overline{U - \{death\}}}$.}
	\label{figure23}
\end{figure}

\begin{figure}[h]
	\centering
	\includegraphics[width=0.6\textwidth]{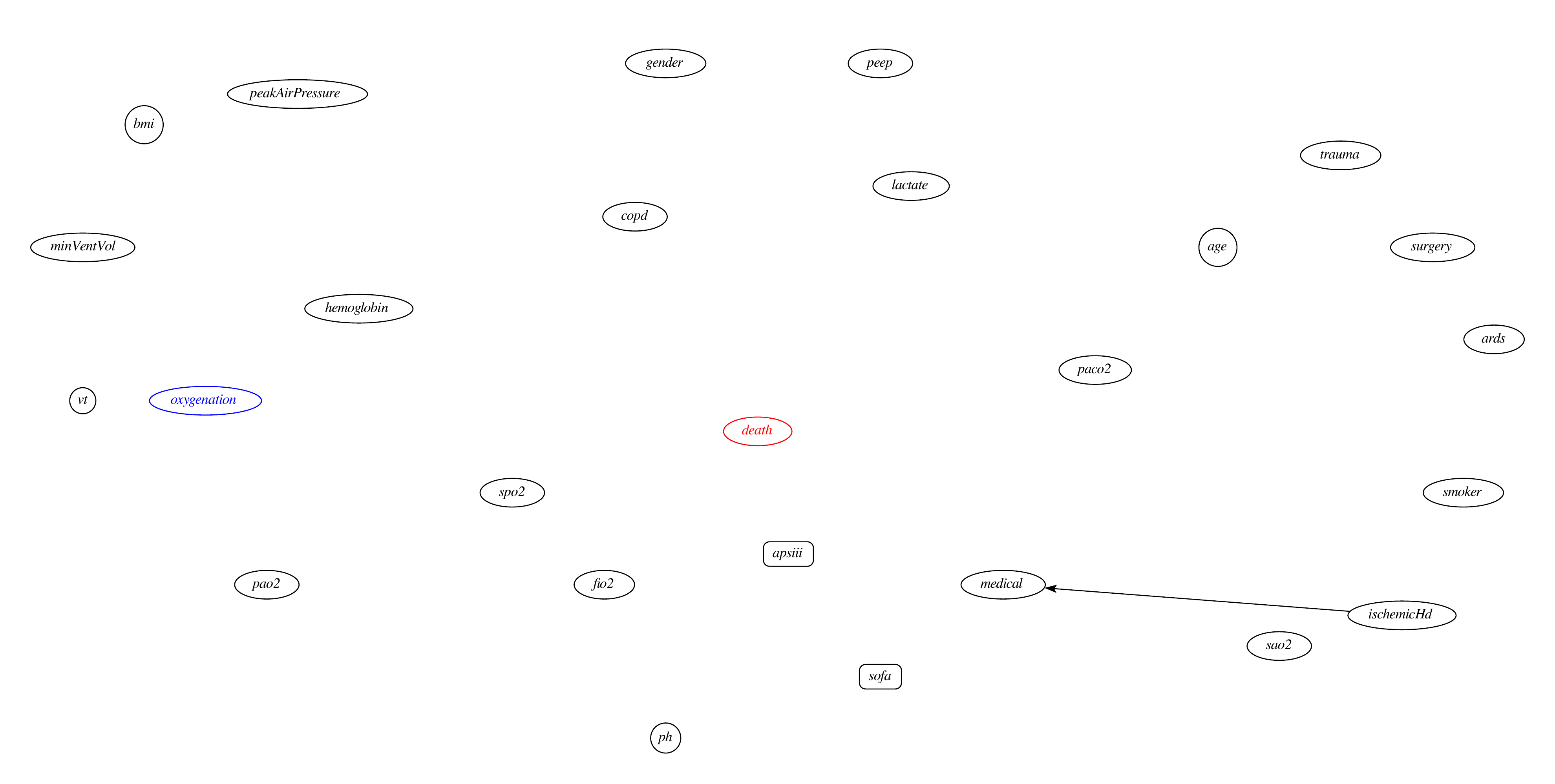}
	\caption{Causal Graph $G_{\overline{}\underline{H}}$.}
	\label{figure24}
\end{figure}

\begin{figure}[h]
	\centering
	\includegraphics[width=0.6\textwidth]{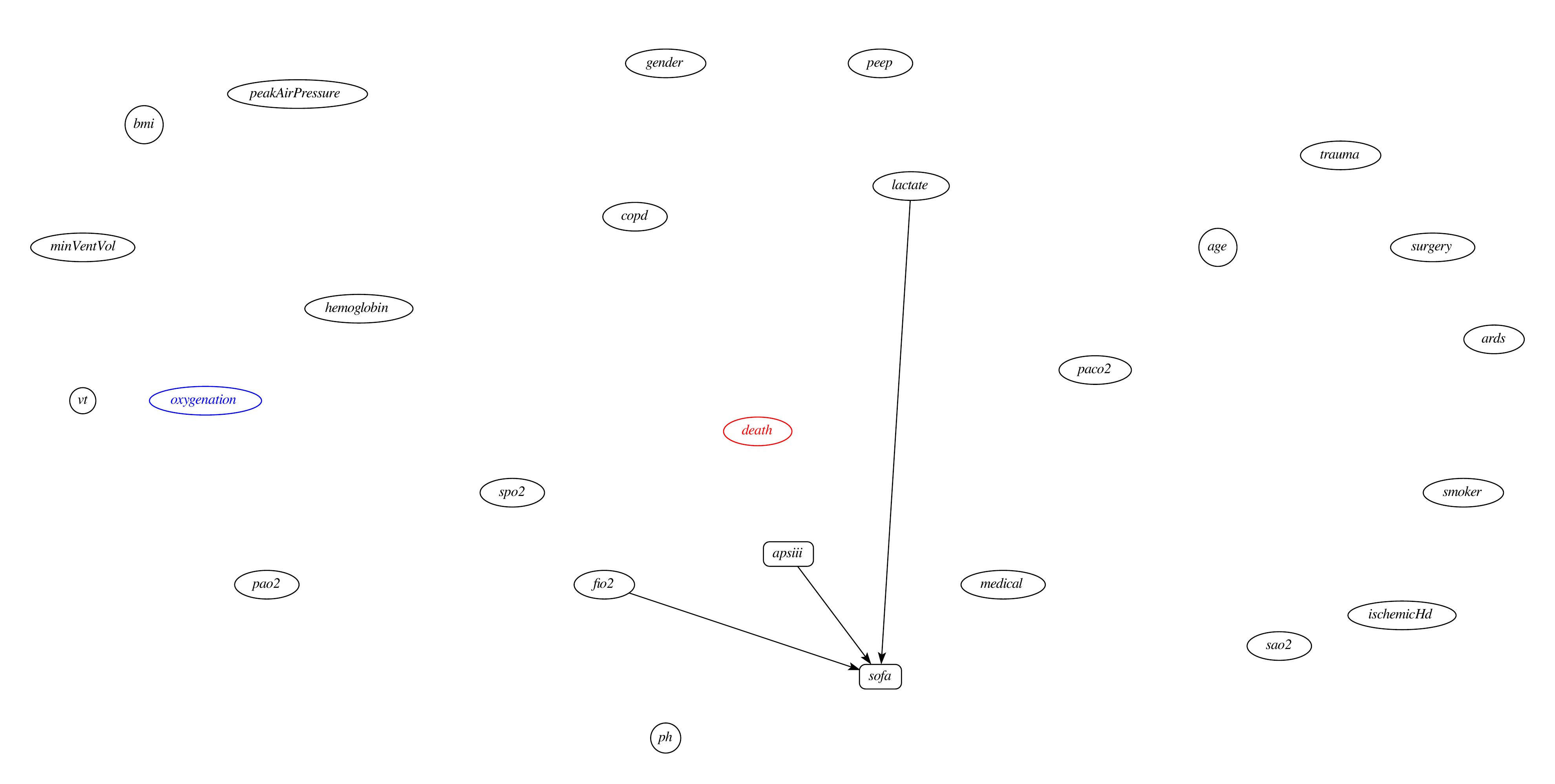}
	\caption{Causal Graph $G_{\overline{U - \{sofa\}}}$.}
	\label{figure25}
\end{figure}

\begin{figure}[h]
	\centering
	\includegraphics[width=0.6\textwidth]{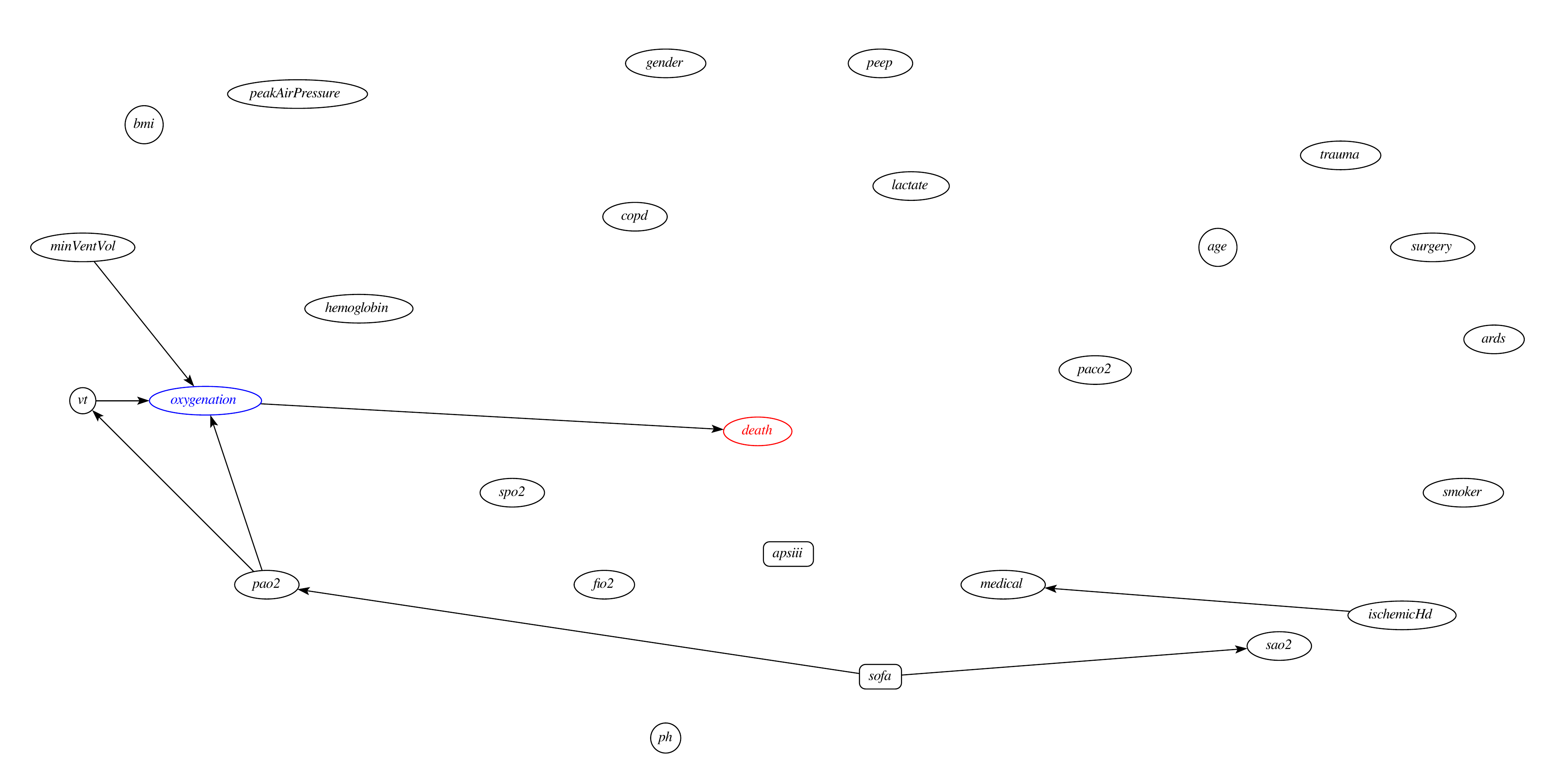}
	\caption{Causal Graph $G_{\overline{}\underline{G}}$.}
	\label{figure26}
\end{figure}
\end{appendices}

\end{document}